\documentclass[11pt]{article}
\usepackage[margin=1in]{geometry}
\usepackage{amsmath,amssymb,amsthm,mathtools}
\usepackage{enumitem}
\usepackage[hypertexnames=false]{hyperref}
\usepackage[T1]{fontenc}
\usepackage[utf8]{inputenc}
\usepackage{lmodern}
\usepackage{microtype}

\usepackage[round]{natbib}
\renewcommand{\cite}{\citep}
\usepackage{multicol}
\usepackage{booktabs}
\usepackage{array}
\usepackage{tabularx}
\usepackage{needspace}
\usepackage{xcolor}
\usepackage{tikz}
\usetikzlibrary{
  arrows.meta,
  backgrounds,
  calc,
  decorations.pathreplacing,
  fit,
  patterns,
  positioning
}
\usepackage{mleftright}
\mleftright

\newcommand{\R}{\mathbb R}
\newcommand{\E}{\mathbb E}

\newcommand{\Var}{\operatorname{Var}}
\newcommand{\1}{\mathbf 1}
\newcommand{\eps}{\varepsilon}

\newcommand{\calP}{\mathcal P}
\newcommand{\calD}{\mathcal D}
\newcommand{\calW}{\mathcal W}
\newcommand{\calG}{\mathcal G}
\newcommand{\calF}{\mathcal F}

\newcommand{\calI}{\mathcal I}
\newcommand{\calT}{\mathcal T}

\newcommand{\calU}{\mathcal U}
\newcommand{\calQ}{\mathcal Q}

\newcommand{\len}{\operatorname{len}}

\DeclareMathOperator{\dist}{dist}

\DeclareMathOperator{\Proj}{Proj}

\DeclareMathOperator*{\argmin}{arg\,min}

\newcommand{\imax}{i_{\max}}

\newtheorem{theorem}{Theorem}
\newtheorem{lemma}{Lemma}

\theoremstyle{remark}

\theoremstyle{plain}

\newtheorem{corollary}[theorem]{Corollary}

\definecolor{queryblue}{RGB}{56,111,164}
\definecolor{querylight}{RGB}{211,226,241}
\definecolor{coverorange}{RGB}{230,126,34}
\definecolor{covergreen}{RGB}{70,150,110}
\definecolor{filterred}{RGB}{190,65,65}
\definecolor{quantpurple}{RGB}{126,87,166}
\definecolor{quantlight}{RGB}{232,222,242}
\definecolor{softgray}{RGB}{238,240,243}
\hypersetup{
  colorlinks=true,
  linkcolor=queryblue,
  citecolor=covergreen,
  urlcolor=queryblue,
  pdftitle={Non-Adaptive 1-Bit Mean Estimation: Minimax Rates and the Sample-Interval Tradeoff},
  pdfauthor={Ivan Lau and Jonathan Scarlett}
}

\title{Non-Adaptive 1-Bit Mean Estimation:\\
Minimax Rates and the Sample--Interval Tradeoff}
\author{Ivan Lau \qquad Jonathan Scarlett
\\
\\
National University of Singapore}
\date{}

\begin{document}
\maketitle

\begin{abstract}
We study distributed one-dimensional mean estimation under a 1-bit
communication constraint.  Each agent observes one sample, drawn independently
from an unknown distribution, and returns a single bit in response to a query
$Q\colon\mathbb R\to\{0,1\}$ chosen by a central learner.  The distribution
has mean in $[-\lambda,\lambda]$ and
$k$-th central moment at most $\sigma^k$, for a fixed $k>1$.  The order-optimal
two-stage protocol of \citet{lau2026order} uses responses from the first batch
to choose the second-batch queries, and whether this single round of interaction
is necessary was posed as an open problem.  We answer this negatively: for every
$k>1$, a non-adaptive protocol attains the adaptive 1-bit minimax
rate (and we note that concurrent works reached the same conclusion via different strategies).  We further determine the minimax sample
complexity among non-adaptive 1-bit estimators when every one-set
$Q^{-1}(1)$ is restricted to a union of at most $s$ intervals.  Relative to
unrestricted non-adaptive 1-bit querying, this constraint adds a term of
order $(\lambda\sigma/(s\eps^2))\log(1/\delta)$, giving the full tradeoff
between sample complexity and interval complexity to within $k$-dependent constant factors.  As a corollary, we
identify, order-wise, the minimum interval budget needed
to retain the unrestricted 1-bit minimax sample rate.
\end{abstract}

\section{Introduction}
\label{sec:introduction}

In distributed 1-bit mean estimation, a central learner does not observe the
samples directly. Instead, a number of agents have access to a single sample each, and each agent returns one bit indicating 
whether its sample lies in some measurable set chosen by the learner.  
When the mean may lie anywhere in a large range $[-\lambda,\lambda]$, 
a query aimed at the wrong region can carry almost
no information (e.g., due to the region having almost no probability mass). An adaptive protocol can address this obstacle by using an initial batch of responses to localize the mean near a coarse center $c$, and then choosing later queries to refine the estimate within that neighborhood. 
This is the architecture of the order-optimal two-stage estimator of
\citet{lau2026order}, illustrated in Figure~\ref{fig:information-flow}(a).
In that estimator, $c$ determines the second-batch refinement queries, which therefore cannot be chosen until the first-batch responses have been observed.
\citet{lau2026open} asked whether this query-side interaction is necessary for order-optimal 1-bit mean estimation.

We answer this question negatively by constructing a non-adaptive protocol that attains the adaptive 1-bit minimax rate. In our protocol, both the localization and refinement queries are chosen before any response is observed, so no query-side feedback is required and all agents can answer in parallel.  After all responses have been collected, the learner recovers $c$ from the localization bits and uses it to decode the stored refinement bits. The procedure is therefore ``two-stage'' only in the decoding mechanism, and not in query selection; Figure~\ref{fig:information-flow}(b) illustrates this distinction. 
We also characterize the minimax sample complexity when the set $Q^{-1}(1)$ associated with each query is a union of at most $s$ intervals (with $s=1$ recovering the non-adaptive interval-query setting considered in \citep[Theorem~11]{lau2026order}).  This describes the full sample--interval tradeoff up to constant factors and, as a consequence, the minimum order of $s$ that preserves the unrestricted 1-bit sample rate. Section~\ref{sec:contributions} gives a more detailed summary of these contributions.

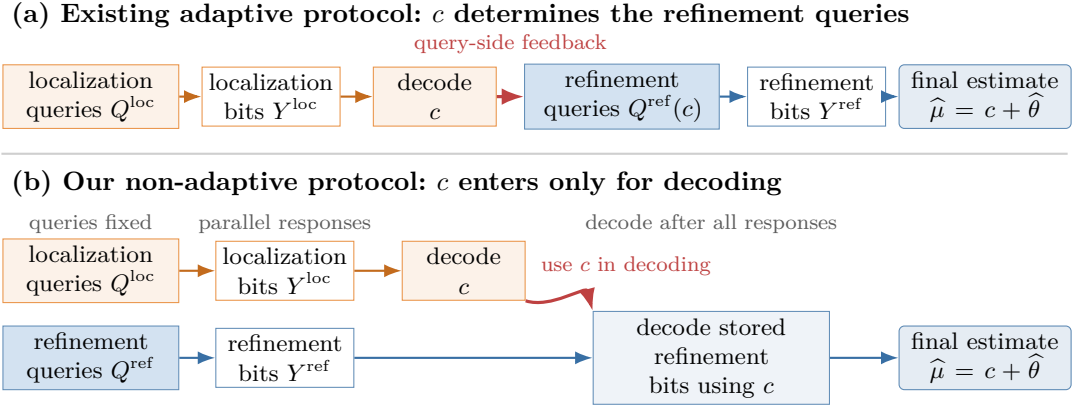
\begin{figure}[t]
\centering
\begin{tikzpicture}[
    x=0.91cm,y=0.72cm,>=Latex,font=\footnotesize,
    qnode/.style={draw=queryblue!78,fill=querylight,
      text width=2.15cm,minimum height=0.78cm,align=center,inner sep=2.5pt},
    lnode/.style={draw=coverorange!82,fill=coverorange!10,
      text width=2.15cm,minimum height=0.78cm,align=center,inner sep=2.5pt},
    bnode/.style={draw=gray!65,fill=white,
      text width=1.65cm,minimum height=0.78cm,align=center,inner sep=2.5pt},
    cnode/.style={draw=coverorange!82,fill=coverorange!10,
      text width=1.45cm,minimum height=0.78cm,align=center,inner sep=2.5pt},
    dnode/.style={draw=queryblue!78,fill=queryblue!8,
      text width=2.95cm,minimum height=0.92cm,align=center,inner sep=2.5pt},
    onode/.style={draw=queryblue!82,fill=queryblue!13,rounded corners=2pt,
      text width=2.10cm,minimum height=0.82cm,align=center,inner sep=2.5pt},
    stagelabel/.style={font=\scriptsize,text=gray!72!black}
]
  \node[font=\small\bfseries,anchor=west] at (0.0,7.55)
  {(a) Existing adaptive protocol: $c$ determines the refinement queries};

  \node[lnode] (aqloc) at (1.30,6.05)
    {localization\\queries $Q^{\rm loc}$};
  \node[bnode,draw=coverorange!82] (ayloc) at (3.92,6.05)
    {localization\\bits $Y^{\rm loc}$};
  \node[cnode] (acenter) at (6.30,6.05)
    {decode\\$c$};
  \node[qnode,text width=2.40cm] (aqref) at (9.02,6.05)
    {refinement\\queries $Q^{\rm ref}(c)$};
  \node[bnode,draw=queryblue!78] (ayref) at (11.85,6.05)
    {refinement\\bits $Y^{\rm ref}$};
  \node[onode] (aout) at (14.30,6.05)
    {final estimate\\$\widehat\mu=c+\widehat\theta$};
  \draw[->,thick,coverorange!85!black] (aqloc)--(ayloc);
  \draw[->,thick,coverorange!85!black] (ayloc)--(acenter);
  \draw[->,very thick,filterred] (acenter)--(aqref)
    node[midway,above=12pt,font=\scriptsize,text=filterred]
      {query-side feedback};
  \draw[->,thick,queryblue] (aqref)--(ayref);
  \draw[->,thick,queryblue] (ayref)--(aout);

  \draw[gray!38,thick] (0.0,5.00)--(15.50,5.00);

  \node[font=\small\bfseries,anchor=west] at (0.0,4.45)
    {(b) Our non-adaptive protocol: $c$ enters only for decoding};
  \node[stagelabel] at (1.30,3.72) {queries fixed};
  \node[stagelabel] at (4.12,3.72) {parallel responses};
  \node[stagelabel] at (10.32,3.72) {decode after all responses};

  \node[lnode] (bqloc) at (1.30,2.85)
    {localization\\queries $Q^{\rm loc}$};
  \node[bnode,draw=coverorange!82] (byloc) at (4.12,2.85)
    {localization\\bits $Y^{\rm loc}$};
  \node[cnode] (bcenter) at (6.72,2.85)
    {decode\\$c$};
  \node[qnode] (bqref) at (1.30,1.25)
    {refinement\\queries $Q^{\rm ref}$};
  \node[bnode,draw=queryblue!78] (byref) at (4.12,1.25)
    {refinement\\bits $Y^{\rm ref}$};
  \node[dnode] (bdecoder) at (10.32,1.25)
    {decode stored refinement\\bits using $c$};
  \node[onode] (bout) at (14.30,1.25)
    {final estimate\\$\widehat\mu=c+\widehat\theta$};
  \draw[->,thick,coverorange!85!black] (bqloc)--(byloc);
  \draw[->,thick,coverorange!85!black] (byloc)--(bcenter);
  \draw[->,thick,queryblue] (bqref)--(byref);
  \draw[->,thick,queryblue] (byref)--(bdecoder);
  \draw[->,very thick,filterred] (bcenter.south east)
    to[out=-18,in=105] (bdecoder.north west);
  \node[font=\scriptsize,text=filterred,anchor=south] at (9.08,2.58)
    {use $c$ in decoding};
  \draw[->,thick,queryblue] (bdecoder)--(bout);
\end{tikzpicture}
\caption{Where the decoded center enters the protocol.
(a) In the adaptive protocol of \citet{lau2026order}, localization produces a
center $c$ that determines the subsequent refinement queries.
(b) In our protocol, both query families are fixed in advance and their responses are collected in parallel. After all responses have been collected, $c$ is used to decode the stored refinement bits.  The red arrows show where $c$ enters: query selection in (a) and decoding in (b).}
\label{fig:information-flow}
\end{figure}

\subsection{Problem Setup}
\label{sec:setup}

\paragraph{Distribution class.}
For known parameters $k > 1$ and $\lambda\ge\sigma>0$, define the nonparametric distribution class
\[
    \calP(k,\lambda,\sigma)
    =
    \left\{
        P:
        \mu(P)=\E_P[X]\in[-\lambda,\lambda],
        \quad
        \E_P[|X-\mu(P)|^k]\le\sigma^k
    \right\}.
\]
We do not impose any density, symmetry, or bounded-support assumptions.  

\paragraph{1-bit communication protocol.}
The learner is interested in estimating the population mean $\mu = \mu(P)$
from $n$ independent and identically distributed (i.i.d.) samples $X_1, \dots, X_n \sim P$ subject to a 1-bit communication constraint per sample.
The learner communicates separately with $n$ agents, each contacted once; the agents do not communicate with one another.  Agent $t$ observes $X_t\sim P$, receives a measurable query function~$Q_t:\mathbb R\to\{0,1\}$ from the learner (or pre-specified as part of the protocol), and returns the single bit $Y_t=Q_t(X_t)$.  We call $Q_t^{-1}(1)$ the query's \emph{one-set}, namely, the set on which it returns one.  After receiving all $n$ responses, the learner forms an estimate $\widehat{\mu}$ from the completed transcript $(Q_1, Y_1, \dots, Q_n, Y_n)$.

\paragraph{Adaptive versus non-adaptive queries.}
We say that a protocol is \emph{adaptive} (or \emph{sequential}) if $Q_t$ may depend on the earlier responses $Y_1, \dots, Y_{t-1}$ and possible public randomness. It is \emph{non-adaptive} if the entire query sequence
$Q_1,\ldots,Q_n$ is chosen before any response is observed, possibly using public randomness.
Hence, none of the responses affect the query selection, and all queries can be issued and answered in parallel.  
Throughout, the learner knows the realized queries and all public randomness.  We say that a non-adaptive protocol has \emph{i.i.d.~queries} if $Q_1,\ldots,Q_n$ are independent draws from a common distribution on binary query functions. Because the common distribution is chosen in advance, i.i.d.~querying is a special case of non-adaptive querying.

\paragraph{Learner's goal.}
For $\eps>0$ and $\delta\in(0,1)$, the learner seeks a mean estimate based on the 1-bit query results that is \emph{$(\eps,\delta)$-PAC} over
$\calP(k,\lambda,\sigma)$, i.e., its output $\widehat\mu$ satisfies
\[
    \sup_{P\in\calP(k,\lambda,\sigma)}
    \Pr_P(|\widehat\mu-\mu(P)|>\eps)
    \le\delta,
\]
where the probability is over the samples and all randomness in the estimator.
Unless stated otherwise, the learner may choose the query rule based on the known model parameters ($k$, $\lambda$, $\sigma$), the target $(\eps,\delta)$, and the prescribed sample budget $n$.  (Some discussion on partially unknown parameters is given in \citep{lau2026order}, but it remains open which dependencies can be removed entirely.)

\paragraph{Notation.}
We use standard asymptotic notation $O(\cdot)$, $\Omega(\cdot)$, and
$\Theta(\cdot)$ to hide absolute constants.  When these constants depend on
the moment parameter $k$, we make the dependence explicit using
$O_k(\cdot)$, $\Omega_k(\cdot)$, and $\Theta_k(\cdot)$.  Throughout, $\log$ denotes the natural logarithm.  In upper-bound rate
expressions, logarithmic factors are implicitly lower bounded by one.

\subsection{Summary of Contributions}
\label{sec:contributions}

We establish two main results: a non-adaptive protocol that matches the
adaptive 1-bit minimax sample complexity for every $k>1$, and matching upper
and lower bounds that characterize the minimax sample complexity of
non-adaptive estimators under every interval budget $s$.

\begin{enumerate}[label=\textbf{\arabic*.},leftmargin=2.5em]
\item \textbf{Minimax-optimal 1-bit mean estimation with i.i.d. queries.}
For every $k>1$, Theorem~\ref{thm:main} gives an $(\eps,\delta)$-PAC estimator whose sample complexity matches the adaptive 1-bit minimax sample complexity.
The construction retains the localization--refinement architecture of~\citet{lau2026order} but moves the dependence on $c$ from query selection to decoding. Its queries are i.i.d., and every realized one-set is a finite union of intervals.

\item \textbf{The minimax sample--interval tradeoff.}
Theorem~\ref{thm:interval-complexity} gives matching upper and lower bounds on
the minimax sample complexity under non-adaptive queries when the one-set of
every query is a union of at most $s$ intervals.  Relative to the unrestricted
1-bit minimax sample complexity, the interval
restriction contributes the additional term $\Theta_k\left( (\lambda\sigma)/(s\eps^2) \cdot \log(1/\delta)\right)$.
The upper bound can be attained with i.i.d. queries.  As a special case,
Corollary~\ref{cor:optimal-components} identifies, up to constants depending
on $k$, the smallest interval budget that preserves the unrestricted 1-bit
minimax sample complexity.
\end{enumerate}

\subsection{Prior and Concurrent Work}
\label{sec:related-work}

Here we focus on the work most directly related to nonparametric 1-bit mean estimation.  Appendix~\ref{app:earlier-work} surveys the broader literature on high-probability mean estimation and communication-constrained learning.

A substantial line of work studies 1-bit mean estimation in parametric settings.  \citet{kipnis2017mean} characterized the optimal asymptotic mean-squared error attainable by adaptive 1-bit protocols for a class of symmetric log-concave location families, including the Gaussian and Laplace families, and compared their performance with that of centralized and non-adaptive distributed protocols.
\citet{cai2020distributed} established minimax communication--risk tradeoffs for distributed Gaussian mean estimation using a localization--refinement decomposition. \citet{kumar2025unknown} studied adaptive and non-adaptive 1-bit estimation when both location and scale are unknown within a known one-dimensional location--scale family. For several such families, they showed that
adaptivity does not improve the asymptotic mean-squared error rate but strictly improves its leading constant.

Beyond parametric families, \citet{abdalla2026robust} constructed
non-adaptive 1-bit estimators for finite-variance distributions using randomized thresholds, with guarantees under adversarial corruption. Their 1-bit construction requires prior knowledge of an interval in which  $X$ lies with high probability.
For the nonparametric finite-moment classes studied here, \citet{lau2025sequential} gave sample-complexity guarantees under finite variance ($k=2$) that are optimal up to logarithmic factors, using adaptive interval queries. \citet{lau2026order} subsequently established the minimax rates for every fixed moment regime $k>1$ using adaptive threshold queries and also gave a two-stage construction using general 1-bit queries. \citet{lau2026open} asked whether the same rates can be attained when all queries are fixed in advance.

Independently and concurrently with our work, \citet{miao2026universal}
and~\citet{hu2026interaction} also answered the open problem posed by
\citet{lau2026open}.  Moreover, \citet{zhang2026valg} introduce VALG, an agentic system for research in machine learning theory, and present two theorem candidates for the same problem as part of its evaluation.  Together with ours, these approaches share the information flow illustrated in Figure~\ref{fig:information-flow}(b): localization and refinement bits are collected in parallel, and the decoded center is used only to interpret refinement bits that have already been collected.  They all use the localization strategy of \citet{lau2026order}.

The various refinement strategies are outlined here and discussed in more
detail in Appendix~\ref{app:concurrent-work}. \citet{miao2026universal}
gives grid-based constructions, while \citet{hu2026interaction} use a
successive-scale decomposition; the two VALG candidates are closely related
to these approaches.
Our decoder instead follows the centered refinement strategy of~\citet{lau2026order}: it uses~$c$ to select relevant cells, estimates each selected cell's contribution to the residual, and sums the resulting estimates; see Section~\ref{sec:stoquant-intuition}.  Overall, our approach can be viewed as taking the adaptive strategy in \citep{lau2026order} and ``making it non-adaptive'', whereas the other concurrent approaches are less directly connected to it.  Appendices~\ref{app:miao-comparison} and~\ref{app:hu-zhong-comparison} compare these refinement mechanisms in detail.  

As currently formulated, the concurrent refinement constructions include query families whose one-sets may have infinitely many interval components, whereas every realized one-set in our construction is a finite union of intervals.  We expect that these
refinement constructions could be modified to achieve the same finite-union property.  By contrast, attaining the order-optimal sample--interval tradeoff in Theorem~\ref{thm:interval-complexity} requires a separate adjustment to the localization strategy; see Section~\ref{sec:interval-complexity} for details.

\section{A Minimax-Optimal Non-Adaptive Estimator}
\label{sec:construction}

We now state the non-adaptive i.i.d.-query upper bound and describe the estimator (which we interchangeably call the \emph{decoder}) that
attains it.  The construction retains the localization--refinement
architecture of~\citet{lau2026order}, but moves the dependence on the decoded
center~$c$ from query selection to decoding (Figure~\ref{fig:information-flow}).  All localization and refinement
queries are fixed before any response is observed and can therefore be issued
and answered in parallel.  Only the decoder is two-stage: it first recovers
$c$ and then uses it to decode refinement bits that have already been
collected.
For clarity, this section analyzes the estimator under the prescribed numbers
of localization and scale-specific refinement queries specified below.  These
queries are independent but have type-dependent laws.
Appendices~\ref{app:construction} and~\ref{app:upper-bound} provide the omitted
verification and sample-count details, while
Appendix~\ref{app:localization-iid} converts these prescribed query counts into
the i.i.d. query law stated below.

\begin{theorem}[Minimax-Optimal 1-Bit Mean Estimation with i.i.d. Queries]
\label{thm:main}
Fix $k>1$, $\lambda\ge\sigma>\eps>0$, and
$\delta\in(0, 1)$.  There is a randomized non-adaptive 1-bit estimator that is 
$(\eps,\delta)$-PAC over $\calP(k,\lambda,\sigma)$ and uses $n$ i.i.d. queries, where
\begin{equation}
\label{eq:main-rate-compact}
    n
    =O\left(
      \log\frac{\lambda}{\sigma} \right)
      +
      \begin{cases}
      O_k \left( (\sigma/\eps)^2 \cdot \log(1/\delta) \right),&k>2,\\[0.2em]
      O \left( (\sigma/\eps)^2 \cdot \log(\sigma/\eps) \cdot \log(1/\delta) \right),&k=2,\\[0.2em]
      O_k \left( (\sigma/\eps)^{k/(k-1)} \cdot \log(1/\delta)  \right),&1<k<2.
      \end{cases}.
\end{equation}
Moreover, the one-set of each query is almost surely a union of $O(\lambda/\sigma)$ intervals.
\end{theorem}
The rate in each of the three cases in
\eqref{eq:main-rate-compact} matches, in the relevant parameter regimes and up to constants depending only on $k$, the adaptive upper and lower bounds~\citep[Theorems~5 and~9]{lau2026order}. Hence, non-adaptive queries attain the adaptive 1-bit minimax sample complexity.
The construction in this section uses $O(\lambda/\sigma)$ interval
components per query.  Theorem~\ref{thm:interval-complexity} in
Section~\ref{sec:interval-complexity} characterizes the minimax sample--interval
tradeoff, while Corollary~\ref{cor:optimal-components} therein identifies, up to constants depending
on $k$, the minimum interval budget that preserves the unrestricted 1-bit
minimax sample complexity.

\subsection{Localization}
\label{sec:localization-main}

Theorem~\ref{thm:main} uses a tightened form of the localization
strategy of~\citet[Theorem~16]{lau2026order}.\footnote{
Relative to the $50\sigma$ localization radius obtained by taking the midpoint of the interval constructed in the proof of
\citet[Theorem~16]{lau2026order}, the $8\sigma$
guarantee reduces the transferred scale~$\bar\sigma$ in \eqref{eq:moment-transfer} from $51\sigma$ to $9\sigma$.  Since the leading term of $n_i$ in \eqref{eq:main-ni} scales as $\bar\sigma^2$, this reduces its per-scale constant by $(51/9)^2$.
The same substitution tightens the implicit refinement constants in the concurrent constructions~\citep{miao2026universal,hu2026interaction}.} The following lemma
gives the localization guarantee needed for refinement and records the
interval complexity of its query one-sets.  
We note that the localization guarantee from \citet{lau2026order} would already suffice for Theorem~\ref{thm:main} (other than the i.i.d.~query claim), but we still state and prove the following lemma because we will need to adapt it to the $s$-interval restricted case in Section~\ref{sec:interval-complexity}.  Appendix~\ref{app:coding-localization}
gives the proof and specifies the hidden constants.

\begin{lemma}[Localization]
\label{lem:LS-localization}
For every $\lambda\ge\sigma>0$, $\delta_{\rm loc}\in(0,1/2)$, and every
distribution satisfying $\mu\in[-\lambda,\lambda]$ and
$\E[|X-\mu|]\le\sigma$, there is a randomized non-adaptive 1-bit protocol
that returns $c\in[-\lambda,\lambda]$ satisfying
$|c-\mu|\le 8\sigma$ with probability at least
$1-\delta_{\rm loc}$.  It uses
$O(\log(\lambda/\sigma)+\log(1/\delta_{\rm loc}))$ i.i.d.~queries, and the one-set of each query is a union of $O(\lambda/\sigma)$ intervals.
\end{lemma}

The moment assumption defining $\calP(k,\lambda,\sigma)$ implies
$\E[|X-\mu|]\le(\E[|X-\mu|^k])^{1/k}\le\sigma$  by Lyapunov's inequality, so the assumption $\E[|X-\mu|]\le\sigma$ is indeed satisfied in our setting.  On localization success, Minkowski's inequality gives the following moment bound required for the refinement analysis:
\begin{equation}
\label{eq:moment-transfer}
    \bigl(\E[|X-c|^k]\bigr)^{1/k}
    \le \bigl(\E[|X-\mu|^k]\bigr)^{1/k}+|\mu-c|
    \le\bar\sigma,
    \qquad\text{where }\bar\sigma=9\sigma.
\end{equation}
Because localization and refinement use disjoint samples and independent
randomness, conditional on localization success and the realized value of
$c$, the refinement analysis may treat $c\in[-\lambda,\lambda]$ and
$\E[|X-c|^k]\le\bar\sigma^k$ as fixed.

\subsection[Known-Center Refinement]
{Known-Center Refinement of \citet{lau2026order}}
\label{sec:stoquant-intuition}

To motivate our fixed-query construction, we first recall the center-dependent
refinement strategy of~\citet{lau2026order} for estimating the residual mean
$\E[X-c]=\mu-c$.  This subsection begins with a high-level overview.  The three
paragraphs thereafter respectively introduce the idea of cell decomposition, a first-moment identity supplied
by stochastic quantization, and a variance bound carrying a relevant tail factor.  Section~\ref{sec:fixed-query-difficulties} then identifies the
obstacles to preserving these properties non-adaptively, while
Sections~\ref{sec:queries-main} and~\ref{sec:decoder-after} formally define
our non-adaptive refinement queries and the corresponding decoder.

\paragraph{Overview.}
Once the localization center $c$ is known, the learner partitions a truncation window around
$c$ into cells whose widths grow with their distance from $c$, and estimates the cell contributions to the residual mean using randomized threshold queries (which we simplify below by using interval queries).  The performance guarantee relies on two properties:
each cell contribution is estimated without bias, and the variance of its
estimator decreases with the probability that $X$ reaches the distance from $c$
at which the cell lies.

\paragraph{Truncation and decomposition into cell contributions.}
Fix a localization output $c$ satisfying the conditions above.  For any finite
collection $\mathcal J$ of disjoint intervals, we have
\begin{equation}
\label{eq:cell-decomposition}
    \E\left[(X-c)\cdot\1\left\{X\in\bigcup_{J\in\mathcal J}J\right\}\right]
    =
    \sum_{J\in\mathcal J}
    \E\left[(X-c)\cdot\1\{X\in J\}\right].
\end{equation}
\citet{lau2026order} choose $\mathcal J$ to be a partition of
$[c-r,c+r]$, with $r$ large enough that the omitted tail contributes at most
$\eps/2$.  Their refinement problem therefore reduces to estimating the
\emph{cell contributions} $\theta_J \coloneqq \E\left[(X-c)\cdot\1\{X\in J\}\right]$
and summing them over the cells in the truncation window.
The partition is organized into dyadic scales.  Writing $\ell_i$ for the
scale-$i$ cell width, the widths double with $i$.  Cells near $c$ have the
finest width $\ell_1$, while, for $i\ge2$, cells of width $\ell_i$ lie at
distance $\Theta(\ell_i)$ from $c$.  The construction below specifically takes
$\ell_i=2^{i-1}\bar\sigma=\Theta(2^i\sigma)$.  Figure~\ref{fig:multiscale-cover}
later contrasts this idealized known-center geometry with a different cover selected from fixed $c$-independent grids.

\paragraph{Interval queries through stochastic quantization.}

We now consider how to estimate a single cell contribution $\theta_J$.
We describe randomized queries that yield a first-moment identity for this
purpose.  Write the cell as $J=[a,b)$, and let the \emph{left and right splits} $U^{\mathrm L}$ and $U^{\mathrm R}$ 
be uniform on $J$ and independent of $X$.  For every fixed $x\in\mathbb R$, it is a simple calculation to show that averaging over the uniform splits gives the following identity \citep[Appendix~A, Step~4]{lau2026order}; see also Lemma~\ref{lem:cell-stochastic} in Appendix~\ref{app:construction}, where this identity is stated and proved:
\begin{equation}
\label{eq:stoquant-query-main}
    (a-c) \cdot \E_{U^{\mathrm L}}
      \left[\1\{x\in[a,U^{\mathrm L})\}\right]
    +(b-c) \cdot \E_{U^{\mathrm R}}
      \left[\1\{x\in[U^{\mathrm R},b)\}\right]
    =(x-c) \cdot \1\{x\in J\},
\end{equation}
in analogy with the identity $\E[ \1\{U \le x\}] = x$ for $x \in [0,1]$ and $U \sim {\rm Unif}(0,1)$.  
Weighted empirical averages of these two indicators therefore produce an
unbiased cell estimator $\widehat\theta_J$ satisfying
$\E[\widehat\theta_J]=\theta_J$.
Summing these cell estimators over the partition estimates the truncated
residual mean in \eqref{eq:cell-decomposition}.

\paragraph{The tail factor in the variance and sample allocation.}
Unbiasedness alone does not determine how many samples are needed at each
scale.  The key additional property is that the variance of a cell estimator
decreases with the probability that $X$ reaches the corresponding cell.  
Suppose that a cell $J=[a_J,b_J)$ lies at a positive distance $L=\dist(c,J) \coloneqq
    \inf_{x\in J}|x-c| > 0$ from $c$ and
has width $\Theta(L)$.  The interval indicators
in \eqref{eq:stoquant-query-main} vanish unless $X\in J$, and hence unless
$|X-c|\ge L$.  Since the endpoint weights satisfy
$|a_J-c|,|b_J-c|=O(L)$ and $\Var(\operatorname{Ber}(p))\le p$,
independence of the samples gives
$\Var(\widehat\theta_J)
    =
    O\left(
        (L^2/n_J) \cdot\Pr(|X-c|\ge L)
    \right)
$
when $n_J$ samples are used for each query type.  Hence, the variance carries
the tail factor $\Pr(|X-c|\ge L)$, rather than merely the worst-case bound
$O(L^2/n_J)$.  We refer to this property as \emph{tail-local variance}.  Together with the
first-moment identity for $\widehat\theta_J$, it permits the order-optimal
sample allocation of~\citet{lau2026order}.

\subsection{Three Obstacles for Non-Adaptive Queries}
\label{sec:fixed-query-difficulties}
This subsection gives a conceptual roadmap for converting the known-center
refinement reviewed in Section~\ref{sec:stoquant-intuition} into a
non-adaptive construction.  That known-center method selects each cell $J$ only after $c$
has been decoded, whereas our queries must be fixed before any localization
response is observed.  Our goal is to preserve its two key properties:
(i) the cell-wise first-moment identity, which recovers each selected cell contribution
$\theta_J$ without bias, and
(ii) tail-local variance, which weights the variance by the probability that $X$ reaches the corresponding distance scale.  We identify three resulting obstacles and preview how our construction
resolves them.  Sections~\ref{sec:queries-main}
and~\ref{sec:decoder-after} formally define the queries and decoder respectively.

\begin{enumerate}[leftmargin=2.2em]
\item \emph{The relevant cells are not yet known.}  At each scale $i$, we
include in a finite bank every grid cell that the decoder might later select,
forming $\calD_i^{\rm bank}$ in
\eqref{eq:main-query-bank} below.  Once $c$ is decoded, the decoder selects finest-scale cells near $c$ and, farther away, selects cells whose widths are comparable to their distances
from $c$; see Figure~\ref{fig:multiscale-cover} and \eqref{eq:main-cover-map}.

\item \emph{A single response mixes contributions from many cells.}  For each query,
we assign independent Rademacher signs, taking values $\pm1$ with equal
probability, to the bank cells.  The sign of each cell
determines whether its candidate interval is included in the scale-$i$
query; see Figure~\ref{fig:query-filter}(a) and \eqref{eq:main-query}.
Multiplying the returned query bit by the sign of a decoder-selected cell
\emph{isolates that cell in expectation} (i.e., the ``interference'' from other cells vanishes on average); see \eqref{eq:simple-demodulation}. Combining the two weighted orientations through \eqref{eq:stoquant-query-main} then recovers its cell contribution.

\item \emph{Near-center intervals can remove the tail factor from the variance.}  Although the 
preceding step recovers each selected
cell contribution $\theta_J$ in expectation, the same random-union query may contain a candidate interval close
to $c$.  A near-center observation can then make the conditional variance of the decoded
statistic of order $\ell_i^2$, rather than the desired
$O(\ell_i^2\cdot\Pr(|X-c|\ge\ell_i))$.  The decoder restores the tail-local variance property by discarding responses whose queries include such a near-center interval and reweighting the retained
responses to preserve the first-moment identity; see Figure~\ref{fig:query-filter}(b) and
Section~\ref{sec:filter-main}.
\end{enumerate}

\subsection{The Fixed Refinement Queries}
\label{sec:queries-main}

We first give a high-level overview of the refinement query construction
and preview how the decoder will use the resulting queries after localization. 
The remaining paragraphs formally specify the fixed grids, query banks, and query law, and the final paragraph proves the basic sign-isolation identity. 
The decoder itself is formally defined in
Section~\ref{sec:decoder-after}.

\paragraph{Overview.}
At each queried scale $i$, we prepare a finite bank $\calD_i^{\rm bank}$ 
containing every cell that the decoder might later
select; see \eqref{eq:main-query-bank}.  A single scale-$i$ query draws one Rademacher sign 
for each bank cell and includes a randomly split subinterval (analogous to \eqref{eq:stoquant-query-main}) of every cell whose sign is~$+1$.  These multiscale banks replace the above-mentioned center-dependent
placement of refinement cells; each query is formed from the fixed bank
before $c$ is known, whereas the decoder later selects only the
scale-appropriate cells.
Figure~\ref{fig:query-filter}(a) illustrates one realized query, while
Figure~\ref{fig:query-filter}(b) previews the decoder overlay formalized in
Section~\ref{sec:filter-main}. Figure~\ref{fig:multiscale-cover}
previews how the decoder selects a cover from the fixed banks, and the
selection rule is formally defined in Section~\ref{sec:cover-main}.

\paragraph{Scales and fixed grids.}
As outlined in Section~\ref{sec:stoquant-intuition}, the existing known-center estimator
truncates the residual $X-c$ to a centered window and uses cells whose widths double
with their distance from $c$.  We use the same geometric sequence of widths,
but place a \emph{fixed grid} at each scale before $c$ is known.  We set
$\ell_i=2^{i-1}\bar\sigma$ as before, and again choose a truncation width so that the omitted
tail contributes at most $\eps/2$ to the mean.  Specifically, we claim that the choice
\begin{equation}
\label{eq:main-radii}
    \imax
    =
    \min\left\{
      i\ge1:
      \frac{\bar\sigma^k}{\ell_i^{k-1}}\le\frac{\eps}{2}
    \right\},
    \qquad\text{and}\qquad
    r=\ell_{\imax}. 
\end{equation}
suffices to ensure 
$\E[|X-c|\cdot\1\{|X-c|>r\}]\le
\bar\sigma^k/r^{k-1}\le\eps/2$.
Appendix~\ref{sec:conditioning} gives the derivation of this inequality (see \eqref{eq:moment-tail-bound-detail} therein).
It therefore suffices to estimate the residual $X-c$ on a suitable \emph{cover}, namely, a disjoint collection of cells whose union contains
$[c-r,c+r]$.  

For every $i\ge1$, let
$\calD_i=\{[m\ell_i,(m+1)\ell_i):m\in\mathbb Z\}$ (with $\ell_i=2^{i-1}\bar\sigma$).
These are fixed nested dyadic grids (see the gray regions in Figure~\ref{fig:multiscale-cover} below).  The scale-$(i+1)$ cell containing a
scale-$i$ cell $J$ is referred to as its \emph{parent}
$\operatorname{par}(J)$.  Only scales $i=1,\dotsc,\imax$ will be queried. 

Since $c$ is unknown in advance, we need to consider the entire set of \emph{potentially relevant cells} at each scale $i=1,\dotsc,\imax$, which we refer to as the \emph{bank} of cells.  Specifically, the bank includes all scale-$i$ cells that intersect the enlarged
range $[-\lambda-3\ell_i,\lambda+3\ell_i]$:
\begin{equation}
\label{eq:main-query-bank}
    \calD_i^{\rm bank}
    =
    \left\{
        J\in\calD_i:
        J\cap[-\lambda-3\ell_i,\lambda+3\ell_i]\ne\varnothing
    \right\}.
\end{equation}
The margin $3\ell_i$ turns out to be sufficient, since Lemma~\ref{lem:cover-main} will show that every
scale-$i$ cell $J$ later selected by the decoder intersects
$(c-3\ell_i,c+3\ell_i)$.  Since $c\in[-\lambda,\lambda]$, this implies
$J\in\calD_i^{\rm bank}$, so the bank contains every potentially selected
cell, regardless of the localization output.  We show in Appendix~\ref{sec:query-details} that the size of the bank scales as $|\calD_i^{\rm bank}|=O(\max\{1,\lambda/\ell_i\})=O(\lambda/\sigma)$. 

While using an infinite grid would preserve the same statistical identities, a realized query one-set
would then have infinitely many interval components.  In contrast, using the finite bank leads to
the $O(\lambda/\sigma)$ interval bound in Theorem~\ref{thm:main}.
Section~\ref{sec:interval-complexity} studies the more general tradeoff between the number of intervals and the sample complexity.

\paragraph{One query at a fixed scale.}
We now describe the mechanism for forming random-union queries across the entire bank, as we previewed in
Section~\ref{sec:fixed-query-difficulties}.  At scale $i$, let
$t=1,\dotsc,n_i$ index independent draws of the query law, where $n_i$ will be
specified in Section~\ref{sec:performance-main}.  For each draw $t$ and orientation
$d\in\{\mathrm L,\mathrm R\}$, we draw one \emph{common relative split}
$\zeta_{i,t}^{d}\sim\operatorname{Unif}[0,1]$ and i.i.d.~(across $J\in\calD_i^{\rm bank}$) \emph{Rademacher signs}
$\eta_{J,t}^{d}\sim\operatorname{Unif}\{-1,+1\}$.  
The relative splits are independent across $(i,t,d)$, the signs are
independent across $(J,i,t,d)$, and all signs are independent of all split
fractions.  For $J=[a_J,b_J)$, set
$U_{J,t}^{d}=a_J+\zeta_{i,t}^{d}\ell_i$ and define the following candidate subintervals in analogy with \eqref{eq:stoquant-query-main}:
\begin{equation}
\label{eq:candidate_subintervals}
    A_{J,t}^{\mathrm L}
    =
    [a_J,U_{J,t}^{\mathrm L}),
    \qquad\text{and}\qquad
    A_{J,t}^{\mathrm R}
    =
    [U_{J,t}^{\mathrm R},b_J).
\end{equation}
We will specify a collection of queries $Q_{i,t}^d$ indexed by $(i,t,d)$.  Within each such query, the same relative split
$\zeta_{i,t}^d$ is shared across all bank cells, implying that the split
points $U_{J,t}^d$ are dependent across cells, while each
$U_{J,t}^d$ remains marginally uniform on its cell $J$.\footnote{Sharing the split
fraction reduces the public randomness from one split variable per cell to
one per query.  Cross-cell independence is unnecessary (though would also suffice given the required public randomness) because stochastic
quantization is applied cell-by-cell and the resulting expectations are
combined by linearity; see Appendix~\ref{sec:stoquant-details}.}
The corresponding query is the union of candidate intervals carrying sign $+1$:
\begin{equation}
\label{eq:main-query}
    Q_{i,t}^{d}(x)
    =
    \1\left\{
        x\in
        \bigcup_{\substack{
            J\in\calD_i^{\rm bank}\\
            \eta_{J,t}^{d}=+1
        }}
        A_{J,t}^{d}
    \right\}.
\end{equation}
On an independent observation $X_{i,t}^{d}\sim P$, the 1-bit query returns
$Y_{i,t}^{d}=Q_{i,t}^{d}(X_{i,t}^{d})$.  Because there is at most one
candidate interval per bank cell, each one-set contains at most
$|\calD_i^{\rm bank}|=O(\lambda/\sigma)$ bounded intervals.

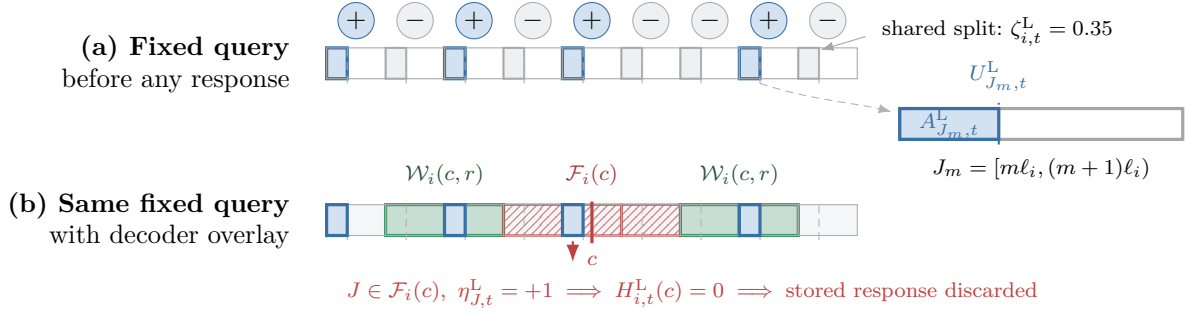
\begin{figure}[htbp]
\centering
\begin{tikzpicture}[x=0.78cm,y=0.68cm,>=Latex,font=\small]
  \def\split{0.35}

  % Panel (a): the query is fixed before the center is known.
  \node[anchor=east,align=right] at (-0.45,5.30)
    {\textbf{(a) Fixed query}\\before any response};

  \foreach \j in {0,2,4,7} {
    \fill[querylight] (\j,5.00) rectangle ({\j+\split},5.60);
    \draw[queryblue,very thick]
      (\j,5.00) rectangle ({\j+\split},5.60);
  }
  \foreach \j in {1,3,5,6,8} {
    \fill[softgray] (\j,5.00) rectangle ({\j+\split},5.60);
    \draw[gray!60,thick]
      (\j,5.00) rectangle ({\j+\split},5.60);
  }
  \foreach \j in {0,...,8} {
    \draw[gray!70] (\j,5.00) rectangle ({\j+1},5.60);
    \draw[gray!65,densely dashed]
      ({\j+\split},4.91)--({\j+\split},5.69);
  }
  \foreach \j in {0,2,4,7} {
    \node[circle,draw=queryblue,fill=querylight,inner sep=1.2pt]
      at ({\j+0.5},6.13) {$+$};
  }
  \foreach \j in {1,3,5,6,8} {
    \node[circle,draw=gray!65,fill=softgray,inner sep=1.2pt]
      at ({\j+0.5},6.13) {$-$};
  }

 \node[anchor=west,font=\scriptsize] at (9.25,5.92)
  {shared split: $\zeta_{i,t}^{\mathrm L}=0.35$};
\draw[->,gray!70] (9.14,5.82)--(8.38,5.56);

  % Inset: one magnified bank cell.
  \draw[gray!60,densely dashed,->]
    (7.35,4.91)
    .. controls (8.15,4.60) and (8.95,4.53)
    .. (9.58,4.36);

  \fill[querylight]
    (9.72,3.82) rectangle (11.40,4.42);
  \draw[gray!70,very thick]
    (9.72,3.82) rectangle (14.52,4.42);
  \draw[queryblue,very thick]
    (9.72,3.82) rectangle (11.40,4.42);
  \draw[queryblue,densely dashed,thick]
    (11.40,3.72)--(11.40,4.53);

  \node[queryblue,font=\scriptsize,anchor=south]
    at (11.40,4.57) {$U_{J_m,t}^{\mathrm L}$};
  \node[queryblue,font=\scriptsize]
    at (10.56,4.12) {$A_{J_m,t}^{\mathrm L}$};
  \node[font=\scriptsize,anchor=north]
    at (12.12,3.66)
    {$J_m=[m\ell_i,(m+1)\ell_i)$};

  % Panel (b): the same query is interpreted after localization.
  \node[anchor=east,align=right] at (-0.45,2.25)
    {\textbf{(b) Same fixed query}\\with decoder overlay};

  \foreach \j in {0,...,8} {
    \fill[softgray!60] (\j,1.95) rectangle ({\j+1},2.55);
  }
  \foreach \j in {1,2,6,7} {
    \fill[covergreen!30] (\j,1.95) rectangle ({\j+1},2.55);
    \draw[covergreen,very thick]
      (\j,1.95) rectangle ({\j+1},2.55);
  }
  \foreach \j in {3,4,5} {
    \fill[pattern=north east lines,pattern color=filterred!70]
      (\j,1.95) rectangle ({\j+1},2.55);
    \draw[filterred,thick]
      (\j,1.95) rectangle ({\j+1},2.55);
  }
  \foreach \j in {0,...,8} {
    \draw[gray!65] (\j,1.95) rectangle ({\j+1},2.55);
    \draw[gray!60,densely dashed]
      ({\j+\split},1.86)--({\j+\split},2.64);
  }
  \foreach \j in {0,2,4,7} {
    \fill[querylight] (\j,1.95) rectangle ({\j+\split},2.55);
    \draw[queryblue,very thick]
      (\j,1.95) rectangle ({\j+\split},2.55);
  }

  \node[covergreen!55!black,font=\scriptsize,anchor=south]
    at (2.00,2.72) {$\calW_i(c,r)$};
  \node[filterred!85!black,font=\scriptsize,anchor=south]
    at (4.50,2.72) {$\calF_i(c)$};
  \node[covergreen!55!black,font=\scriptsize,anchor=south]
    at (7.00,2.72) {$\calW_i(c,r)$};

  \draw[filterred,very thick] (4.50,1.80)--(4.50,2.61);
  \node[filterred,font=\scriptsize,anchor=north]
    at (4.50,1.76) {$c$};

  \draw[filterred,->,thick]
    ({4+\split/2},1.84)--({4+\split/2},1.43);
  \node[filterred,font=\scriptsize,anchor=north]
    at (6.20,1.31)
    {$J\in\calF_i(c),\ \eta_{J,t}^{\mathrm L}=+1
      \implies H_{i,t}^{\mathrm L}(c)=0
      \implies
      \text{stored response discarded}$};
\end{tikzpicture}

\caption{A fixed refinement query before and after localization.
(a) A shared split determines one candidate interval in each bank cell,
and the $+1$ signs select the intervals shown in blue for inclusion in the
query one-set.  The inset magnifies one included candidate interval.
(b) After $c$ is decoded, the decoder overlays the  query on the
green cover cells $\calW_i(c,r)$ and red-hatched filter cells
$\calF_i(c)$; see \eqref{eq:main-cover-map} and
\eqref{eq:main-filter}.  The blue candidate interval lying in a filter cell
is included in the query because $\eta_{J,t}^{\mathrm L}=+1$.
Consequently, the filter indicator $H_{i,t}^{\mathrm L}(c)$ given in
\eqref{eq:filter-indicator} equals zero, so the
decoder discards the entire stored response, regardless of the location
of $X$.}
\label{fig:query-filter}
\end{figure}

\paragraph{Isolating a cell in expectation with Rademacher signs.}
\label{sec:demodulation-intuition}

The Rademacher signs allow the decoder to isolate any specified cell from the
bank-wide response in expectation (i.e., remove ``interference'' from other cells on average).  Because the bank cells are pairwise disjoint, 
conditional on the observation and the common relative split $\zeta_{i,t}^d$, 
at most one candidate interval contains $X_{i,t}^d$.  If that
interval belongs to cell $K$, then
$Y_{i,t}^d=\1\{\eta_{K,t}^d=+1\}=(1+\eta_{K,t}^d)/2$.  This implies that the bank-wide
response can be written (in terms of the sets from \eqref{eq:candidate_subintervals} and the Rademacher variables) as
\[
    Y_{i,t}^d
    =
    \sum_{K\in\calD_i^{\rm bank}}
      \1\{X_{i,t}^d\in A_{K,t}^d\}
      \cdot\frac{1+\eta_{K,t}^d}{2}.
\]
Taking the expectation of $2Y_{i,t}^d\eta_{J,t}^d$ with respect to the signs
and using
$\E_\eta[(1+\eta_{K,t}^d) \cdot \eta_{J,t}^d]=\1\{K=J\}$ therefore gives
\begin{equation}
\label{eq:simple-demodulation}
    \E_\eta\left[
      2Y_{i,t}^{d}\cdot\eta_{J,t}^{d}
      \mid X_{i,t}^{d},\zeta_{i,t}^{d}
    \right]
    =
    \1\{X_{i,t}^{d}\in A_{J,t}^{d}\}
    \qquad\text{for every }J\in\calD_i^{\rm bank}.
\end{equation}

\subsection{Decoding After Localization}
\label{sec:decoder-after}

In this subsection, we formally define the refinement decoder.  We begin with a
high-level overview.  The subsequent three parts of this subsection develop, respectively, the cover-selection rule, the sign-based recovery identity, and the
filtering-and-reweighting rule.

\paragraph{Overview.}
All refinement bits in Section~\ref{sec:queries-main} are collected without
using $c$.  After all responses have been collected and the localization
decoder has recovered $c$, the refinement decoder performs three operations:
it selects from the prequeried banks a cover resembling the
known-center partition of Section~\ref{sec:stoquant-intuition}; 
it uses the stored Rademacher signs to isolate the selected cells in expectation; 
and it discards responses whose queries include a near-center interval.  
Together with the stochastic-quantization identity, the first two operations
yield unbiased estimates of the selected cell contributions $\theta_J$; the third restores tail-local variance.  All three operations are performed only at the
decoder and do not alter any query.

\paragraph{Selecting a geometric cover.}
\label{sec:cover-main}

The decoder approximates the center-dependent geometry of
Section~\ref{sec:stoquant-intuition} by selecting prequeried cells that form a
disjoint cover of the truncation window $[c-r,c+r]$ (recalling $r$ from \eqref{eq:main-radii}). We first describe the
selection rule informally and illustrate the resulting geometry in
Figure~\ref{fig:multiscale-cover}; equation~\eqref{eq:main-cover-map} below
then gives its formal definition.

Informally, first consider the scale-1 (i.e., narrowest) cells intersecting the truncation window.
Moving away from $c$, merge cells along the dyadic tree until one further
merge would make a cell wider than its distance from $c$.  Cells near $c$
therefore remain at the finest scale, while cells farther away form
progressively wider blocks.  This reproduces the qualitative geometry of the
centered partition from Section~\ref{sec:stoquant-intuition}, while still using only cells that were queried before $c$ was known.
Figure~\ref{fig:multiscale-cover} illustrates the fixed banks at three
representative scales, their decoder-selected overlay, and the idealized
known-center partition they mimic.  Note that since the available cells come from
grids that are fixed in advance, the selected cover is not necessarily symmetric with respect to $c$.

\begin{figure}[htbp]
\centering
\begin{tikzpicture}[x=0.265cm,y=0.61cm,>=Latex,font=\small]
  \def\cc{0.6}

  % Scale 1: the finest full grid, finite bank, and selected cells.
  \node[anchor=east,font=\scriptsize] at (-24.75,5.68)
    {$i=1$ \;($\ell_1$-cells)};
  \fill[softgray] (-24,5.50) rectangle (24,5.86);
  \fill[querylight!58] (-9,5.50) rectangle (9,5.86);
  \foreach \x in {-24,...,23} {
    \draw[gray!52,line width=0.28pt]
      (\x,5.50) rectangle ({\x+1},5.86);
  }
  \draw[queryblue,very thick] (-9,5.50) rectangle (9,5.86);
  \foreach \a in {-2,-1,0,1,2,3} {
    \fill[coverorange!32] (\a,5.50) rectangle ({\a+1},5.86);
    \draw[coverorange,very thick] (\a,5.50) rectangle ({\a+1},5.86);
  }
  \node[anchor=north,font=\scriptsize,queryblue!80!black] at (-9,6.56)
    {$-\lambda-3\ell_1$};
  \node[anchor=north,font=\scriptsize,queryblue!80!black] at (9,6.56)
    {$\lambda+3\ell_1$};

  % Scale 2: cells double in width and the bank margin grows accordingly.
  \node[anchor=east,font=\scriptsize] at (-24.75,4.58)
    {$i=2$ \;($\ell_2$-cells)};
  \fill[softgray] (-24,4.40) rectangle (24,4.76);
  \fill[querylight!58] (-12,4.40) rectangle (12,4.76);
  \foreach \x in {-24,-22,...,22} {
    \draw[gray!52,line width=0.28pt]
      (\x,4.40) rectangle ({\x+2},4.76);
  }
  \draw[queryblue,very thick] (-12,4.40) rectangle (12,4.76);
  \foreach \a/\b in {-4/-2,4/6,6/8} {
    \fill[covergreen!31] (\a,4.40) rectangle (\b,4.76);
    \draw[covergreen,very thick] (\a,4.40) rectangle (\b,4.76);
  }
  \node[anchor=north,font=\scriptsize,queryblue!80!black] at (-12,5.46)
    {$-\lambda-3\ell_2$};
  \node[anchor=north,font=\scriptsize,queryblue!80!black] at (12,5.46)
    {$\lambda+3\ell_2$};

  % Scale 3: the same fixed-grid construction at a coarser scale.
  \node[anchor=east,font=\scriptsize] at (-24.75,3.48)
    {$i=3$ \;($\ell_3$-cells)};
  \fill[softgray] (-24,3.30) rectangle (24,3.66);
  \fill[querylight!58] (-18,3.30) rectangle (18,3.66);
  \foreach \x in {-24,-20,...,20} {
    \draw[gray!52,line width=0.28pt]
      (\x,3.30) rectangle ({\x+4},3.66);
  }
  \draw[queryblue,very thick] (-18,3.30) rectangle (18,3.66);
  \foreach \a/\b in {-8/-4,8/12} {
    \fill[quantpurple!22] (\a,3.30) rectangle (\b,3.66);
    \draw[quantpurple,very thick] (\a,3.30) rectangle (\b,3.66);
  }
  \node[anchor=north,font=\scriptsize,queryblue!80!black] at (-18,4.36)
    {$-\lambda-3\ell_3$};
  \node[anchor=north,font=\scriptsize,queryblue!80!black] at (18,4.36)
    {$\lambda+3\ell_3$};

  % Overlay the selected cells to show the resulting disjoint cover.
  \node[anchor=east,align=right,font=\scriptsize] at (-24.75,2.12)
    {decoder-selected\\cover~$\calW(c,r)$};
  \foreach \a/\b in {-8/-4,8/12} {
    \fill[quantpurple!22] (\a,1.94) rectangle (\b,2.30);
    \draw[quantpurple,very thick] (\a,1.94) rectangle (\b,2.30);
  }
  \foreach \a/\b in {-4/-2,4/6,6/8} {
    \fill[covergreen!31] (\a,1.94) rectangle (\b,2.30);
    \draw[covergreen,very thick] (\a,1.94) rectangle (\b,2.30);
  }
  \foreach \a in {-2,-1,0,1,2,3} {
    \fill[coverorange!32] (\a,1.94) rectangle ({\a+1},2.30);
    \draw[coverorange,very thick] (\a,1.94) rectangle ({\a+1},2.30);
  }

  % Idealized centered geometry available when the center determines the cells.
  \node[anchor=east,align=right,font=\scriptsize] at (-24.75,1.02)
    {known-center partition};
  \fill[quantpurple!22] ({\cc-8},0.84) rectangle ({\cc-4},1.20);
  \draw[quantpurple,very thick] ({\cc-8},0.84) rectangle ({\cc-4},1.20);
  \fill[quantpurple!22] ({\cc+4},0.84) rectangle ({\cc+8},1.20);
  \draw[quantpurple,very thick] ({\cc+4},0.84) rectangle ({\cc+8},1.20);
  \fill[covergreen!31] ({\cc-4},0.84) rectangle ({\cc-2},1.20);
  \draw[covergreen,very thick] ({\cc-4},0.84) rectangle ({\cc-2},1.20);
  \fill[covergreen!31] ({\cc+2},0.84) rectangle ({\cc+4},1.20);
  \draw[covergreen,very thick] ({\cc+2},0.84) rectangle ({\cc+4},1.20);
  \foreach \a in {-2,-1,0,1} {
    \fill[coverorange!32] ({\cc+\a},0.84)
      rectangle ({\cc+\a+1},1.20);
    \draw[coverorange,very thick] ({\cc+\a},0.84)
      rectangle ({\cc+\a+1},1.20);
  }

  % Align the decoded center and truncation window across all rows.
  \draw[filterred,densely dashed,thick] (\cc,0.73)--(\cc,6.00);

  \draw[<->,gray!75,thick] ({\cc-8},0.44)--({\cc+8},0.44);
  \node[anchor=north,font=\scriptsize] at ({\cc-8},0.38) {$c-r$};
  \node[anchor=north,font=\scriptsize,filterred] at (\cc,0.38) {$c$};
  \node[anchor=north,font=\scriptsize] at ({\cc+8},0.38) {$c+r$};
\end{tikzpicture}
\caption{Fixed grids, finite query banks, and the decoder-selected cover.  
At scale $i$, the gray cells form the fixed grid $\calD_i$, and the cells
intersecting $[-\lambda-3\ell_i,\lambda+3\ell_i]$ form the finite query bank
$\calD_i^{\rm bank}$ (blue). 
After decoding $c$, the decoder selects the cover cells $\calW_i(c,r)$ 
(orange, green, and purple for $i=1,2,3$).  
The combined cover $\calW(c,r)$ consists of pairwise disjoint cells whose
union contains $[c-r,c+r]$, with fine cells near $c$ and wider cells farther
away.  Lemma~\ref{lem:cover-main} states the precise guarantees.
The bottom row illustrates the centered partition used by
\citet{lau2026order} when $c$ is known; see Section~\ref{sec:stoquant-intuition}.  The decoder-selected cover mimics this multiscale structure but need not be symmetric with respect to $c$.}
\label{fig:multiscale-cover}
\end{figure}
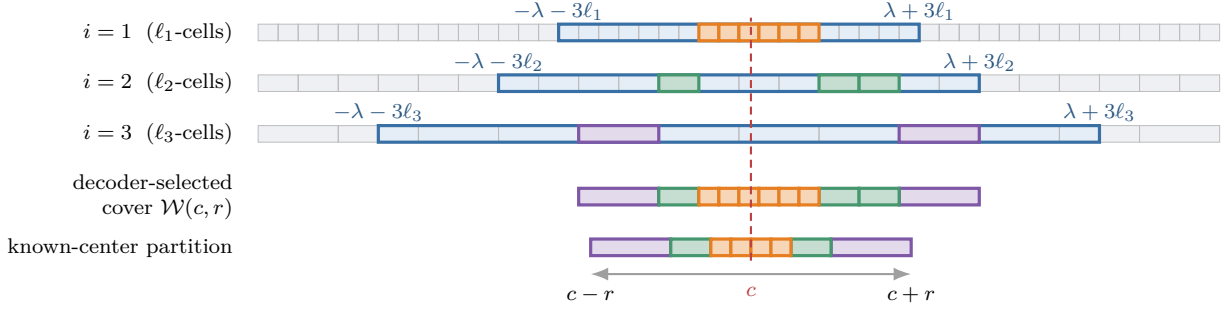

We now formalize the cover selection rule.  Writing $\len(J)$ for the length of an interval $J$, we call a dyadic cell
$J$ \emph{admissible} if $\len(J)\le\dist(c,J)$, and let
\begin{equation}
\label{eq:cover-intersecting-cells}
    \mathcal L(c,r)
    =
    \{B\in\calD_1:B\cap[c-r,c+r]\ne\varnothing\}
\end{equation}
be the scale-1 cells intersecting the truncation window.  The full grid
$\calD_1$ is used here only to ensure coverage; scale-1 cells far from $c$
will be replaced by coarser ancestors.  For each
$B\in\mathcal L(c,r)$, define
\begin{equation}
\label{eq:cover-selection-rule}
    J_c^\star(B)
    =
    \begin{cases}
        B,
        & \text{if $B$ is not admissible},\\
        \text{the coarsest admissible dyadic ancestor of $B$},
        & \text{if $B$ is admissible}.
    \end{cases}
\end{equation}
To see that the ancestor in the second case is well defined, note that $B$ itself is admissible.  Moreover, by~\eqref{eq:cover-intersecting-cells}, $B$ intersects $[c-r,c+r]$, and hence every ancestor $J\supseteq B$ does as well.  Hence, every admissible ancestor satisfies
\[
    \len(J)\le\dist(c,J)\le r=\ell_{\imax}.
\]
This implies that every admissible ancestor has scale at most $\imax$, and hence the
coarsest admissible ancestor exists.  For an admissible $B$, the decoder
therefore ascends the dyadic tree to its coarsest admissible ancestor, whereas a
non-admissible $B$ remains unchanged at scale~$1$.  

The distinct selected cells form the cover.  Grouping them by scale (equivalently, by width), we define
\begin{equation}
\label{eq:main-cover-map}
    \calW(c,r)
    =
    \{J_c^\star(B):B\in\mathcal L(c,r)\}
    \quad \text{and} \quad
    \calW_i(c,r)
    =
    \{J\in\calW(c,r):\len(J)=\ell_i\}.
\end{equation}
We call the members of $\calW(c,r)$ \emph{cover cells}.  Every cover cell has scale at most $\imax$ and therefore belongs to exactly one of the collections $\calW_i(c,r)$.  The geometric properties used in the analysis are summarized in the following lemma and proved in Appendix~\ref{sec:cover-details}.

\begin{lemma}[Geometry of the cover]
\label{lem:cover-main}
For every $c\in[-\lambda,\lambda]$, the cells in $\calW(c,r)$ are pairwise
disjoint and cover $[c-r,c+r]$.  For every
$i\in\{1,\ldots,\imax\}$ and every
$J=[a_J,b_J)\in\calW_i(c,r)$, we have
\[
    J\cap(c-3\ell_i,c+3\ell_i)\ne\varnothing,
    \qquad
    \max\{|a_J-c|,|b_J-c|\}<4\ell_i.
\]
Moreover,
\[
    |\calW_i(c,r)| \le 9,
    \qquad
    \calW_i(c,r)\subseteq\calD_i^{\rm bank}.
\]
 Finally, for $i\ge2$, each cell
$J\in\calW_i(c,r)$ satisfies
\[
    \ell_i\le\dist(c,J)<3\ell_i.
\]
\end{lemma}

\paragraph{Recovering cover-cell contributions with Rademacher signs.}
\label{sec:decoder-main}
For each cell $J=[a_J,b_J)\in\calW_i(c,r)$ in the cover, define the \emph{endpoint weights}
$w_J^{\mathrm L}(c)=a_J-c$ and
$w_J^{\mathrm R}(c)=b_J-c$.  To establish the first-moment identity, for each scale
$i=1,\ldots,\imax$, repetition index $t=1,\ldots,n_i$, and orientation
$d\in\{\mathrm L,\mathrm R\}$, define the following provisional decoded
variable:\footnote{This provisional variable satisfies the desired
first-moment identity but not yet the required tail-local variance bound. 
Equation~\eqref{eq:main-decoded-variable} gives the final choice of decoded variable.}
\begin{equation}
\label{eq:weighted-demodulation}
    \widetilde V_{i,t}^{d}(c)
    =
    2Y_{i,t}^{d}
    \sum_{J\in\calW_i(c,r)}
      w_J^{d}(c)\cdot\eta_{J,t}^{d}.
\end{equation}
By linearity of expectation, we may apply
\eqref{eq:simple-demodulation} to each summand in
\eqref{eq:weighted-demodulation} to obtain 
\[
    \E_\eta\left[
      \widetilde V_{i,t}^{d}(c)
      \mid X_{i,t}^{d},\zeta_{i,t}^{d}
    \right]
    =
    \sum_{J\in\calW_i(c,r)}  w_J^{d}(c)\cdot \1\{X_{i,t}^{d}\in A_{J,t}^{d}\}.
\]
Averaging next over the random split
$\zeta_{i,t}^d$ from Section~\ref{sec:queries-main} and applying the
cell-wise stochastic quantization identity
\eqref{eq:stoquant-query-main} gives
\begin{equation}
\label{eq:provisional-first-moment}
\begin{aligned}
    \E\left[
      \widetilde V_{i,t}^{\mathrm L}(c)
      +\widetilde V_{i,t}^{\mathrm R}(c)
    \right]
    =
    \sum_{J\in\calW_i(c,r)}
      \E\left[(X-c)\cdot\1\{X\in J\}\right]
    =
    \sum_{J\in\calW_i(c,r)}\theta_J,
\end{aligned}
\end{equation}
thus giving the desired sum of cell-specific means.

\paragraph{Discarding responses influenced by near-center intervals.}
\label{sec:filter-main}

For $i\ge2$, a scale-$i$ query in the centered refinement construction~\citep{lau2026order} is supported only at distance
$\Theta(\ell_i)$ from~$c$.  This property ensures that the
second moment of its decoded statistic is of order
$\ell_i^2\cdot\Pr(|X-c|\ge\ell_i)$.  Our decoder aims to preserve this scaling.
A query from the fixed bank, however, is a random union over many cells and may
include a candidate interval arbitrarily close to $c$.  The obstruction is visible by conditioning on the relative split ($\zeta_{i,t}^{d}$ from Section~\ref{sec:queries-main}).  Suppose that an observation $x$ falls in such a near-center interval.  Whether the returned bit equals one is then determined by the Rademacher sign of that cell, which is independent of all other signs.  Since the cover-cell weights are of order $\ell_i$, the conditional sign mean and variance are
\[
    \E_\eta[\widetilde V_{i,t}^{d}(c)\mid X=x,\zeta_{i,t}^d]=0,
    \qquad
    \Var_\eta(\widetilde V_{i,t}^{d}(c)\mid X=x,\zeta_{i,t}^d)
    =
    2\sum_{K\in\calW_i(c,r)}\bigl(w_K^d(c)\bigr)^2
    =
    O(\ell_i^2).
\]
Hence, although the first-moment identity is preserved, the conditional variance may lack the tail-probability factor discussed in Section~\ref{sec:stoquant-intuition}. This missing factor would make the sample allocation in Section~\ref{sec:performance-main} insufficient to obtain an order-optimal PAC guarantee.  To restore this factor, we discard a
response whenever its query includes a candidate interval from a scale-$i$ bank cell within distance $\ell_i$ of $c$; Figure~\ref{fig:query-filter}(b) illustrates this rule. At the finest scale $\ell_1$, a direct $O(\ell_1^2)$ variance bound suffices, so no filter is needed. We therefore
set
\begin{equation}
\label{eq:main-filter}
    \calF_1(c)=\varnothing,
    \qquad\text{and}\qquad
    \calF_i(c)
    =
    \left\{
        J\in\calD_i^{\rm bank}:
        \dist(c,J)<\ell_i
    \right\}
    \quad\text{for }i\ge2.
\end{equation}
For each index $t$ and orientation $d$, define
\begin{equation}
\label{eq:filter-indicator}
    H_{i,t}^{d}(c)
    =
    \1\left\{
        \eta_{J,t}^{d}=-1
        \text{ for every }J\in\calF_i(c)
    \right\},
    \qquad
    p_i(c)
    =
    \E_\eta[H_{i,t}^{d}(c)]
    =
    2^{-|\calF_i(c)|}.
\end{equation}
Lemma~\ref{lem:filter-geometry} in Appendix~\ref{sec:filter-details} shows that
$|\calF_i(c)|\le3$, and hence $p_i(c)\ge1/8$.  The value
$H_{i,t}^d(c)$ is computed from $c$ and the stored Rademacher signs, whereas
$p_i(c)$ is deterministic once $i$ and $c$ are fixed.  Hence, filtering is performed entirely by the decoder, thereby maintaining the non-adaptivity of the queries.  
The final decoded variable is obtained by filtering and reweighting the
provisional decoded variable:
\begin{equation}
\label{eq:main-decoded-variable}
    V_{i,t}^{d}(c)
    = \frac{H_{i,t}^{d}(c)}{p_i(c)}
    \cdot \widetilde V_{i,t}^{d}(c)
    = \frac{2H_{i,t}^{d}(c) \cdot Y_{i,t}^{d}}{p_i(c)}
    \cdot\sum_{J\in\calW_i(c,r)}
      w_J^d(c)\cdot\eta_{J,t}^d,
\end{equation}
where $\widetilde V_{i,t}^{d}(c)$ is the provisional decoded variable defined
in \eqref{eq:weighted-demodulation}.

The factor $H_{i,t}^d(c)/p_i(c)$ in \eqref{eq:main-decoded-variable} serves two purposes: Inverse-probability reweighting preserves the first moment, while filtering restores tail locality.  Indeed, for $i\ge2$,
\begin{equation}
\label{eq:filter-certificate-main}
    V_{i,t}^d(c)\ne0
    \implies H_{i,t}^d(c) \cdot Y_{i,t}^d=1
    \implies X_{i,t}^d\in
    \bigcup_{J\in\calD_i^{\rm bank}\setminus\calF_i(c)}J
    \implies |X_{i,t}^d-c|\ge\ell_i.
\end{equation}
Hence, the decoded variable can be nonzero only on the required tail event.
The next subsection formalizes these first-moment and variance guarantees
and sketches how they yield Theorem~\ref{thm:main}.

\subsection{Performance Guarantee and Sample Complexity}
\label{sec:performance-main}

We now state the main construction-specific guarantee and sketch how it
yields Theorem~\ref{thm:main}.
Lemma~\ref{lem:per-scale-main} below, proved in
Appendix~\ref{sec:demodulation-details}, shows that the decoded refinement
variables~$V_{i,t}^{d}(c)$ defined in
\eqref{eq:main-decoded-variable} satisfy the per-scale first-moment identity
and tail-local variance bound used by~\citet{lau2026order} in their
multiscale analysis.
Given this lemma, the truncation-bias, variance-aggregation, confidence-amplification, and sample-count arguments follow, up to constant factors and differences in notation, the proof of~\citep[Theorem~5]{lau2026order}.  Appendix~\ref{app:upper-bound} gives the complete calculations.

\begin{lemma}[Per-Scale Refinement Properties]
\label{lem:per-scale-main}
For every fixed $c\in[-\lambda,\lambda]$, scale
$i\in\{1,\ldots,\imax\}$, repetition index $t\in\{1,\ldots,n_i\}$, and
orientation $d\in\{\mathrm L,\mathrm R\}$, the decoded variables in
\eqref{eq:main-decoded-variable} satisfy the following per-scale first-moment identity and tail-local variance bound:
\begin{equation}
\label{eq:per-scale-properties}
    \E\left[V_{i,t}^{\mathrm L}(c)+V_{i,t}^{\mathrm R}(c)\right]
    =
    \sum_{J\in\calW_i(c,r)}\theta_J
    \quad \text{and} \quad
    \Var\left(V_{i,t}^{d}(c)\right)
    =O\left(\ell_i^2\cdot\tau_i(c)\right),
\end{equation}
where $\tau_1(c)=1$ and $\tau_i(c)=\Pr(|X-c|\ge\ell_i)$ for $i\ge2$.
\end{lemma}

We sketch the proof here; Appendix~\ref{sec:demodulation-details} gives
the complete argument.

\begin{proof}[Proof sketch of Lemma~\ref{lem:per-scale-main}]
Using the disjointness of the filter and cover cells, the calculation in Appendix~\ref{sec:demodulation-details} shows that filtering and inverse-probability reweighting preserve the conditional first moment.  Together with~\eqref{eq:provisional-first-moment}, this gives the first identity in~\eqref{eq:per-scale-properties}.
For the variance bound, $p_i(c)\ge1/8$ and Lemma~\ref{lem:cover-main} give $|V_{i,t}^d(c)|=O(\ell_i)$.
Equation~\eqref{eq:filter-certificate-main} restricts every non-zero value
to the event $\{|X_{i,t}^d-c|\ge\ell_i\}$ for $i\ge2$, while $\tau_1(c)=1$ handles the
base scale.  Hence,
$\Var\big(V_{i,t}^d(c)\big)
    \le \E\big[\big(V_{i,t}^d(c)\big)^2\big]
    = O\big(\ell_i^2 \cdot \tau_i(c)\big).
$
\end{proof}

We next sketch how these per-scale guarantees yield
Theorem~\ref{thm:main}; Appendix~\ref{app:upper-bound} gives the complete
calculations.

\paragraph{Aggregation across scales.}
With Lemma~\ref{lem:per-scale-main} in hand, we use, up to constant factors,
the same scale-wise geometric allocation as \citep{lau2026order}:
\begin{equation}
\label{eq:main-ni}
    n_i
    =
    \left\lceil
      C_1\frac{\bar\sigma^2}{\eps^2}
      \cdot 2^{(i-1)(2-k)}
    \right\rceil,
\end{equation}
where $C_1$ is a sufficiently large universal constant.  This allocation balances the
increasing $O(\ell_i^2)$ squared endpoint weights against the decreasing tail probability $O((\bar\sigma/\ell_i)^k)$ of reaching scale $i$.
An initial ``constant-probability confidence'' estimate of the residual mean can be formed as
\begin{equation}
\label{eq:base-estimator-main}
    \widehat\theta_{\rm base}(c)
    =
    \sum_{i=1}^{\imax}
    \frac{1}{n_i}
    \sum_{t=1}^{n_i}
    \left(
        V_{i,t}^{\mathrm L}(c)
        +
        V_{i,t}^{\mathrm R}(c)
    \right).
\end{equation}

Condition on successful localization and on the realized center $c$, and write
$\calU(c,r)=\bigcup_{J\in\calW(c,r)}J$.  The per-scale first-moment identity in
\eqref{eq:per-scale-properties} and disjointness of the cover give
\begin{equation}
\label{eq:base-first-moment}
    \E[\widehat\theta_{\rm base}(c)]
    =
    \E[(X-c) \cdot \1\{X\in\calU(c,r)\}].
\end{equation}
Since $[c-r,c+r]\subseteq\calU(c,r)$, the omitted contribution satisfies
\begin{equation}
\label{eq:base-bias}
    \left|
      \E[\widehat\theta_{\rm base}(c)]-(\mu-c)
    \right|
    \le
    \E\left[|X-c| \cdot \1\{|X-c|>r\}\right]  
    \le
    \frac{\E[|X-c|^k]}{r^{k-1}}
    \le
    \frac{\bar\sigma^k}{r^{k-1}}
    \le\frac{\eps}{2},
\end{equation}
where the second-last inequality uses the transferred moment bound
\eqref{eq:moment-transfer}, and the final inequality follows from the choice
of $r$ in \eqref{eq:main-radii}.
Independence of samples across $(i,t)$ and the tail-local variance bound in
\eqref{eq:per-scale-properties}, together with~\eqref{eq:main-ni}, give
\begin{equation}
\label{eq:base-variance-main}
    \Var(\widehat\theta_{\rm base}(c))
    =
    O\left(
      \sum_{i=1}^{\imax}
      \frac{\ell_i^2}{n_i}\tau_i(c)
    \right)=
    O\left(
      \frac{\eps^2}{C_1}
      \sum_{i=1}^{\imax}
      2^{k(i-1)}\tau_i(c)
    \right)
    =
    O\left(\frac{\eps^2}{C_1}\right).
\end{equation}
The last step uses the same dyadic moment-tail bound proved by
\citet{lau2026order}; see~\eqref{eq:dyadic-tail-sum} in Appendix~\ref{sec:tail-variance} for more details.  Choosing $C_1$ sufficiently large and applying
Chebyshev's inequality, together with the bias bound~\eqref{eq:base-bias}, establishes that
$
    \Pr\big(
      |\widehat\theta_{\rm base}(c)-(\mu-c)|\le\eps
      \,\big|\,
      |c-\mu|\le8\sigma
    \big)
    \ge 2/3
$ (or similarly with $2/3$ replaced by any fixed constant in $(1/2,1)$).

\paragraph{High-probability final guarantee.}
Set $\delta_{\rm loc}=\delta_{\rm ref}=\delta/2$.  Repeat the base estimate
independently an odd $K=O(\log(1/\delta_{\rm ref}))$ times, and let $\widehat\theta$ be their
median.  Standard median amplification gives
$\Pr\big(
      |\widehat\theta-(\mu-c)|>\eps
      \,\big|\,
      |c-\mu|\le8\sigma
    \big)
    \le\delta_{\rm ref}.
$
Consequently, the final estimate
$\widehat\mu=c+\widehat\theta$ is $(\eps,\delta)$-PAC after combining the
localization and refinement failure probabilities.

\paragraph{Sample complexity.}
One base estimate uses $2\sum_{i=1}^{\imax}n_i$ refinement queries, where the
factor two corresponds to the two orientations $d\in\{\mathrm L,\mathrm R\}$. 
Since $\bar\sigma=9\sigma$, the cutoff
$\imax$ in \eqref{eq:main-radii} and the allocation in
\eqref{eq:main-ni} have the same orders as their counterparts in
\citep{lau2026order}.  The geometric sum therefore has the same three
regimes as theirs: it is dominated by $i=1$ when $k>2$, receives an equal-order
contribution from each of $O(\log(\sigma/\eps))$ scales when $k=2$, and is
dominated by $i=\imax$ when $1<k<2$.  Multiplying by the
$O(\log(1/\delta))$ median-amplification factor gives the three refinement
terms in \eqref{eq:main-rate-compact}.  Localization contributes
$O\left(\log(\lambda/\sigma) +\log (1/\delta) \right)$
additional queries.  Its confidence term is absorbed by every refinement
term because $\eps<\sigma$.  This proves the stated sample bound.

Appendices~\ref{sec:conditioning}--\ref{sec:sample-complexity} provide the
complete bias, variance, amplification, and geometric-series calculations.
Moreover, Appendix~\ref{app:localization-iid} converts the fixed query counts into one
i.i.d.~query law without changing the sample order or interval bound.  This
completes the proof of Theorem~\ref{thm:main}.

\section{The Sample--Interval Tradeoff} 
\label{sec:interval-complexity}

The construction in Section~\ref{sec:construction} uses
$O(\lambda/\sigma)$ intervals per query, whereas a fully sequential protocol from
\citep{lau2026order} (but not the two-stage one) uses only threshold queries.  On the other hand, \citet[Theorem~11]{lau2026order} showed that restricting a \emph{non-adaptive strategy} to one interval per query must incur an additional sample-complexity cost of order $(\lambda \sigma)/ \eps^2 \cdot \log(1/\delta)$, notably having linear dependence on $\lambda$ rather than logarithmic. 
This motivates us to study how this penalty decreases when each non-adaptive query may use up to~$s$ intervals, for general $s$. 
As we will see in the proof sketch of Theorem~\ref{thm:interval-complexity}, extending the lower bound and the refinement construction to general $s$ is comparatively direct, whereas localization is more delicate and requires a different construction.

A query $Q$ is \emph{$s$-interval} if its one-set $Q^{-1}(1)$ is a union of
at most $s$ intervals.\footnote{Defining interval complexity using
$Q^{-1}(0)=\R\setminus Q^{-1}(1)$ changes the number of components by at
most one, so the choice of output label does not affect our order-wise results.}
The intervals may be open, closed, half-open, or unbounded.  We call the
number of connected components of $Q^{-1}(1)$ the query's
\emph{interval complexity}.  In particular, threshold queries and interval
queries are both 1-interval.
Let $n_s^\star(k,\lambda,\sigma,\eps,\delta)$ denote the minimax sample
complexity among non-adaptive 1-bit estimators over
$\calP(k,\lambda,\sigma)$ whose queries are almost surely $s$-interval.
Subsequently, we use the terminology \emph{unrestricted} to remove (only) the interval-budget
constraint; the 1-bit and non-adaptive restrictions still apply.  
For $u\ge1$, define the refinement-rate function
\begin{equation}
\label{eq:rate-functional}
    \mathfrak R_k(u,\delta)
    =
    \begin{cases}
    u^2\cdot\log(1/\delta),& k>2,\\[0.2em]
    u^2\cdot\log u\cdot\log(1/\delta),& k=2,\\[0.2em]
    u^{k/(k-1)}\cdot\log(1/\delta),&1<k<2,
    \end{cases}
\end{equation}
and observe that $\mathfrak R_k(\sigma/\eps,\delta)$ is the refinement term in
\eqref{eq:main-rate-compact}.

\begin{theorem}[Minimax Sample--Interval Tradeoff for Non-Adaptive
1-Bit Mean Estimation]
\label{thm:interval-complexity}
Fix $k>1$. For sufficiently small constants $c_k,\delta_k>0$, suppose
that $\lambda\ge\sigma>0$, $\eps\in(0,c_k\sigma)$,
$\delta\in(0,\delta_k)$, and $s\ge1$ is an integer. Then
\begin{equation}
\label{eq:full-s-tradeoff}
n_s^\star(k,\lambda,\sigma,\eps,\delta)
=
\Theta_k\left(
\log\frac{\lambda}{\sigma}
+\mathfrak R_k(\sigma/\eps,\delta)
+\left(\frac{1}{s}\right) \cdot \left(\frac{\lambda}{\sigma} \right) \cdot
 \left(\frac{\sigma}{\eps}\right)^2
 \log\frac1\delta
\right).
\end{equation}
The upper bound is attained by an estimator whose query functions are
i.i.d.~and whose one-sets almost surely have at most $s$ interval
components. 
\end{theorem}

The first two terms in \eqref{eq:full-s-tradeoff} are the localization and refinement costs without an interval restriction.  The final term is the
additional cost of the $s$-interval constraint.  For $s=1$, this recovers the interval-dependent cost from~\citep[Theorem~11]{lau2026order}, which can be viewed as the cost of treating
$\Theta(\lambda/\sigma)$ possible $\sigma$-scale locations separately,
with $O((\sigma/\eps)^2\log(1/\delta))$ samples per location.  Allowing $s$
intervals reduces this cost by a factor of $s$, until it is absorbed by the
unrestricted rate.  Balancing this interval-dependent cost against the unrestricted rate yields
the following corollary.

\begin{corollary}[Minimum Interval Budget at the Unrestricted 1-Bit Sample Rate]
\label{cor:optimal-components}
Fix $k>1$.  Let $c_k,\delta_k>0$ be as in Theorem~\ref{thm:interval-complexity}, and suppose that $\lambda\ge\sigma>0$, $\eps\in(0,c_k\sigma)$, and
$\delta\in(0,\delta_k)$. 
Let $N_k^{\rm unres}=\log(\lambda/\sigma)
+ \mathfrak R_k(\sigma/\eps,\delta)$ denote the unrestricted 1-bit minimax sample complexity
expression, and define
\begin{equation}
\label{eq:global-sopt-main}
    s_k^{\rm opt}
    =
    \left\lceil
      \max\left\{
        1,
        \frac{
          (\lambda\sigma/\eps^2)\cdot\log(1/\delta)
        }{N_k^{\rm unres}}
      \right\}
    \right\rceil.
\end{equation}
There is a non-adaptive estimator that is $(\eps,\delta)$-PAC over
$\calP(k,\lambda,\sigma)$, uses $O_k(N_k^{\rm unres})$ samples, and has
i.i.d.~queries that are almost surely $s_k^{\rm opt}$-interval. Moreover, any
non-adaptive estimator using only $s$-interval queries that is
$(\eps,\delta)$-PAC over $\calP(k,\lambda,\sigma)$ and uses
$O_k(N_k^{\rm unres})$ samples must have
$s=\Omega_k(s_k^{\rm opt})$.
\end{corollary}

% \begin{proof}
% By the definition of $s_k^{\rm opt}$, the interval-dependent term in
% \eqref{eq:full-s-tradeoff} is $O(N_k^{\rm unres})$ when
% $s=s_k^{\rm opt}$.  The upper bound and i.i.d. assertion in
% Theorem~\ref{thm:interval-complexity} therefore give the stated estimator. 
% For the converse, the lower bound in Theorem~\ref{thm:interval-complexity} requires $(\lambda\sigma)/(s\eps^2) \cdot \log(1/\delta) =O_{k}(N_k^{\rm unres})$.
% Combining this inequality with $s\ge1$ yields $s=\Omega_k(s_k^{\rm opt})$.
% \end{proof}

We next give a proof sketch of Theorem~\ref{thm:interval-complexity}.  
The formal details are given in Appendix~\ref{app:interval-complexity}, and the conversion to a single i.i.d.~query law is given in Appendix~\ref{app:localization-iid}.

\begin{proof}[Proof sketch of Theorem~\ref{thm:interval-complexity}]
The proof has three parts.  For the lower bound, the only new step compared 
to~\citep[Theorem~11]{lau2026order} is to count how many hard locations an~$s$-interval 
query can distinguish.  For the upper bound, we group the refinement queries to 
respect the interval budget, whereas localization requires a separate $s$-interval construction.

\emph{Lower bound.}
The unrestricted localization and refinement terms follow from the
unrestricted lower bound \citep[Theorem~9]{lau2026order}.  For the interval-dependent term, we use the hard family from \citep[Theorem~11]{lau2026order}, which places two point masses at each of
$N=\Theta_k(\lambda/\sigma)$ locations with pairwise disjoint support
segments.  Conditioning on the protocol randomness, let $n_j$ be the number
of queries that separate the two support points at location $j$. An $s$-interval one-set has at most $2s$ boundary points.  Separating a pair
requires a boundary point in its support segment, so the disjointness of the
segments implies that each query separates at most $2s$ pairs.  Hence,
$\sum_{j=1}^N n_j\le 2sn$.  Each informative response
contributes at most $O_k(\eps^2/\sigma^2)$ KL divergence, while every other
response contributes zero.  Repeating the testing and averaging argument of
\citep[Theorem~11]{lau2026order} with $2sn$ in place of $2n$ gives
$n = \Omega_k\left( (\lambda\sigma)/(s\eps^2) \cdot \log(1/\delta) \right)$.

\emph{Refinement.}
At each scale $i$, we partition the bank $\calD_i^{\rm bank}$ in \eqref{eq:main-query-bank} into
$G_i=O(\max\{1,\lambda/(s\ell_i)\})$ groups of at most $s$ cells and apply
the refinement construction of Section~\ref{sec:construction} separately to
each group using independent samples.  The resulting queries are $s$-interval.  
The groupwise expectations sum to the original first-moment
target, while independence makes their variances add and preserve the same
tail-local bound.  Hence, the analysis of
Section~\ref{sec:performance-main} applies with $G_i n_i$ queries at scale
$i$.  Summing over the scales and amplifying yields $O_k(\mathfrak R_k(\sigma/\eps,\delta)
+(\lambda\sigma/(s\eps^2))\log(1/\delta))$ refinement samples.  

\emph{Localization.}
Grouping the localization queries of Section~\ref{sec:localization-main} would introduce
an extraneous term of order~$(\lambda/(s\sigma)) \cdot \log(\lambda/\sigma)$, so we instead use a different localization construction.
For $s\le8$, ordinary random-threshold queries suffice; this case is
analyzed using Lemma~\ref{lem:random-threshold-localization} in
Appendix~\ref{app:sharp-localization}.
For $s>8$, we first describe the construction for a single block partition
and then explain how to avoid a boundary instability issue.
Fix in advance a partition of $[-\lambda,\lambda]$ into at most $s$ blocks of width
$L=\Theta(\max\{\sigma,\lambda/s\})$.  For this partition, precommit to two query families.
The first identifies the block index using the coding mechanism of
Lemma~\ref{lem:LS-localization}.  The second places the same random threshold
relative to the left endpoint of every block, allowing the decoder to
estimate the within-block offset without knowing the relevant block in
advance.  The decoder recovers the block index and offset separately and
combines them.
Both query families are $s$-interval: a block-index query selects a union of
whole blocks, while an offset query selects at most one prefix interval from
each block.  

While the preceding decomposition conveys the key idea, it can be unstable when the mean lies near
a block boundary.  To remove this instability, we precommit to both query
families for each of three suitably shifted partitions.  Every mean lies safely inside a
block for at least two shifts, so taking the median of the three
reconstructions removes the boundary instability.
Appendix~\ref{app:sharp-localization} shows that the resulting localizer uses $O(\log (\lambda/\sigma) +\max\{1, \lambda/(\sigma s)\} \cdot \log(1/\delta))$ samples. Since $\eps<\sigma$, the second term is absorbed by the refinement terms above, completing the upper bound.
\end{proof}

\section{Conclusion}
\label{sec:conclusion}

Theorem~\ref{thm:main} resolves the main motivating question posed in
\citep{lau2026open}: for every $k>1$, a non-adaptive estimator with
i.i.d.~queries attains the adaptive 1-bit minimax optimal sample rate in
the relevant parameter regimes.  All queries are fixed before any response; only the decoder waits
for the localization output before interpreting the refinement bits.

Theorem~\ref{thm:interval-complexity} then establishes the minimax-optimal tradeoff between the number of samples $n$ and the maximum number of intervals per query $s$.  Relative to unrestricted non-adaptive querying, the
constraint adds $(\lambda\sigma/(s\eps^2))\log(1/\delta)$ samples, up to
constants depending on $k$ and the fixed failure-probability cap.
Corollary~\ref{cor:optimal-components} identifies, up to constants depending only on $k$, the minimum interval budget that preserves the unrestricted 1-bit sample rate.
By contrast, a fully sequential protocol attains this rate using only threshold queries \citep{lau2026order}, so
the penalty comes from combining non-adaptivity with a limited interval budget.

As noted in \citep{lau2026order}, several open problems still remain including settings where $(\sigma,\eps)$ is unknown to the learner and multivariate settings.  Our results in Section~\ref{sec:interval-complexity} also raise the possibility of studying a more general tradeoff between samples, intervals, \emph{and rounds of adaptivity}.  Finally, a broader direction is to understand when adaptivity helps under quantization in other statistical problems.

\section*{Declaration of AI Usage}

The authors determined the overall research direction and led the development of the non-adaptive 1-bit query design, estimator, and the proof strategy. AI models (initially ChatGPT 5.5 Pro and later ChatGPT-5.6 Sol Pro) were used to explore mathematical arguments, draft proof sketches for several intermediate results, and assist with planning and drafting parts of the exposition. Some technical ideas used in the final analysis arose from this exploration, including Rademacher-sign averaging. All AI-generated output was carefully checked and heavily revised or rewritten by the authors. The authors take full responsibility for the final manuscript and the validity of its results.

\section*{Acknowledgment}

This work was supported by the Singapore National Research Foundation (NRF) under its AI Visiting Professorship programme.

\bibliography{bibliography}

\clearpage
\appendix
{\bf \Huge \centering Appendix \par}

\section{Finite-Union Localization
  (Lemma~\ref{lem:LS-localization})}
\label{app:coding-localization}

This appendix proves Lemma~\ref{lem:LS-localization}, stated in
Section~\ref{sec:localization-main}.  We start from the coding-theoretic localizer used by~\citet{lau2026order} to prove Theorem~16 and modify its coordinate schedule by sampling code coordinates independently with replacement, leading to i.i.d. localization queries.  Writing each coordinate query as the
union of the bins carrying code bit one will make explicit that every
one-set is a union of $O(\lambda/\sigma)$ intervals.
We additionally tighten the constant factors. 
\citet{lau2026order} use bins of width between $10\sigma$ and $20\sigma$ and
return an interval comprising at most five consecutive bins; taking its
midpoint gives a radius of at most $50\sigma$.  Here we use bins of width
$5\sigma$ and return the projected midpoint of the decoded bin.  With
probability at least $1-\delta_{\rm loc}$, the decoded bin is the true bin or
one of its two neighbors, giving error at most $7.5\sigma<8\sigma$.

\subsection{Nearly Equidistant Codebook}

We use the following standard random-coding lemma, which records the codebook construction of \citet{lau2026order} with a general distance tolerance.  We apply it with $\xi=0.01$ below, as in their localization proof, and with
$\xi=1/4$ for the $s$-interval localizer in
Appendix~\ref{app:sharp-localization}.

\begin{lemma}[Nearly equidistant binary codebook]
\label{lem:nearly-equidistant-code}
For every integer $K\ge2$ and every fixed $\xi\in(0,1/2)$, there are
codewords $z_1,\ldots,z_K\in\{0,1\}^{d_{\rm code}}$, with
$d_{\rm code}=O_\xi(\log K)$, such that
\[
    \left(\frac{1}{2}-\xi\right)
    \le \frac{d_H(z_j,z_{j'})}{d_{\rm code}}
    \le\left(\frac{1}{2}+\xi\right)
    \qquad(j\ne j').
\]
\end{lemma}

\begin{proof}
Draw the $K$ codewords independently and uniformly from
$\{0,1\}^{d_{\rm code}}$.
For each pair, the Hamming distance has distribution
$\operatorname{Bin}(d_{\rm code},1/2)$.  Hoeffding's inequality and a union
bound give
\[
    \Pr\left(
      \max_{j\ne j'}
      \left|\frac{d_H(z_j,z_{j'})}{d_{\rm code}}-\frac{1}{2}\right|>\xi
    \right)
    \le K^2\cdot2e^{-2\xi^2d_{\rm code}}.
\]
For a sufficiently large constant $C_\xi$, taking
$d_{\rm code}=\lceil C_\xi\log K\rceil$ makes the right-hand side less than
one.  Hence, a deterministic realization with the claimed property exists.
\end{proof}

\subsection{Finite-Union Localization with an \texorpdfstring{$8\sigma$}{8 sigma}
Guarantee}
\label{app:tight-coding-localization}

We now prove Lemma~\ref{lem:LS-localization} in three steps.  We first partition $[-\lambda,\lambda]$ into bins and construct the i.i.d.\ finite-union queries.  For each candidate bin, we next consider the fraction of responses that disagree with its assigned code bits.  We show that the true bin or one of its two neighbors has an expected disagreement fraction strictly smaller than that of every bin outside this three-bin set.  Finally, we use
concentration to show that the decoder selects one of these three bins.

\paragraph{Bin construction and i.i.d. queries.}
Set $h=5\sigma$ and $a=-\lambda-2h$.  Let $N = \lceil 2 \lambda/h \rceil +4$ so that $b\coloneqq a+Nh\ge\lambda+2h$.  Use the interior bins
\[
    B_j=[a+(j-1)h,a+jh),
    \qquad j=1,\ldots,N,
\]
together with the exterior bins
$B_0=(-\infty,a)$ and $B_{N+1}=[b,\infty)$.  There are
$K=N+2=O(\lambda/\sigma)$ bins.  The grid extends $2h$ past $-\lambda$ and at least $2h$ past $\lambda$.
This boundary guard ensures that the interior bin containing any
$\mu\in[-\lambda,\lambda]$ has an interior neighbor on both sides.  Thus, the
true bin and its two neighbors are interior bins even when $\mu$ is at an endpoint,
so the three-bin comparison below needs no separate boundary case.  The guard adds only $O(1)$ bins and has only a mild effect on the constants; the $8\sigma$ guarantee primarily comes from the choice $h=5\sigma$ and decoding at most one bin away.

Choose deterministic codewords
$z_0,\ldots,z_{N+1}\in\{0,1\}^{d_{\rm code}}$ such that
\begin{equation}
\label{eq:localization-code-distance}
    0.49
    \le \frac{d_H(z_i,z_j)}{d_{\rm code}}
    \le 0.51
    \qquad \text{for each }i\ne j,
\end{equation}
with $d_{\rm code}=O(\log K)$; their existence follows from
Lemma~\ref{lem:nearly-equidistant-code} with $\xi=0.01$.
Independently of the samples, the learner draws
$I_1,\ldots,I_n\stackrel{\mathrm{i.i.d.}}{\sim}
\operatorname{Unif}\{1,\ldots,d_{\rm code}\}$
and issues the queries
\[
    Q_t(x)
    =
    \1\left\{
      x\in\bigcup_{u:z_{u,I_t}=1}B_u
    \right\}.
\]
Agent $t$ returns $Y_t=Q_t(X_t)$.  Thus, the query functions are i.i.d.,
and every one-set is a union of at most
$K=O(\lambda/\sigma)$ intervals.  Because the bins partition $\mathbb R$,
on the event $\{X_t\in B_u\}$ we have $Y_t=z_{u,I_t}$.
\paragraph{Expected disagreement comparison.}
Let $B_i$ be the bin containing $\mu$, and set
$S=\{i-1,i,i+1\}$.  If $X$ is not in these three bins, then we must have
$|X-\mu|\ge h$.  Therefore, Markov's inequality gives
\[
    q
    \coloneqq
    \Pr\left(X\notin\bigcup_{u\in S}B_u\right)
    \le
    \Pr(|X-\mu|\ge h)
    \le
    \frac{\E[|X-\mu|]}{h}
    \le
    \frac{1}{5}.
\]
Writing $p_u=\Pr(X\in B_u)$, we have
\begin{equation}
\label{eq:S-sum-bound}
    \sum_{u\in S}p_u=1-q \ge \frac{4}{5}.
\end{equation}
Since $|S|=3$, a most probable bin among these three,
$i^\star\in\arg\max_{u\in S}p_u$, satisfies
\begin{equation}
\label{eq:p_i_star_bound}
    p_{i^\star}
    \ge
    \frac{1}{3}\sum_{u\in S}p_u
    =
    \frac{1-q}{3}
    \ge
    \frac{4}{15}.
\end{equation}
Let $(X,I,Y)$ denote a generic sample, sampled code coordinate, and response, sharing the common distribution of $(X_t,I_t,Y_t)$.
For every candidate bin $j\in\{0,\ldots,N+1\}$, define its
\emph{empirical disagreement score} as the fraction of observed responses
that disagree with its assigned code bits:
\begin{equation}
\label{eq:coding-empirical-score}
    \widehat H_j
    =
    \frac{1}{n}\sum_{t=1}^n\1\{Y_t\ne z_{j,I_t}\}.
\end{equation}
We call its expectation the corresponding \emph{population score}:
\begin{equation}
\label{eq:coding-population-score}
    H_j
    \coloneqq
    \E[\widehat H_j]
    = \E\left[\1\{Y\ne z_{j,I}\}\right] 
    =
    \sum_{u=0}^{N+1}p_u\cdot
      \E_I\left[\1\{z_{u,I}\ne z_{j,I}\}\right] 
    =
    \sum_{u \ne j}
      p_u\cdot\frac{d_H(z_j,z_u)}{d_{\rm code}},
\end{equation}
where we conditioned on the bin containing $X$, used the uniformity of $I$, the definition of Hamming distance, and the fact that $d_H(z_j, z_j) = 0$.
Applying~\eqref{eq:coding-population-score} to reference bin $i^\star$, and using the Hamming distance upper bound in~\eqref{eq:localization-code-distance} and bound~\eqref{eq:p_i_star_bound} on $p_{i^\star}$ yields
\[
    H_{i^\star}
    \le
    0.51\sum_{u\ne i^\star}p_u
    =
    0.51(1-p_{i^\star})
    \le
    0.51\left(1-\frac{4}{15}\right)
    =0.374.
\]
Conversely, for any candidate bin $j \notin S$, we have $j \ne u$ for all $u \in S$. Restricting the summation in~\eqref{eq:coding-population-score} to $S$, and using the Hamming distance lower bound from~\eqref{eq:localization-code-distance} together with the lower bound on $\sum_{u \in S} p_u$ from~\eqref{eq:S-sum-bound} yields
\[
    H_j
    \ge
    0.49 \sum_{u \in S} p_u
    =
    0.49(1 - q)
    \ge
    0.49 \cdot \frac{4}{5}
    = 0.392.
\]
Combining the preceding two bounds shows that every candidate bin $j\notin S$ has a larger population score than the reference bin $i^\star$, with gap at least
\begin{equation}
\label{eq:localization-population-gap}
    H_j-H_{i^\star}\ge g,
    \quad \text{where} \quad
    g\coloneqq0.392-0.374=0.018.
\end{equation}
\paragraph{Concentration and decoding.}
For each fixed $j$, the summands defining $\widehat H_j$~in~\eqref{eq:coding-empirical-score} are independent
$[0,1]$-valued variables with mean $H_j$.  Hoeffding's inequality and a union
bound over the $K$ candidates therefore give
\begin{equation}
\label{eq:localization-uniform-concentration}
    \Pr\left(
      \max_{0\le j\le N+1}|\widehat H_j-H_j|>\frac{g}{4}
    \right)
    \le 2K\exp\left(-2n\left(\frac{g}{4}\right)^2\right).
\end{equation}
It follows that the right-hand side of
\eqref{eq:localization-uniform-concentration} is at most
$\delta_{\rm loc}$ whenever the number of samples $n$ meets the condition
\begin{equation}
\label{eq:localization-sample-requirement}
    n
    \ge
    \frac{8}{g^2}\log\frac{2K}{\delta_{\rm loc}}.
\end{equation}
When this condition holds,
\eqref{eq:localization-uniform-concentration} shows that, with probability
at least $1-\delta_{\rm loc}$, all empirical scores differ from their
population values by at most $g/4$.  On this event,
\eqref{eq:localization-population-gap} gives, for every $j\notin S$,
\[
    \widehat H_j-\widehat H_{i^\star}
    \ge
    \left(H_j-\frac{g}{4}\right)
    -
    \left(H_{i^\star}+\frac{g}{4}\right)
    \ge
    \frac{g}{2}>0.
\]
Thus, no candidate outside $S$ can minimize the empirical score.  Let
$\hat\imath\in\arg\min_j\widehat H_j$, with ties broken arbitrarily.  It
follows that $\hat\imath\in S$ with probability at least
$1-\delta_{\rm loc}$.  Since
$K=O(\lambda/\sigma)$, condition~\eqref{eq:localization-sample-requirement} gives
\[
    n = O\left(
      \log\frac{\lambda}{\sigma}
      +\log\frac{1}{\delta_{\rm loc}}
    \right).
\]
For an interior bin $B_j$, let $m_j$ be its midpoint, and set
$c_j=\Proj_{[-\lambda,\lambda]}(m_j)$.  For the two exterior bins, set $c_0=-\lambda$ and
$c_{N+1}=\lambda$.  This defines the decoder on every transcript: it returns
$c=c_{\hat\imath}$. On the success event
$\hat\imath\in S=\{i-1,i,i+1\}$, we have
$|\hat\imath-i|\le1$.  The boundary guard also ensures that
$B_{\hat\imath}$ is an interior bin, so its midpoint
$m_{\hat\imath}$ is defined as above.  Therefore,
$
    |m_{\hat\imath}-\mu|
    \le
    3h/2
    =7.5\sigma.
$
Projection onto $[-\lambda,\lambda]$ cannot increase the distance to
$\mu\in[-\lambda,\lambda]$.  Thus $|c-\mu|\le7.5\sigma<8\sigma$, which proves
the lemma.

\section{Geometry and Moment Identities for the Refinement Construction (Section~\ref{sec:construction})}
\label{app:construction}

This appendix verifies the construction-specific claims used in
Section~\ref{sec:construction}.  The argument proceeds in three steps.
Appendix~\ref{sec:query-details} counts the cells in each refinement query
bank, proving that every refinement query uses
$O(\lambda/\sigma)$ intervals.  Appendices~\ref{sec:cover-details}
and~\ref{sec:filter-details} establish the required properties of
the cover $\calW_i(c,r)$ and filter $\calF_i(c)$.  Finally,
Appendix~\ref{sec:demodulation-details} combines these properties with
the cell-wise stochastic-quantization and filtered random-sign identities to
prove the first-moment identity and tail-local variance bound in
Lemma~\ref{lem:per-scale-main}.

\subsection{Interval Bound for Refinement Queries}
\label{sec:query-details}

Recall that at each scale $i=1,\dots,\imax$, the queried cells come from the finite bank defined in
\eqref{eq:main-query-bank}:
\[
    \calD_i^{\rm bank}
    =
    \{J\in\calD_i:J\cap[-\lambda-3\ell_i,\lambda+3\ell_i]\ne\varnothing\}.
\]
An interval of length $L$ intersects at most $L/\ell_i+2$ cells of an
$\ell_i$-grid.  Applying this observation with
$L=2\lambda+6\ell_i$ gives
\begin{equation}
\label{eq:bank-cardinality-detail}
    |\calD_i^{\rm bank}|
    \le \frac{2\lambda+6\ell_i}{\ell_i}+2
    =\frac{2\lambda}{\ell_i}+8
    =O\left(\max\left\{1,\frac{\lambda}{\ell_i}\right\}\right).
\end{equation}
Since each bank cell contributes at most one candidate interval, every
realized query is therefore a finite union of at most
$|\calD_i^{\rm bank}|$ intervals.  Combining~\eqref{eq:bank-cardinality-detail} with $\ell_i\ge\bar\sigma=9\sigma$ and $\lambda\ge\sigma$, this bound simplifies to $O(\lambda/\sigma)$.

\subsection{Geometry of the Decoder-Chosen Cover (Lemma~\ref{lem:cover-main})}
\label{sec:cover-details}

In this subsection, we prove Lemma~\ref{lem:cover-main}.  
We first establish coverage and disjointness.  We then derive the location and endpoint bounds, 
and finish with bank containment and the per-scale cardinality bound.

\paragraph{Coverage.}
Recall that $\mathcal L(c,r)$ is defined in
\eqref{eq:cover-intersecting-cells} and that $J_c^\star(B)$ is selected
according to \eqref{eq:cover-selection-rule}.
Every $x\in[c-r,c+r]$ belongs to a scale-$1$ cell
$B\in\mathcal L(c,r)$, and the selected cell $J_c^\star(B)$ contains $B$.
Hence,
\[
    [c-r,c+r]
    \subseteq
    \bigcup_{B\in\mathcal L(c,r)}J_c^\star(B)
    =
    \bigcup_{J\in\calW(c,r)}J.
\]
\paragraph{Pairwise disjointness.}
Two cells in nested dyadic grids are either disjoint or one contains the
other.  Suppose for contradiction that two distinct selected cells satisfy
$J\subsetneq K$.  The larger cell $K$ has scale at least 2.  Since the decoder
keeps a non-admissible cell only at scale 1, $K$ must be admissible.

If $J$ is admissible, choose a finest cell $B\subseteq J$ that selected
$J$.  Then $K$ is a strictly coarser admissible ancestor of the same $B$,
contradicting the definition of $J=J_c^\star(B)$ as the coarsest admissible
ancestor.  If $J$ is not admissible, then $J$ is a scale-1 cell.  For every
strict ancestor $A\supsetneq J$,
\[
    \dist(c,A)\le\dist(c,J)<\len(J)<\len(A),
\]
so $A$ is also non-admissible.  In particular, $K$ cannot be admissible,
again giving a contradiction.  Hence, distinct cover cells are disjoint.
\paragraph{Location and endpoint bounds.}
Fix an admissible $J\in\calW_i(c,r)$ and let
$P=\operatorname{par}(J)$, so $\len(J)=\ell_i$ and $\len(P)=2\ell_i$.
Admissibility of $J$ gives
\begin{equation}
\label{eq:cover-admissible-lower}
    \dist(c,J)\ge\ell_i.
\end{equation}
The parent $P$ is not admissible; if it were, it would be a coarser admissible
ancestor of every finest cell that selected $J$.  Consequently,
\begin{equation}
\label{eq:cover-parent-upper}
    \dist(c,P)<\len(P)=2\ell_i.
\end{equation}
Since $J$ is one half of $P$, its nearest point can be at most one child
width farther from $c$ than the nearest point of $P$.  Therefore,
\begin{equation}
\label{eq:cover-distance-upper}
    \dist(c,J)
    \le\dist(c,P)+\ell_i
    <3\ell_i.
\end{equation}
For $i\ge2$, every selected cell is admissible, so
\eqref{eq:cover-admissible-lower} and
\eqref{eq:cover-distance-upper} give the stated distance bounds.
For either endpoint $x\in\{a_J,b_J\}$, the distance to $c$ is at most the
distance to the cell plus one cell width.  Combining this observation with
\eqref{eq:cover-distance-upper} gives
\[
    |x-c|\le\dist(c,J)+\ell_i<4\ell_i.
\]
Equation~\eqref{eq:cover-distance-upper} also shows that $J$ intersects
$(c-3\ell_i,c+3\ell_i)$.
The only selected cells not handled by this argument are non-admissible
scale-1 cells.  For each such cell $J$,
$\dist(c,J)<\ell_1$ by definition, and hence its farther endpoint obeys
\[
    \max\{|a_J-c|,|b_J-c|\}
    \le\dist(c,J)+\ell_1<2\ell_1.
\]
Moreover, $J$ intersects $(c-\ell_1,c+\ell_1)$.  Thus, the stated endpoint and
intersection bounds hold in both cases.  Finally, since
$c\in[-\lambda,\lambda]$, any cell intersecting
$(c-3\ell_i,c+3\ell_i)$ also intersects
$[-\lambda-3\ell_i,\lambda+3\ell_i]$ and hence belongs to
$\calD_i^{\rm bank}$.  
\paragraph{Number of cells at one scale.}
By \eqref{eq:cover-admissible-lower} and the endpoint bound, every admissible
scale-$i$ cover cell is contained entirely in
\[
    [c-4\ell_i,c-\ell_i]
    \ \cup\
    [c+\ell_i,c+4\ell_i].
\]
Each component has length $3\ell_i$ and can contain at most three disjoint
scale-$i$ grid cells of width $\ell_i$.  There are therefore at most six admissible cover 
cells at scale $i$.  At scale $1$, the non-admissible cover cells all intersect
$(c-\ell_1,c+\ell_1)$.  An interval of length $2\ell_1$ intersects at most
three cells of an $\ell_1$-grid.  Hence, $|\calW_i(c,r)|\le 9$ as desired.

\subsection{Properties of the Near-Center Filter}
\label{sec:filter-details}

This subsection verifies three properties of the near-center filter $\calF_i(c)$ defined in \eqref{eq:main-filter}, which are used in the second moment analysis (see Lemma~\ref{lem:per-scale-main} in Appendix~\ref{sec:demodulation-details}).  First, the filter contains at most three cells, so its
retention probability $p_i(c)$ defined in~\eqref{eq:filter-indicator} is at least $1/8$.  Second, it is disjoint from the decoder-selected cover $\calW_i(c, r)$, so the retention event $H_{i,t}^d(c) = 1$ defined in~\eqref{eq:filter-indicator} is independent of the signs
used to decode the cover cells.  Third, at every scale $i\ge 2$, a retained
positive response certifies that the corresponding observation lies in the
required tail event.

\begin{lemma}[Filter size, disjointness, and tail implication]
\label{lem:filter-geometry}
For every $c\in[-\lambda,\lambda]$ and
$i\in\{1,\ldots,\imax\}$,
\begin{equation}
\label{eq:filter-size-disjointness}
    |\calF_i(c)|\le 3, \quad
    p_i(c)\ge\frac{1}{8}, \quad
    \text{and} \quad
    \calF_i(c)\cap\calW_i(c,r)=\varnothing.
\end{equation}
Moreover, if $i\ge2$, then for every repetition index $t\in\{1,\ldots,n_i\}$ and orientation $d\in\{\mathrm L,\mathrm R\}$, we have
\begin{equation}
\label{eq:filter-tail-implication}
    H_{i,t}^d(c) \cdot Y_{i,t}^d = 1
   \implies
    |X_{i,t}^d-c|\ge\ell_i.
\end{equation}
\end{lemma}

\begin{proof}
For $i=1$, \eqref{eq:main-filter} and~\eqref{eq:filter-indicator} give
$\calF_1(c)=\varnothing$ and $p_1(c)=1$, respectively, so the claims in~\eqref{eq:filter-size-disjointness} follow.  We therefore assume that $i\ge2$.

Every $J\in\calF_i(c)$ is a scale-$i$ bank cell satisfying
$\dist(c,J)<\ell_i$ and hence intersects
$(c-\ell_i,c+\ell_i)$.  This interval has length $2\ell_i$ and intersects at
most three cells of the scale-$i$ grid.  Therefore,
\[
    |\calF_i(c)|\le 3 
    \quad \text{and} \quad
    p_i(c)=2^{-|\calF_i(c)|} \ge \frac{1}{8}.
\]
Lemma~\ref{lem:cover-main} gives
$\dist(c,J)\ge\ell_i$ for every $J\in\calW_i(c,r)$.  This is incompatible
with the strict inequality $\dist(c,J)<\ell_i$ defining the filter cells in~\eqref{eq:main-filter}, implying that
\[
    \calF_i(c)\cap\calW_i(c,r)=\varnothing.
\]
It remains to prove the tail implication~\eqref{eq:filter-tail-implication}.  Suppose that $H_{i,t}^d(c)\cdot Y_{i,t}^d=1$.
Since both factors are binary, we have $H_{i,t}^d(c)=Y_{i,t}^d=1$.
The definition of the query then gives a cell $J\in\calD_i^{\rm bank}$ such that
\[
    \eta_{J,t}^d=+1
    \quad\text{and}\quad
    X_{i,t}^d\in A_{J,t}^d\subseteq J.
\]
On the other hand, $H_{i,t}^d(c)=1$ forces
$\eta_{K,t}^d=-1$ for every $K\in\calF_i(c)$.  Thus
$J\notin\calF_i(c)$, and the definition of the filter in~\eqref{eq:main-filter} gives
$\dist(c,J)\ge\ell_i$.  Since $X_{i,t}^d\in J$, we have
\[
    |X_{i,t}^d-c|
    \ge
    \dist(c,J)
    \ge
    \ell_i,
\]
which establishes the tail implication~\eqref{eq:filter-tail-implication}.
\end{proof}

\subsection{Proof of the Per-Scale Refinement Properties}
\label{sec:demodulation-details}

This subsection proves Lemma~\ref{lem:per-scale-main}.  We first record the
cell-wise stochastic-quantization identity used in the known-center refinement
of~\citet{lau2026order} in Lemma~\ref{lem:cell-stochastic} below.  We then
analyze how filtering and reweighting transform the provisional decoder
$\widetilde V_{i,t}^d(c)$ in \eqref{eq:weighted-demodulation} into the final
decoder $V_{i,t}^d(c)$ in \eqref{eq:main-decoded-variable}.
Lemma~\ref{lem:filtered-demodulation} below shows that, conditional on the
observation and relative split, the final decoder has the same expectation
over the Rademacher signs as the provisional decoder.  It also shows that the
final decoder is zero whenever the observation lies in a filter cell.
Finally, we combine these probabilistic identities with the cover and filter
geometry in Lemmas~\ref{lem:cover-main} and~\ref{lem:filter-geometry} to obtain
the required first-moment identity and tail-local variance bound.

\paragraph{Cell-wise stochastic quantization.}
\label{sec:stoquant-details}

Although all cells in one query share the same relative split
$\zeta_{i,t}^d$, only its uniform marginal in each fixed cell is needed.  For
$J=[a_J,b_J)$, the split point
$U_{J,t}^d=a_J+\zeta_{i,t}^d\ell_i$ satisfies, for every $u\in J$,
\[
    \Pr(U_{J,t}^d\le u)
    =
    \Pr\left(\zeta_{i,t}^d\le\frac{u-a_J}{\ell_i}\right)
    =
    \frac{u-a_J}{\ell_i}.
\]
Thus $U_{J,t}^d\sim\operatorname{Unif}(J)$.  Dependence between split points
in different cells is immaterial; the identity below is applied cell-by-cell,
and the resulting expectations are summed by linearity. This identity is given in~\citep[Appendix A, Step 4]{lau2026order}; we include a short proof for completeness.

\begin{lemma}[Cell-wise stochastic quantization]
\label{lem:cell-stochastic}
Fix $c\in\mathbb R$ and an interval $J=[a,b)$ with $a<b$.  Let
$U^{\mathrm L}$ and $U^{\mathrm R}$ each be uniform on $J$.  Then, for every
fixed $x\in\mathbb R$,
\begin{equation}
\label{eq:cell-stochastic-pointwise}
     (a-c) \cdot \Pr(a\le x<U^{\mathrm L}) 
    +(b-c) \cdot \Pr(U^{\mathrm R}\le x<b) 
    =(x-c) \cdot \1\{x\in J\}.
\end{equation}
Consequently, if $X$ is independent of both split points, then
\[
    \E[(X-c) \cdot \1\{X\in J\}]
    =
    (a-c) \cdot \Pr(a\le X<U^{\mathrm L})
    +(b-c) \cdot \Pr(U^{\mathrm R}\le X<b).
\]
\end{lemma}

\begin{proof}
If $x\notin J$, both probabilities on the left-hand side of
\eqref{eq:cell-stochastic-pointwise} vanish, so the identity holds.  If
$x\in J=[a,b)$, uniformity gives
\[
    \Pr(a\le x<U^{\mathrm L})
    =
    \frac{b-x}{b-a}
    \quad\text{and}\quad
    \Pr(U^{\mathrm R}\le x<b)
    =
    \frac{x-a}{b-a}.
\]
The left-hand side of \eqref{eq:cell-stochastic-pointwise} is therefore
\[
    (a-c) \cdot \frac{b-x}{b-a}
    +(b-c) \cdot \frac{x-a}{b-a}
    =x-c.
\]
This proves the pointwise identity.  For an independent random variable $X$,
conditioning on $X$ and applying the law of total expectation gives the
integrated identity.
\end{proof}

We next isolate the random-sign calculation underlying the transformation from
the provisional decoder in \eqref{eq:weighted-demodulation} to the final
decoder in \eqref{eq:main-decoded-variable}.  After conditioning on the split
fraction, the candidate intervals are fixed, and the calculation no longer
depends on the dyadic geometry.  It uses only the pairwise disjointness of the
candidate intervals, the independence of the Rademacher signs, and the
disjointness of the selected cover and filter cells.
Lemma~\ref{lem:filtered-demodulation} below states this calculation in a general
form that applies uniformly across all scales, repetition indices, and
orientations.  In the proof of Lemma~\ref{lem:per-scale-main} below, we apply
it with $\mathcal I=\calD_i^{\rm bank}$,
$\mathcal T=\calW_i(c,r)$, and $\mathcal F=\calF_i(c)$.  The same lemma is
applied separately within each group of bank cells in Appendix~\ref{sec:grouped-refinement-proof}.

\begin{lemma}[Filtered random-sign decoding]
\label{lem:filtered-demodulation}
Let $\{A_m:m\in\calI\}$ be a finite family of pairwise disjoint measurable
sets, and let $(\eta_m)_{m\in\calI}$ be independent Rademacher signs that are
independent of $X$.  Fix disjoint sets $\calF,\calT\subseteq\calI$ and
deterministic weights $w_m\in\mathbb R$, $m\in\calT$.  Define
\[
    Q(x)=\1\left\{x\in\bigcup_{m:\eta_m=+1}A_m\right\},
\]
and set
\[
    H=\1\{\eta_m=-1\text{ for all }m\in\calF\},
    \quad
    p=2^{-|\calF|},
    \quad
    S=\sum_{m\in\calT}w_m \cdot \eta_m,
    \quad \text{and} \quad
    V=\frac{2H}{p} \cdot Q(X) \cdot S.
\]
Then, for every fixed $x$, we have
\begin{equation}
\label{eq:filtered-first-identity-abstract}
    \E_\eta[V\mid X=x]
    =
    \sum_{m\in\calT}w_m \cdot \1\{x\in A_m\},
\end{equation}
and
\begin{equation}
\label{eq:filtered-second-identity-abstract}
    \E_\eta[V^2\mid X=x]
    =
    \frac{2}{p}
    \left(\sum_{m\in\calT}w_m^2\right) \cdot 
    \1\left\{x\in\bigcup_{m\in\calI\setminus\calF}A_m\right\}.
\end{equation}
Consequently, we have
\begin{equation}
\label{eq:filtered-integrated-second-identity-abstract}
    \E[V^2]
    =
    \frac{2}{p}
    \left(\sum_{m\in\calT}w_m^2\right)
    \Pr\left(X\in\bigcup_{m\in\calI\setminus\calF}A_m\right).
\end{equation}
\end{lemma}

\begin{proof}
Extend the weights by setting $w_m=0$ for $m\notin\calT$, and write
\[
    S=\sum_{m\in\calI} w_m \cdot \eta_m.
\]
We claim that for every $j\in\calI$, independence and symmetry of the Rademacher signs give
\begin{align}
    \E_\eta[S]
    &=0,
    \label{eq:sign-id-mean}\\
    \E_\eta[\eta_j \cdot S]
    &=w_j,
    \label{eq:sign-id-linear}\\
    \E_\eta[S^2]
    &=\sum_{m\in\calT}w_m^2,
    \label{eq:sign-id-square}\\
    \E_\eta[\eta_j \cdot S^2]
    &=0.
    \label{eq:sign-id-cubic}
\end{align}
To verify these identities, observe first that the joint law of the signs is
invariant under the global sign flip $\eta\mapsto-\eta$.  Under this
transformation, both $S$ and $\eta_jS^2$ change sign.  Their expectations
therefore equal their own negatives and hence are zero.  For the remaining
identities, independence and linearity give
\[
    \E_\eta[\eta_jS]
    =
    \sum_{m\in\calI} w_m \cdot \E_\eta[\eta_j \cdot \eta_m]
    =
    w_j
\]
and
\[
    \E_\eta[S^2]
    =
    \E_\eta\left[
        S\sum_{m\in\calI} w_m \cdot \eta_m
    \right]
    =
    \sum_{m\in\calI} w_m \cdot \E_\eta[\eta_m \cdot S]
    =
    \sum_{m\in\calI}w_m^2
    =
    \sum_{m\in\calT}w_m^2.
\]
Next, for fixed $x$,  pairwise disjointness of the sets $A_m$ gives the pointwise
expansion
\[
    Q(x)
    = \sum_{j\in\calI} \1\{x\in A_j\} \cdot \1\{\eta_j=+1\}.
\]
Pairwise disjointness also leaves three mutually exclusive and exhaustive
cases: (i) $x$ lies outside all the sets $A_m$, (ii) $x$ lies in $A_j$ for some
$j\in\calF$, or (iii) $x$ lies in $A_j$ for some
$j\in\calI\setminus\calF$.  In the latter two cases, the index $j$ is unique.

\emph{Case (i).}
Since $x\notin\bigcup_{m\in\calI}A_m$, we have $Q(x)=0$ and hence $V=0$.  The right-hand sides of
\eqref{eq:filtered-first-identity-abstract} and
\eqref{eq:filtered-second-identity-abstract} also vanish, so both identities
hold.

\emph{Case (ii).}
Suppose that $x\in A_j$ for some $j\in\calF$.
If $H=1$, then
$\eta_j=-1$, whereas $Q(x)=1$ would require $\eta_j=+1$.  Hence
$H\cdot Q(x)=0$ and again $V=0$.  By pairwise disjointness and
$\calF\cap\calT=\varnothing$,
\[
    \sum_{m\in\calT}w_m\cdot\1\{x\in A_m\}=0.
\]
Pairwise disjointness also gives
\[
    x\notin\bigcup_{m\in\calI\setminus\calF}A_m.
\]
Hence, both pointwise identities \eqref{eq:filtered-first-identity-abstract}--\eqref{eq:filtered-second-identity-abstract} hold in this case as well.

\emph{Case (iii).}
Suppose that $x\in A_j$ for some $j\in\calI\setminus\calF$.  Then
\[
    Q(x)
    =
    \1\{\eta_j=+1\}
    =
    \frac{1+\eta_j}{2}.
\]
The variable $H$ depends only on signs indexed by $\calF$, whereas
$(\eta_j,S)$ depends only on signs indexed by $\{j\}\cup\calT$.  Since
$j\notin\calF$ and $\calF\cap\calT=\varnothing$, it follows that $H$ is independent of
$(\eta_j,S)$.  Using this independence, together with $\E_\eta[H]=p$ and
\eqref{eq:sign-id-mean}--\eqref{eq:sign-id-linear}, we obtain
\[
    \E_\eta[V\mid X=x]
    =
    \frac{2}{p}\E_\eta[H] \cdot
    \E_\eta\left[\frac{1+\eta_j}{2}S\right]
    =
    \E_\eta[S]+\E_\eta[\eta_jS]
    =
    w_j.
\]
By pairwise disjointness and the convention $w_j=0$ for
$j\notin\calT$, we have
\[
    w_j
    =
    \sum_{m\in\calT}w_m \cdot \1\{x\in A_m\}.
\]
This proves \eqref{eq:filtered-first-identity-abstract} in case (iii).
For~\eqref{eq:filtered-second-identity-abstract}, using $H^2=H$, $Q(x)^2=Q(x)$, and
\eqref{eq:sign-id-square}--\eqref{eq:sign-id-cubic}, we obtain
\[
    \E_\eta[V^2\mid X=x]
    =
    \frac{4}{p^2}\E_\eta[H] \cdot
    \E_\eta\left[\frac{1+\eta_j}{2}S^2\right] =
    \frac{2}{p}
    \left(\E_\eta[S^2]+\E_\eta[\eta_jS^2]\right) =
    \frac{2}{p}\sum_{m\in\calT}w_m^2.
\]
Since $x\in A_j\subseteq\bigcup_{m\in\calI\setminus\calF}A_m$,
this is exactly the right-hand side of \eqref{eq:filtered-second-identity-abstract}.

Thus, both pointwise identities~\eqref{eq:filtered-first-identity-abstract}--\eqref{eq:filtered-second-identity-abstract} hold in all three cases. Taking expectation with
respect to $X$ in \eqref{eq:filtered-second-identity-abstract} gives \eqref{eq:filtered-integrated-second-identity-abstract} and completes the proof.
\end{proof}

We now combine Lemmas~\ref{lem:cell-stochastic}
and~\ref{lem:filtered-demodulation} with the cover and filter geometry in
Lemmas~\ref{lem:cover-main} and~\ref{lem:filter-geometry} to prove
Lemma~\ref{lem:per-scale-main}.

\begin{proof}[Proof of Lemma~\ref{lem:per-scale-main}]
Fix $c\in[-\lambda,\lambda]$, a scale
$i\in\{1,\ldots,\imax\}$, and a repetition index
$t\in\{1,\ldots,n_i\}$.  For each orientation
$d\in\{\mathrm L,\mathrm R\}$, condition on the relative split
$\zeta_{i,t}^d$.  The candidate intervals
$\{A_{J,t}^d:J\in\calD_i^{\rm bank}\}$ are then fixed and pairwise
disjoint because each is contained in a distinct scale-$i$ grid cell, 
while the Rademacher signs $\eta_{J,t}^d$ remain mutually independent and
independent of $X_{i,t}^d$.  Moreover,
$\calF_i(c)\subseteq\calD_i^{\rm bank}$ by definition, and
Lemmas~\ref{lem:cover-main} and~\ref{lem:filter-geometry} give
\[
    \calW_i(c,r)\subseteq\calD_i^{\rm bank}
    \qquad\text{and}\qquad
    \calF_i(c)\cap\calW_i(c,r)=\varnothing.
\]
Consequently, Lemma~\ref{lem:filtered-demodulation} applies with
\[
    \calI=\calD_i^{\rm bank},
    \qquad
    A_J=A_{J,t}^d,
    \qquad
    \calT=\calW_i(c,r),
    \qquad
    \calF=\calF_i(c),
    \qquad
    w_J=w_J^d(c),
\]
and $X=X_{i,t}^d$.
Under this correspondence, we have
\[
    Q(X)=Y_{i,t}^d, \quad  
    H=H_{i,t}^d(c), \quad 
    p=p_i(c), 
   \quad \text{and} \quad
   V=V_{i,t}^d(c).
\]
The first pointwise identity~\eqref{eq:filtered-first-identity-abstract} in the lemma gives
\[
    \E_\eta\left[
        V_{i,t}^d(c)
        \,\middle|\,
        X_{i,t}^d,\zeta_{i,t}^d
    \right]
    =
    \sum_{J\in\calW_i(c,r)}
        w_J^d(c) \cdot \1\{X_{i,t}^d\in A_{J,t}^d\}.
\]
Taking expectations over the observation and relative split, and then summing
over the two orientations, yields
\begin{align*}
    \E\left[
        V_{i,t}^{\mathrm L}(c)+V_{i,t}^{\mathrm R}(c)
    \right]
    &=
    \sum_{J\in\calW_i(c,r)}
    \left(
        w_J^{\mathrm L}(c)\cdot
        \Pr(X_{i,t}^{\mathrm L}\in A_{J,t}^{\mathrm L})
        +
        w_J^{\mathrm R}(c)\cdot
        \Pr(X_{i,t}^{\mathrm R}\in A_{J,t}^{\mathrm R})
    \right) \\
    &=
    \sum_{J\in\calW_i(c,r)}
        \E\left[(X-c)\cdot\1\{X\in J\}\right] \\
    &=
    \sum_{J\in\calW_i(c,r)}\theta_J,
\end{align*}
where the second equality follows from
Lemma~\ref{lem:cell-stochastic}, since
$X_{i,t}^{\mathrm L} \stackrel{d}{=}  X_{i,t}^{\mathrm R} \stackrel{d}{=}  X$ and are independent of their respective relative splits.

We next bound the second moment.  After conditioning on
$\zeta_{i,t}^d$, the integrated second-moment identity
\eqref{eq:filtered-integrated-second-identity-abstract} in Lemma~\ref{lem:filtered-demodulation} gives
\begin{equation}
\label{eq:filtered-second-per-scale}
    \E\left[
        \left(V_{i,t}^d(c)\right)^2
        \,\middle|\,
        \zeta_{i,t}^d
    \right]
    =
    \frac{2}{p_i(c)}
    \left(
        \sum_{J\in\calW_i(c,r)}
            \left(w_J^d(c)\right)^2
    \right)  \cdot
    \Pr\left(
        X_{i,t}^d\in
        \bigcup_{K\in\calD_i^{\rm bank}\setminus\calF_i(c)}
            A_{K,t}^d
        \,\middle|\,
        \zeta_{i,t}^d
    \right).
\end{equation}
We bound the three factors on the right-hand side separately.  First,
Lemma~\ref{lem:filter-geometry} gives
\begin{equation}
\label{eq:filtered-second-per-scale-first-factor}
    \frac{2}{p_i(c)}\le16.
\end{equation}
Second, Lemma~\ref{lem:cover-main} gives
$|w_J^d(c)|<4\ell_i$ for every $J\in\calW_i(c,r)$ and
$|\calW_i(c,r)|=O(1)$.  Hence,
\begin{equation}
 \label{eq:filtered-second-per-scale-second-factor}
    \sum_{J\in\calW_i(c,r)}
        \bigl(w_J^d(c)\bigr)^2
    \le
    |\calW_i(c,r)|\cdot(4\ell_i)^2
    =
    O(\ell_i^2).
\end{equation}
It remains to bound the conditional probability. 
For $i\ge2$, the definition of $\calF_i(c)$ and the containment
$A_{K,t}^d\subseteq K$ give
\[
    \bigcup_{K\in\calD_i^{\rm bank}\setminus\calF_i(c)}
        A_{K,t}^d
    \subseteq
    \bigcup_{K\in\calD_i^{\rm bank}\setminus\calF_i(c)}K
    \subseteq
    \{x:|x-c|\ge\ell_i\}.
\]
Since $X_{i,t}^d\stackrel{\mathrm d}{=}X$ is independent of
$\zeta_{i,t}^d$, the conditional probability is at most
$\Pr(|X-c|\ge\ell_i)$ for $i\ge2$.  At scale $i=1$, it is trivially at
most one.  Thus, for every scale $i$,
\begin{equation}
 \label{eq:filtered-second-per-scale-third-factor}
    \Pr\left(
        X_{i,t}^d\in
        \bigcup_{K\in\calD_i^{\rm bank}\setminus\calF_i(c)}
            A_{K,t}^d
        \,\middle|\,
        \zeta_{i,t}^d
    \right)
    \le
    \tau_i(c)=
    \begin{cases}
        1, & i=1,\\
        \Pr(|X-c|\ge\ell_i), & i\ge2.
    \end{cases}
\end{equation}
All three upper bounds~\eqref{eq:filtered-second-per-scale-first-factor}--\eqref{eq:filtered-second-per-scale-third-factor} are independent of the realized relative split. Consequently, there is a universal constant $C_0$ such that
\[
    \E\left[
        \bigl(V_{i,t}^d(c)\bigr)^2
        \,\middle|\,
        \zeta_{i,t}^d=z
    \right]
    \le
    C_0 \ell_i^2\cdot\tau_i(c)
    \quad \text{for every } z\in[0,1].
\]
Taking expectation over $\zeta_{i,t}^d$ gives
\begin{equation}
\label{eq:single-query-second-moment}
    \E\left[\bigl(V_{i,t}^d(c)\bigr)^2\right]
    \le
    C_0 \ell_i^2\cdot\tau_i(c).
\end{equation}
Therefore,
\[
    \Var\left(V_{i,t}^d(c)\right)
    \le
    \E\left[\bigl(V_{i,t}^d(c)\bigr)^2\right]
    =
    O\left(\ell_i^2\cdot\tau_i(c)\right).
\]
Together with the first-moment identity, this completes the proof.
\end{proof}

\section{Proof of the Minimax Upper Bound (Theorem~\ref{thm:main})}
\label{app:upper-bound}

This appendix analyzes the estimator underlying
Theorem~\ref{thm:main} under the prescribed query counts from
Section~\ref{sec:construction}.  The localization stage uses $n_{\rm loc}$
queries.  We call the collection consisting of $n_i$ refinement queries of
each orientation at every scale $i$ a \emph{base refinement block}.  Each
such block produces one base refinement estimate
$\widehat\theta_{\rm base}(c)$, defined in
\eqref{eq:base-estimator-main}.  All localization and refinement queries
are mutually independent but not identically distributed.
Appendix~\ref{app:localization-iid} converts the resulting protocol into one
with an i.i.d. query law without changing the sample or interval bounds.

We begin by separating the construction-specific input from the multiscale analysis inherited from \citet{lau2026order} and then give a brief roadmap.  Appendix~\ref{sec:demodulation-details} establishes  Lemma~\ref{lem:per-scale-main}, the main construction-specific input to the analysis below.  This lemma supplies the per-scale first-moment identity and tail-local variance bound for the decoded refinement variables.  With these properties in hand, the remaining upper-bound calculations follow, up to notation and constants, the multiscale argument used to prove the corresponding upper bound \citep[Theorem~5]{lau2026order}.  Specifically, Appendix~\ref{sec:conditioning} controls the truncation bias,
Appendix~\ref{sec:tail-variance} aggregates the variances across scales, Appendix~\ref{sec:amplification} amplifies the resulting
constant-confidence estimate, and Appendix~\ref{sec:sample-complexity} sums the scale budgets.  We give these calculations in full to make the multiscale part of the proof
self-contained and to verify its application to our decoder-selected cover.
The interval-bound verification at the end of
Appendix~\ref{sec:sample-complexity} is additional to the inherited
multiscale analysis.  

Set $\delta_{\rm loc}=\delta_{\rm ref}=\delta/2$, and fix any
$c\in[-\lambda,\lambda]$ satisfying $|c-\mu|\le8\sigma$.  Throughout this
appendix, expectations, variances, and probabilities refer to the refinement
samples and randomness with $c$ held fixed.  All bounds are uniform over such
$c$.  Because localization and refinement use disjoint samples and independent
randomness, conditioning on the localization transcript leaves the refinement
law unchanged.  The bounds below therefore apply whenever localization
succeeds.

\subsection{Conditional Mean and Truncation Bias}
\label{sec:conditioning}

We first identify the conditional mean of the base refinement estimate and
then bound its difference from the residual $\mu-c$.  This establishes~\eqref{eq:base-first-moment} and~\eqref{eq:base-bias} in Section~\ref{sec:performance-main}.
Write
\[
    \calU(c,r)
    =
    \bigcup_{J\in\calW(c,r)}J
    =
    \bigcup_{i=1}^{\imax}
    \bigcup_{J\in\calW_i(c,r)}J.
\]
By the definition of $\widehat\theta_{\rm base}(c)$ in \eqref{eq:base-estimator-main}, 
the first-moment identity in \eqref{eq:per-scale-properties}, and the
disjointness of the cover in Lemma~\ref{lem:cover-main},
\begin{equation}
\label{eq:base-first-moment-appendix}
\begin{aligned}
    \E[\widehat\theta_{\rm base}(c)]
    &=
    \sum_{i=1}^{\imax}
    \frac{1}{n_i}
    \sum_{t=1}^{n_i}
    \E\left[
        V_{i,t}^{\mathrm L}(c)+V_{i,t}^{\mathrm R}(c)
    \right] \\
    &=
    \sum_{i=1}^{\imax}
    \sum_{J\in\calW_i(c,r)}
    \E\left[(X-c)\cdot\1\{X\in J\}\right] \\
    &=
    \E\left[
        (X-c)\cdot\1\{X\in\calU(c,r)\}
    \right].
\end{aligned}
\end{equation}
Because $\calU(c,r)$ contains $[c-r,c+r]$,
\begin{equation}
\label{eq:cover-complement-tail}
    \mathbb R\setminus\calU(c,r)
    \subseteq
    \{x:|x-c|>r\}.
\end{equation}
On the event $\{|X-c|>r\}$,
\[
    |X-c| \le \frac{|X-c|^k}{r^{k-1}}.
\]
Taking expectations, using the transferred moment bound~\eqref{eq:moment-transfer}, and recalling the choice of $r$ in~\eqref{eq:main-radii}, we obtain
\begin{equation}
\label{eq:moment-tail-bound-detail}
    \E\left[
        |X-c|\cdot\1\{|X-c|>r\}
    \right]
    \le
    \frac{\E[|X-c|^k]}{r^{k-1}}
    \le
    \frac{\bar\sigma^k}{r^{k-1}}
    \le
    \frac{\eps}{2}.
\end{equation}
Since $\mu-c=\E[X-c]$, it follows from
\eqref{eq:base-first-moment-appendix},~\eqref{eq:cover-complement-tail} and
\eqref{eq:moment-tail-bound-detail} that
\begin{equation}
\label{eq:base-bias-appendix}
\begin{aligned}
    \left|
        \E[\widehat\theta_{\rm base}(c)]-(\mu-c)
    \right|
    &=
    \left|
        \E\left[
            (X-c)\cdot\1\{X\notin\calU(c,r)\}
        \right]
    \right|
    \\
    &\le
    \E\left[
        |X-c|\cdot\1\{X\notin\calU(c,r)\}
    \right]
    \\
    &\le
    \E\left[
        |X-c|\cdot\1\{|X-c|>r\}
    \right] 
    \\
    &\le
    \frac{\eps}{2}.
\end{aligned}
\end{equation}

\subsection{Variance Aggregation Across Scales}
\label{sec:tail-variance}

We next aggregate the tail-local variance bounds from Lemma~\ref{lem:per-scale-main} to obtain~\eqref{eq:base-variance-main} in Section~\ref{sec:performance-main}.  The calculation has two steps: we first use
the scale allocation to reduce the variance to a dyadic tail sum, and then
control that sum using the transferred $k$-th moment bound~\eqref{eq:moment-transfer}.

Recall that
\[
    \tau_1(c)=1
    \quad\text{and}\quad
    \tau_i(c)=\Pr(|X-c|\ge\ell_i)
    \quad\text{for }i\ge2.
\]
By the second-moment bound
\eqref{eq:single-query-second-moment}, established in the proof of
Lemma~\ref{lem:per-scale-main}, there is a universal constant $C_0$ such
that, for every scale $i\in\{1,\dots,\imax\}$, repetition index
$t\in\{1,\dots,n_i\}$, and orientation
$d\in\{\mathrm L,\mathrm R\}$,
\begin{equation}
\label{eq:single-query-variance}
    \Var(V_{i,t}^d(c))
    \le
    C_0 \ell_i^2 \cdot \tau_i(c).
\end{equation}
Independence across $(i, t, d)$ therefore
gives
\begin{equation}
\label{eq:base-variance-sum-detail}
    \Var(\widehat\theta_{\rm base}(c))
    =
    \sum_{i=1}^{\imax}
    \frac{1}{n_i^2}
    \sum_{t=1}^{n_i}
    \sum_{d\in\{\mathrm L,\mathrm R\}}
    \Var(V_{i,t}^d(c))
    \le
    2C_0
    \sum_{i=1}^{\imax}
    \frac{\ell_i^2}{n_i}\tau_i(c).
\end{equation}
By the allocation in \eqref{eq:main-ni},
\[
    \frac{\ell_i^2}{n_i}
    \le
    \frac{\eps^2}{C_1}\,2^{k(i-1)},
\]
and substituting this into \eqref{eq:base-variance-sum-detail} yields
\begin{equation}
\label{eq:base-variance-before-tail}
    \Var(\widehat\theta_{\rm base}(c))
    \le
    \frac{2C_0\eps^2}{C_1}
    \sum_{i=1}^{\imax}
    2^{k(i-1)} \cdot \tau_i(c).
\end{equation}
It remains to bound the dyadic tail sum.  Let $Z= |X-c|/ \bar\sigma$.
The transferred moment bound \eqref{eq:moment-transfer} gives
$\E[Z^k]\le1$.  For every fixed $z\ge0$,
\[
    \sum_{j\ge1}2^{kj} \cdot \1\{z\ge2^j\}
    \le
    \frac{2^k \cdot z^k}{2^k-1},
\]
since the sum is zero when $z<2$ and when $z\ge2$, it is a geometric sum over
$j\le\lfloor\log_2z\rfloor$.  Consequently,
\begin{equation}
\label{eq:dyadic-tail-sum}
\begin{aligned}
    \sum_{i=1}^{\imax}2^{k(i-1)}\tau_i(c)
    &=
    1+
    \sum_{j=1}^{\imax-1}
    2^{kj} \cdot\Pr(Z\ge2^j)
    \\
    &\le
    1+
    \E\left[
        \sum_{j\ge1}2^{kj} \cdot \1\{Z\ge2^j\}
    \right]
    \\
    &\le
    1+
    \frac{2^k}{2^k-1}\E[Z^k]
    \\
    &\le
    3.
\end{aligned}
\end{equation}
Combining \eqref{eq:base-variance-before-tail} and
\eqref{eq:dyadic-tail-sum} gives 
$\Var(\widehat\theta_{\rm base}(c)) \le 6C_0/C_1 \cdot  \eps^2$.
Taking $C_1\ge384C_0$ therefore ensures that
\begin{equation}
\label{eq:base-variance}
    \Var(\widehat\theta_{\rm base}(c))
    \le
    \frac{\eps^2}{64}.
\end{equation}

\subsection{Accuracy and Confidence Amplification}
\label{sec:amplification}

We now convert the bias and variance bounds
\eqref{eq:base-bias-appendix} and~\eqref{eq:base-variance} into the required
high-probability guarantee.  We first prove a constant-success guarantee for $\widehat\theta_{\rm base}(c)$ in \eqref{eq:base-estimator-main}, the estimate produced by one base refinement block.  We then amplify this guarantee by taking the median of estimates from independent blocks and finally combine the refinement and localization failure probabilities.

By Chebyshev's inequality and the variance bound~\eqref{eq:base-variance},
\[
    \Pr\left(
        \left|
            \widehat\theta_{\rm base}(c)
            -\E[\widehat\theta_{\rm base}(c)]
        \right|
        >\frac{\eps}{2}
    \right)
    \le
    \frac{\eps^2/64}{(\eps/2)^2}
    =
    \frac{1}{16}.
\]
On the complementary event, the triangle inequality and
bias bound~\eqref{eq:base-bias-appendix} give
\[
    |\widehat\theta_{\rm base}(c)-(\mu-c)|
    \le
    \left|
        \widehat\theta_{\rm base}(c)
        -\E[\widehat\theta_{\rm base}(c)]
    \right|
    +
    \left|
        \E[\widehat\theta_{\rm base}(c)]-(\mu-c)
    \right|
    \le
    \eps.
\]
Hence, the estimate produced by one base refinement block satisfies
\begin{equation}
\label{eq:base-failure}
    \Pr\left(
        |\widehat\theta_{\rm base}(c)-(\mu-c)|>\eps
    \right)
    \le
    \frac{1}{16}.
\end{equation}
Let $K$ be the smallest odd integer satisfying
\[
    K
    \ge
    \frac{128}{49}\log\frac{1}{\delta_{\rm ref}},
\]
so that $K=O(\log(1/\delta_{\rm ref}))$.  Form $K$ independent base
refinement estimates
$\widehat\theta_{\rm base}^{(1)}(c),\ldots,
\widehat\theta_{\rm base}^{(K)}(c)$, and let $\widehat\theta$ be their median.
For each block $b$, define its failure indicator
\[
    B_b
    =
    \1\left\{
        |\widehat\theta_{\rm base}^{(b)}(c)-(\mu-c)|>\eps
    \right\}.
\]
The variables $B_1,\ldots,B_K$ are independent and satisfy
$\E[B_b]\le1/16$ by \eqref{eq:base-failure}.  The median can fail only if at
least $(K+1)/2$ blocks fail.  Hoeffding's inequality therefore gives
\begin{equation}
\label{eq:median-failure}
\begin{aligned}
    \Pr\left(
        |\widehat\theta-(\mu-c)|>\eps
    \right)
    &\le
    \Pr\left(
        \sum_{b=1}^K B_b\ge\frac{K+1}{2}
    \right) \\
    &\le
    \Pr\left(
        \sum_{b=1}^K B_b
        -\E\left[\sum_{b=1}^K B_b\right]
        \ge\frac{7K}{16}
    \right) \\
    &\le
    \exp\left(-\frac{49K}{128}\right) \\
    &\le
    \delta_{\rm ref}.
\end{aligned}
\end{equation}
To combine the localization and refinement guarantees, we now take
probability over both sources of randomness.  Let $\mathcal E_{\rm loc}$
denote the localization-success event
$\{|c-\mu|\le8\sigma\}$.  The preceding bound is uniform over every
localization output on $\mathcal E_{\rm loc}$.  Since the final estimate is
$\widehat\mu=c+\widehat\theta$,
\[
    \Pr(|\widehat\mu-\mu|>\eps)
    \le
    \Pr(\mathcal E_{\rm loc}^{\mathsf c})
    +
    \Pr\left(
        |\widehat\theta-(\mu-c)|>\eps,\,
        \mathcal E_{\rm loc}
    \right) 
    \le
    \delta_{\rm loc}+\delta_{\rm ref}
    =
    \delta.
\]

\subsection{Sample Complexity}
\label{sec:sample-complexity}

It remains to count the samples, accounting for the
number of queried scales and the geometric allocation of repetitions across
those scales.

One base refinement block uses $2\sum_{i=1}^{\imax}n_i$
samples, where the factor two corresponds to the two orientations
$d\in\{\mathrm L,\mathrm R\}$.  Since
$\ell_i=2^{i-1}\bar\sigma$, the definition of $\imax$ in
\eqref{eq:main-radii} is equivalent to
\[
    \imax
    =
    \min\left\{
      i\ge1:
      \frac{2\bar\sigma}{\eps} \le 2^{(i-1)(k-1)}
    \right\}.
\]
Writing $u=\bar\sigma/\eps$, we therefore have
\begin{equation}
\label{eq:imax-order-detail}
    \imax-1
    =
    \left\lceil
        \frac{\log_2(2u)}{k-1}
    \right\rceil
    \quad\text{and} \quad
    2^{\imax-1}
    =
    \Theta_k\left(u^{1/(k-1)}\right).
\end{equation}
Using $\lceil x\rceil\le x+1$ in the sample allocation $n_i$ in~\eqref{eq:main-ni} gives
\begin{equation}
\label{eq:sample-geometric-preliminary}
    \sum_{i=1}^{\imax} n_i
    \le
    \sum_{i=1}^{\imax} \left(C_1 u^2
      \cdot 2^{(i-1)(2-k)}  +1\right)
    \le
    C_1u^2
    \sum_{j=0}^{\imax-1}2^{(2-k)j}
    +\imax.
\end{equation}
The ratio of the geometric sum is $2^{2-k}$.  Hence, the first scale dominates
when $k>2$, all scales contribute equally when $k=2$, and the final scale
dominates when $1<k<2$.  Using \eqref{eq:imax-order-detail} in the last case
gives
\[
    u^2 \cdot 2^{(2-k)(\imax-1)}
    =
    \Theta_k\left(
        u^{\,2+(2-k)/(k-1)}
    \right)
    =
    \Theta_k\left(
        u^{k/(k-1)}
    \right).
\]
Combining all three cases gives
\begin{equation}
\label{eq:base-sample-cases}
    \sum_{i=1}^{\imax}n_i
    =
    \begin{cases}
    O_k\left((\bar\sigma/\eps)^2\right),
        & k>2,\\[0.3em]
    O\left(
        (\bar\sigma/\eps)^2
        \log(\bar\sigma/\eps)
    \right),
        & k=2,\\[0.3em]
    O_k\left(
        (\bar\sigma/\eps)^{k/(k-1)}
    \right),
        & 1<k<2.
    \end{cases}
\end{equation}
Here the additive term $\imax=O_k(\log u)$ in
\eqref{eq:sample-geometric-preliminary} is absorbed in all three cases.

Multiplying \eqref{eq:base-sample-cases} by the two orientations and the
$K=O(\log(1/\delta_{\rm ref}))$ independent blocks, and using
$\bar\sigma=9\sigma$ and $\delta_{\rm ref}=\delta/2$,\footnote{Under the
standing convention that logarithmic factors in upper-bound rate expressions
are at least one,
$\log(\bar\sigma/\eps)=O(\log(\sigma/\eps))$ and
$\log(1/\delta_{\rm ref})=O(\log(1/\delta))$.}
gives the total refinement sample complexity

\[
    n_{\rm ref}
    =
    \begin{cases}
    O_k\left(
        (\sigma/\eps)^2 \cdot \log(1/\delta)
    \right),
        & k>2,\\[0.3em]
    O\left(
        (\sigma/\eps)^2
        \cdot \log(\sigma/\eps)
        \cdot \log(1/\delta)
    \right),
        & k=2,\\[0.3em]
    O_k\left(
        (\sigma/\eps)^{k/(k-1)}
        \cdot \log(1/\delta)
    \right),
        & 1<k<2.
    \end{cases}
\]
Lemma~\ref{lem:LS-localization} uses
\[
    n_{\rm loc}
    =
    O\left(
        \log\frac{\lambda}{\sigma}
        +
        \log\frac{1}{\delta_{\rm loc}}
    \right)
\]
additional samples for localization.  Recalling that~$\delta_{\rm loc}=\delta_{\rm ref}=\delta/2$, the localization-confidence term has the
same order as the amplification factor.  It is absorbed by every refinement
term since $\sigma/\eps>1$.  Hence,
$n_{\rm loc}+n_{\rm ref}$ has the order stated in
\eqref{eq:main-rate-compact}.

It remains only to verify the asserted query structure.
Appendix~\ref{sec:query-details} shows that the one-set of every refinement
query is a union of $O(\lambda/\sigma)$ intervals, and
Lemma~\ref{lem:LS-localization} gives the same bound for every localization
query.  The analysis above applies to independent queries with type-dependent laws.
Appendix~\ref{app:localization-iid} converts the prescribed query counts into
one i.i.d. query law without changing the sample order or interval bound.
This completes the proof of Theorem~\ref{thm:main}.
\section{Proof of the Sample--Interval Tradeoff
  (Theorem~\ref{thm:interval-complexity})}
\label{app:interval-complexity}

This appendix proves Theorem~\ref{thm:interval-complexity}.  For the lower bound, Appendix~\ref{sec:interval-complexity-noisy-lower} combines the unrestricted lower bound in Theorem~9 of \citet{lau2026order} with an extension of their adaptivity-gap argument for Theorem~11 from one interval to $s$ intervals per query.
For the upper bound, Appendix~\ref{sec:grouped-refinement-proof} adapts the refinement construction in Section~\ref{sec:construction} by dividing the query banks introduced in~\eqref{eq:main-query-bank} into groups of at most $s$ cells and verifies that
the groupwise targets and variances recombine correctly.  In contrast, adapting the localization strategy to the same interval restriction requires a more substantial redesign, developed in Appendix~\ref{app:sharp-localization}.
Appendix~\ref{sec:interval-complexity-assembly} then combines the resulting
sample counts, and Appendix~\ref{app:localization-iid} converts the prescribed
localization and refinement families into a single i.i.d. query law.

\subsection{Interval-Dependent Lower Bound}
\label{sec:interval-complexity-noisy-lower}

The class of non-adaptive 1-bit estimators whose queries are $s$-interval
is contained in the class of all 1-bit mean estimators.
Consequently, the unrestricted lower bound in~\citep[Theorem~9]{lau2026order} already gives the localization and refinement terms in~\eqref{eq:full-s-tradeoff}. It remains to prove the interval-dependent term. We follow the informative-query argument used in the proof of \citep[Theorem 11]{lau2026order}, which corresponds to $s=1$.  We use the same hard family and the same strategy of averaging the per-location testing errors.  The only change caused by allowing $s$ intervals is that one query
can be informative for at most $2s$ hard locations, rather than at most two.
We prove this count and then record the resulting testing inequalities.

If $\lambda<2\sigma$, the interval-dependent term is already dominated by
the unrestricted refinement lower bound.  We may therefore assume
henceforth that $\lambda\ge2\sigma$. The hard family used in the proof of \citep[Theorem 11]{lau2026order} provides $N=\Theta(\lambda/\sigma)$ pairwise disjoint two-point support segments $I_1,\ldots,I_N$ and, at each location $j$, two distributions $P_{j,-},P_{j,+}\in\calP(k,\lambda,\sigma)$ supported on the endpoints of $I_j$.  Their means differ by more than $2\eps$.  If a deterministic query takes the same value at the two endpoints of $I_j$, then its response laws under $P_{j,-}$ and $P_{j,+}$ coincide.  If it separates the endpoints, its response laws are
$\operatorname{Bern}(1/2+a)$ and $\operatorname{Bern}(1/2-a)$, in either
order, where $a=\eps/\sigma$.  By decreasing the constant in the
small-accuracy assumption if necessary, we may assume that $a\le1/4$.
The two directed KL divergences coincide, and a direct calculation gives
\begin{equation}
\label{eq:informative-response-kl}
    D_{\rm KL}\left(
      \operatorname{Bern}\left(\frac12+a\right)
      \,\middle\|\,
      \operatorname{Bern}\left(\frac12-a\right)
    \right)
    =
    2a\cdot\log\left(\frac{1+2a}{1-2a}\right)
    \le
    16a^2
    \le
    C\frac{\eps^2}{\sigma^2}.
\end{equation}
These are the only properties of the hard family needed below.

Consider an arbitrary randomized non-adaptive protocol using $n$ samples, and
condition on all its internal randomness $R$.  The queries are then
deterministic.  Call a query \emph{informative for location $j$} if it
separates the two endpoints of $I_j$, and let $n_j(R)$ be the number of
queries informative for that location.

For any fixed $s$-interval query $Q$, its one-set $A=Q^{-1}(1)$ has at most $2s$ finite
boundary points.  If $Q$ is informative for location $j$, then $I_j$
contains one of these boundary points.  Since the segments
$I_1,\ldots,I_N$ are pairwise disjoint, each boundary point can belong to at
most one of them.  Hence, each query is informative for at most $2s$
locations, and consequently
\begin{equation}
\label{eq:total-informative-count}
    \sum_{j=1}^N n_j(R)\le2sn
    \qquad\text{for every realization of }R.
\end{equation}
We now complete the per-location testing argument.  Conditional on $R$, let
$\mathbb P_{j,-}^R$ and $\mathbb P_{j,+}^R$ denote the two
response-transcript laws at location $j$, and define
$
    D_j(R)
    =
    D_{\rm KL}\left(
      \mathbb P_{j,+}^R
      \,\middle\|\,
      \mathbb P_{j,-}^R
    \right).
$
Uninformative queries contribute zero divergence.  Since the observations
are independent, tensorization of KL divergence
\citep[Theorem~2.16(c)]{Polyanskiy_Wu_2025} and
\eqref{eq:informative-response-kl} give
\[
    D_j(R)
    \le
    C n_j(R)\frac{\eps^2}{\sigma^2}.
\]
Thresholding the estimate at the midpoint of the two means gives a binary
test.  Since the means differ by more than $2\eps$, an incorrect decision
implies estimation error greater than $\eps$.  Let $e_j(R)$ be the
conditional testing error, averaged over the two hypotheses.  The
Bretagnolle--Huber and Pinsker inequalities
\citep[see, e.g.,][]{tsybakov2009introduction} give, respectively,
\[
    e_j(R)
    \ge
    \frac{1}{4}\exp\left(-D_j(R)\right)
    \quad \text{and} \quad
    e_j(R)
    \ge
    \frac{1}{2}-\sqrt{\frac{D_j(R)}8}.
\]
Set $e_j=\E_R[e_j(R)]$ and
$\bar e=N^{-1}\sum_{j=1}^N e_j$.  Averaging the preceding inequalities over
$R$ and $j$, applying Jensen's inequality, and using
\eqref{eq:total-informative-count} yield
\begin{equation}
\label{eq:average-testing-error-lower}
    \bar e
    \ge
    \frac{1}{4}\exp\left(
      -C\frac{sn\eps^2}{N\sigma^2}
    \right),
    \quad \text{and} \quad
    \bar e
    \ge
    \frac{1}{2}
    -
    C\sqrt{
      \frac{sn\eps^2}{N\sigma^2}
    }.
\end{equation}
An $(\eps,\delta)$-PAC estimator satisfies $\bar e\le\delta$.
Fix a sufficiently small universal constant
$\delta_0\in(0,1/3)$.  If $\delta\le\delta_0$, the first inequality in
\eqref{eq:average-testing-error-lower} gives
\begin{equation}
\label{eq:interval-lower-with-N}
    n
    =
    \Omega\left(
      \frac{N\sigma^2}{s\eps^2}
      \log\frac{1}{\delta}
    \right).
\end{equation}
If $\delta\in[\delta_0,1/3)$, the second inequality in
\eqref{eq:average-testing-error-lower} gives
$n=\Omega(N\sigma^2/(s\eps^2))$.  Since
$\log(1/\delta)=\Theta(1)$ uniformly over this latter range,
\eqref{eq:interval-lower-with-N} again follows.
Using $N=\Theta(\lambda/\sigma)$, we conclude that
\begin{equation}
\label{eq:noisy-interval-lower}
    n
    =
    \Omega\left(
      \frac{\lambda\sigma}{s\eps^2}\log\frac{1}{\delta}
    \right).
\end{equation}
Combining~\eqref{eq:noisy-interval-lower} with the unrestricted lower bound
of \citet[Theorem~9]{lau2026order} proves the lower side of
\eqref{eq:full-s-tradeoff}.

\subsection{Grouped Refinement under the Interval Budget}
\label{sec:grouped-refinement-proof}

This subsection adapts the refinement construction of
Sections~\ref{sec:queries-main}--\ref{sec:filter-main} to the interval budget.
At each scale, we partition the fixed query bank into groups of at most $s$
cells.  For each group, we use independent query randomness and an independent
observation, with the query bank restricted to that group.  After
localization, the decoder likewise restricts the cover and filter to the
group.  With these replacements, the filtered-demodulation and
stochastic-quantization identities from
Lemmas~\ref{lem:filtered-demodulation} and~\ref{lem:cell-stochastic} apply
groupwise.

The additional work is to verify that the groupwise first-moment targets
partition the original target, that the independent group variances add, and
that the enlarged query budget gives the desired interval-dependent term.
We establish the first two properties here.
Appendix~\ref{sec:interval-complexity-assembly} then combines them with the
bias, scale allocation, and amplification analysis of
Appendix~\ref{app:upper-bound} and sums the resulting groupwise query counts.
Table~\ref{tab:grouped-refinement-correspondence} previews the correspondence;
the grouped queries and decoder are defined formally below.

\begin{table}[htbp]
\centering
\small
\caption{Correspondence between the ungrouped refinement construction and
its group-$g$ counterpart.}
\label{tab:grouped-refinement-correspondence}
\begin{tabular}{@{}p{0.42\linewidth}p{0.52\linewidth}@{}}
\toprule
Ungrouped refinement & Group-$g$ counterpart \\
\midrule
$\calD_i^{\rm bank}$ in~\eqref{eq:main-query-bank}
  & consecutive group
    $\calG_{i,g}\subseteq\calD_i^{\rm bank}$ \\

$Q_{i,t}^d$ in~\eqref{eq:main-query}
  & independent group-restricted query
    $Q_{i,g,t}^d$ in~\eqref{eq:grouped-refinement-query} \\

$\calW_i(c,r)$ in~\eqref{eq:main-cover-map}
  & $\calT_{i,g}(c,r)$ in~\eqref{eq:grouped-cover-filter} \\

$\calF_i(c)$ in~\eqref{eq:main-filter}
  & $\calF_{i,g}(c)$ in~\eqref{eq:grouped-cover-filter} \\

$H_{i,t}^d(c)$ and $p_i(c)$ in~\eqref{eq:filter-indicator}
  & $H_{i,g,t}^d(c)$ and $p_{i,g}(c)$ in
   ~\eqref{eq:grouped-filter-indicator} \\

$V_{i,t}^d(c)$ in~\eqref{eq:main-decoded-variable}
  & group-restricted decoder $V_{i,g,t}^d(c)$ in
   ~\eqref{eq:grouped-decoded-variable} \\
\bottomrule
\end{tabular}
\end{table}

Appendix~\ref{app:sharp-localization} constructs an $s$-interval localizer
that supplies a center $c$ satisfying $|c-\mu|\le8\sigma$.  For the analysis
below, condition on successful localization and on the realized center $c$,
and set $\bar\sigma=9\sigma$.  Then $\E[|X-c|^k]\le\bar\sigma^k$
by~\eqref{eq:moment-transfer}.  We retain the scales $\ell_i$, cutoff
$\imax$, truncation radius $r$, fixed banks~$\calD_i^{\rm bank}$, covers~$\calW_i(c,r)$, 
and filters~$\calF_i(c)$ from Section~\ref{sec:construction}.
All probabilities, expectations, and variances below are conditional on this
fixed $c$ and are taken over the independent refinement observations and
public randomness.

\paragraph{Partitioning the fixed bank.}
By Lemma~\ref{lem:cover-main}, $\calD_i^{\rm bank}$ contains every
scale-$i$ cover cell.  The filter cells belong to the bank by definition.
Enumerate the bank from left to right as
\[
    \calD_i^{\rm bank}
    =
    \{J_{i,1},\ldots,J_{i,M_i}\}
    \quad \text{where} \quad
    M_i
    \le C\left(1+\frac{\lambda}{\ell_i}\right).
\]
The cardinality bound is~\eqref{eq:bank-cardinality-detail}.
Partition this list into consecutive blocks of size $s$, except possibly
the last, and denote them by~$\calG_{i,1},\ldots,\calG_{i,G_i}$.  Then
\begin{equation}
\label{eq:number-scale-groups}
    G_i
    =
    \left\lceil\frac{M_i}{s}\right\rceil
    \le
    1+\frac{C}{s}\left(1+\frac{\lambda}{\ell_i}\right) 
    \le
    C'\max\left\{1,\frac{\lambda}{s\ell_i}\right\},
\end{equation}
where the final inequality uses $s\ge1$.  We observe that each replicate at scale $i$ uses
$G_i$ observations per orientation, one for each group.
Lemma~\ref{lem:grouped-refinement-moments} below shows that grouping preserves the
required moment bounds, so $G_i$ enters only through the sample count.

\paragraph{Grouped queries and decoder.}
For every $(i,g,t,d)$, independently of all other index tuples, draw a split
fraction $\zeta_{i,g,t}^d$ and Rademacher signs
$\{\eta_{J,i,g,t}^d:J\in\calG_{i,g}\}$ as in
Section~\ref{sec:queries-main}.  For each $J\in\calG_{i,g}$, define $A_{J,i,g,t}^d$ by the same formula as
$A_{J,t}^d$ in~\eqref{eq:candidate_subintervals}, with
$\zeta_{i,t}^d$ replaced by $\zeta_{i,g,t}^d$.
The grouped version of~\eqref{eq:main-query} is
\begin{equation}
\label{eq:grouped-refinement-query}
    Q_{i,g,t}^d(x)
    =
    \1\left\{
      x\in
      \bigcup_{\substack{J\in\calG_{i,g}\\
                         \eta_{J,i,g,t}^d=+1}}
      A_{J,i,g,t}^d
    \right\}.
\end{equation}
Every realized query in~\eqref{eq:grouped-refinement-query} is a union of at
most $s$ bounded intervals.  Let
\[
    Y_{i,g,t}^d
    =
    Q_{i,g,t}^d(X_{i,g,t}^d),
\]
where the observations $X_{i,g,t}^d$ are independent across all indices and
independent of the public randomness.
The partition and all grouped queries are fixed before the localization
output $c$ is observed.  After localization, the decoder restricts the selected cover and filter to
each group by defining
\begin{equation}
\label{eq:grouped-cover-filter}
    \calT_{i,g}(c,r)
    =
    \calW_i(c,r)\cap\calG_{i,g}
    \quad \text{and} \quad
    \calF_{i,g}(c)
    =
    \calF_i(c)\cap\calG_{i,g}.
\end{equation}
These are the group-$g$ counterparts summarized in
Table~\ref{tab:grouped-refinement-correspondence}.  Define the associated
filter event and its probability by
\begin{equation}
\label{eq:grouped-filter-indicator}
    H_{i,g,t}^d(c)
    =
    \1\left\{
      \eta_{J,i,g,t}^d=-1
      \text{ for every }J\in\calF_{i,g}(c)
    \right\}
    \quad\text{and} \quad
    p_{i,g}(c)
    =
    2^{-|\calF_{i,g}(c)|}
    \ge\frac{1}{8},
\end{equation}
where the inequality follows because
$\calF_{i,g}(c)\subseteq\calF_i(c)$ and $|\calF_i(c)|\le3$.
Substituting these group-restricted objects into
\eqref{eq:main-decoded-variable} gives
\begin{equation}
\label{eq:grouped-decoded-variable}
    V_{i,g,t}^d(c)
    =
    \frac{
      2H_{i,g,t}^d(c)\cdot Y_{i,g,t}^d
    }{
      p_{i,g}(c)
    }
    \cdot
    \sum_{J\in\calT_{i,g}(c,r)}
      w_J^d(c)\cdot\eta_{J,i,g,t}^d.
\end{equation}

The next lemma verifies the grouped analogues of the two per-scale properties
in Lemma~\ref{lem:per-scale-main}.  The first identity shows that the
groupwise targets sum to the original first-moment target.  The second shows
that the group-summed statistic satisfies the same tail-local variance bound.
\begin{lemma}[Grouped Per-Scale Moment Bounds]
\label{lem:grouped-refinement-moments}
Let $\tau_1(c)=1$ and, for $i\ge2$, let
$\tau_i(c)=\Pr(|X-c|\ge\ell_i)$.  For every scale $i$ and replicate $t$,
\begin{equation}
\label{eq:grouped-lemma-first-moment}
    \sum_{g=1}^{G_i}
    \E\left[
      V_{i,g,t}^{\mathrm L}(c)+V_{i,g,t}^{\mathrm R}(c)
    \right]
    =
    \sum_{J\in\calW_i(c,r)}
      \E\left[(X-c)\cdot\1\{X\in J\}\right].
\end{equation}
Moreover, for every $d\in\{\mathrm L,\mathrm R\}$,
\begin{equation}
\label{eq:grouped-lemma-second-moment}
\begin{aligned}
    \Var\left(\sum_{g=1}^{G_i}V_{i,g,t}^d(c)\right)
    =
    \sum_{g=1}^{G_i}\Var\left(V_{i,g,t}^d(c)\right)
    \le
    \sum_{g=1}^{G_i}\E\left[(V_{i,g,t}^d(c))^2\right]
    \le
    2304\ell_i^2\cdot\tau_i(c).
\end{aligned}
\end{equation}
\end{lemma}

\begin{proof}
Fix $i,g,t,d$ and condition first on the split
$\zeta_{i,g,t}^d$.  The intervals
$\{A_{J,i,g,t}^d:J\in\calG_{i,g}\}$ are then fixed and pairwise disjoint.
Moreover, Lemma~\ref{lem:filter-geometry} gives
\[
  \calT_{i,g}(c,r)\cap\calF_{i,g}(c)=\varnothing.
\]
We may therefore apply Lemma~\ref{lem:filtered-demodulation} with
\[
  \calI=\calG_{i,g},
  \qquad
  A_J=A_{J,i,g,t}^d,
  \qquad
  \calT=\calT_{i,g}(c,r),
  \qquad
  \calF=\calF_{i,g}(c),
  \qquad
  w_J=w_J^d(c),
  \qquad
  X=X_{i,g,t}^d.
\]
Its first pointwise identity
\eqref{eq:filtered-first-identity-abstract}, followed by averaging over the
observation and the split, gives
\[
  \E[V_{i,g,t}^d(c)]
  =
  \sum_{J\in\calT_{i,g}(c,r)} 
    w_J^d(c) \cdot \Pr(X_{i,g,t}^d\in A_{J,i,g,t}^d).
\]
Summing this identity over $d\in\{\mathrm L,\mathrm R\}$ and applying
Lemma~\ref{lem:cell-stochastic} recovers
\[
  \sum_{J\in\calT_{i,g}(c,r)}
    \E[(X-c) \cdot \1\{X\in J\}].
\]
The sets $\calT_{i,g}(c,r)$ partition $\calW_i(c,r)$ as $g$ varies, which
proves~\eqref{eq:grouped-lemma-first-moment}.
For the second moment, the identity~\eqref{eq:filtered-integrated-second-identity-abstract} from Lemma~\ref{lem:filtered-demodulation} gives, conditional on
$\zeta_{i,g,t}^d$,
\begin{equation}
\label{eq:grouped-one-group-second-moment}
    \E\left[
      (V_{i,g,t}^d(c))^2
      \,\middle|\,
      \zeta_{i,g,t}^d
    \right]
    =
    \frac{2}{p_{i,g}(c)}
    \left(
      \sum_{J\in\calT_{i,g}(c,r)}
        (w_J^d(c))^2
    \right)
    \cdot
    \Pr\left(
      X_{i,g,t}^d
      \in
      \bigcup_{K\in\calG_{i,g}\setminus\calF_{i,g}(c)}
        A_{K,i,g,t}^d
      \,\middle|\,
      \zeta_{i,g,t}^d
    \right).
\end{equation}
For $i\ge2$, every
$K\in\calG_{i,g}\setminus\calF_{i,g}(c)$ lies outside
$\calF_i(c)$ and therefore satisfies $\dist(c,K)\ge\ell_i$.  Since
$A_{K,i,g,t}^d\subseteq K$, the conditional probability in
\eqref{eq:grouped-one-group-second-moment} is at most
\[
    \Pr(|X-c|\ge\ell_i)=\tau_i(c).
\]
For $i=1$, it is at most $\tau_1(c) = 1$.  Since $p_{i,g}(c)\ge1/8$, taking expectation over the split
therefore yields
\[
  \E[(V_{i,g,t}^d(c))^2]
  \le
  16\tau_i(c)
  \sum_{J\in\calT_{i,g}(c,r)}(w_J^d(c))^2.
\]
The groupwise target sets partition $\calW_i(c,r)$.  Lemma~\ref{lem:cover-main}
gives $|\calW_i(c,r)|\le9$ and $|w_J^d(c)|<4\ell_i$, and hence
\[
  \sum_{g=1}^{G_i}\E[(V_{i,g,t}^d(c))^2]
  \le
  16\tau_i(c)\cdot9(4\ell_i)^2
  =2304\ell_i^2 \cdot \tau_i(c).
\]
Finally, different groups use independent observations and public randomness.
Hence, the variance of their sum is the sum of their variances, which is at most
the sum of their second moments.  This proves the second assertion.
\end{proof}
Lemma~\ref{lem:grouped-refinement-moments} shows that the group-summed decoded
variables satisfy the same per-scale first-moment identity and tail-local
variance bound as the ungrouped variables in
Lemma~\ref{lem:per-scale-main}.  Consequently, the bias, variance
aggregation, scale allocation, and amplification arguments of
Appendix~\ref{app:upper-bound} apply after each ungrouped scale statistic is
replaced by its sum over groups.  The only sample-count change is that one
replicate at scale $i$ uses $G_i$ observations per orientation.
Appendix~\ref{sec:interval-complexity-assembly} accounts for this factor.

\subsection{\texorpdfstring{$s$}{s}-Interval Localization}
\label{app:sharp-localization}

This subsection proves the localization guarantee used in the upper bound of Theorem~\ref{thm:interval-complexity}.  The main challenge is to restrict every query one-set to at most $s$ intervals while preserving~$\log(\lambda/\sigma)$ as an additive coding term.  The precise guarantee is
as follows.

\begin{lemma}[\texorpdfstring{$s$}{s}-interval localization]
\label{lem:sharp-s-localization}
For every $\lambda\ge\sigma>0$, every integer $s\ge1$, and
$\delta_{\rm loc}\in(0,1)$, there is a randomized non-adaptive protocol
using independent query functions that returns $c\in[-\lambda,\lambda]$
satisfying
\[
    \Pr\left(
      |c-\mu|\le8\sigma
    \right)
    \ge
    1-\delta_{\rm loc},
\]
uniformly over all distributions with $\mu\in[-\lambda,\lambda]$ and
$\E[|X-\mu|]\le\sigma$.  Every realized one-set has at most $s$ interval
components, and the sample complexity is
\begin{equation}
\label{eq:sharp-s-localization}
    n_{\rm loc}^{(s)}
    =
    O\left(
      \log\frac{\lambda}{\sigma}
      +
      \max\left\{
        1,\frac{\lambda}{s\sigma}
      \right\}
      \cdot\log\frac{1}{\delta_{\rm loc}}
    \right).
\end{equation}
\end{lemma}
The construction below uses independent, type-dependent query families.
Appendix~\ref{sec:iid-s-localization} converts them into the single i.i.d.\
query law used in Theorem~\ref{thm:interval-complexity}.
Before giving the formal construction, we explain why direct grouping is
suboptimal and give a brief roadmap of the two localization regimes.  The
query families and their analysis are defined formally in the proof of
Lemma~\ref{lem:sharp-s-localization}, following
Lemma~\ref{lem:random-threshold-localization}.

\paragraph{Why direct grouping loses the additive coding term.}
The coding localizer of Lemma~\ref{lem:LS-localization} represents the
possible locations of a mean in $[-\lambda,\lambda]$ by
$\Theta(\lambda/\sigma)$ bins of width
$\Theta(\sigma)$.  The
one-set for a code coordinate is a union of these bins and may therefore have
$O(\lambda/\sigma)$ interval components. Splitting each such one-set into
subunions of at most $s$ components, as in
Appendix~\ref{sec:grouped-refinement-proof}, produces
$G=O(\max\{1,\lambda/(s\sigma)\})$ groups per coordinate. A direct groupwise implementation then incurs this
factor $G$ in the sample count, giving sample
count
$O\left( G
      \cdot
      \left(
        \log(\lambda/\sigma)
        +
        \log(1/\delta_{\rm loc})
      \right)
    \right).
$
Compared to~\eqref{eq:sharp-s-localization}, this gives the desired confidence cost
$G\log(1/\delta_{\rm loc})$ but the additive coding cost~$\log(\lambda/\sigma)$ is replaced by $G\log(\lambda/\sigma)$. Direct grouping is therefore suboptimal in general.  The construction below avoids this multiplicative dependence.

\paragraph{Random thresholds for
  \texorpdfstring{$s\le8$}{s <= 8}.}
For $s\le8$, draw independent thresholds
$T\sim\operatorname{Unif}[-\lambda-4\sigma,\lambda+4\sigma]$ and use the
threshold query $Q_T(x)=\1\{x\le T\}$.
The proof below applies
Lemma~\ref{lem:random-threshold-localization} with margin $h=4\sigma$ to
show that these queries localize the mean to within $8\sigma$ using
$O((\lambda/\sigma)\cdot\log(1/\delta_{\rm loc}))$ samples.  
Since $s\le8$, this has the order required in~\eqref{eq:sharp-s-localization}.

\paragraph{Block and offset for
  \texorpdfstring{$s>8$}{s > 8}.}
For $s>8$, we use the block-and-offset construction illustrated in
Figure~\ref{fig:shifted-block-localization}.  For each of three shifts, we
choose $L=\Theta\left(\max\left\{\sigma, \lambda/s\right\}\right)$
and partition $\mathbb R$ into blocks of width $L$.  The constants are
chosen so that at most $s$ blocks in each partition intersect
$[-\lambda,\lambda]$.  For a fixed shift, localizing $\mu$ amounts to
identifying the block containing it and estimating its offset from that
block's left endpoint.  To identify the block, we apply the coding mechanism of
Lemma~\ref{lem:LS-localization}, treating the blocks that intersect
$[-\lambda,\lambda]$ as the candidate locations.  To form an offset query, we draw a single
$T\sim\operatorname{Unif}[0,L]$.  For every block $[b,b+L)$ that intersects
$[-\lambda,\lambda]$, we include the subinterval $[b,b+T]$ in the query
one-set.  When the observation $X$ falls in the block $[b,b+L)$ containing
$\mu$, the returned bit is $\1\{X-b\le T\}$, an ordinary random-threshold
response for estimating the offset $\mu-b$.

If $\mu$ is near a block boundary, even an observation close to $\mu$ may
fall in a neighboring block.  We therefore call a shifted partition
\emph{safe} when $\mu$ is at least $L/6$ from every block boundary.  Because
the three boundary phases are separated by $L/3$, at least two of the three
partitions are safe for every $\mu$.  The proof below shows that, with high
probability, the block index and offset are recovered accurately for every
safe partition.  Thus at least two of the three resulting estimates are
accurate, and so is their median.
 
\begin{figure}[!t]
\centering
\textbf{(a) Shifted partitions}\par
\vspace{0.40em}
\begin{tikzpicture}[x=1.10cm,y=0.95cm,font=\small]
  \fill[filterred!10] (-0.7,-0.38) rectangle (1.3,2.38);
  \draw[filterred,densely dashed] (0.3,-0.48)--(0.3,2.48);
  \node[filterred,anchor=west,font=\small] at (0.42,2.30) {$\mu$};
  \foreach \y/\lab in {1.8/{shift $0$},0.9/{shift $L/3$},0/{shift $2L/3$}} {
    \draw[gray!75,thick] (-2.5,\y)--(8.5,\y);
    \node[anchor=east,font=\small] at (-2.75,\y) {\lab};
  }
  \foreach \x in {0,6} {
    \draw[queryblue,very thick] (\x,1.52)--(\x,2.08);
  }
  \foreach \x in {2,8} {
    \draw[queryblue,very thick] (\x,0.62)--(\x,1.18);
  }
  \foreach \x in {-2,4} {
    \draw[queryblue,very thick] (\x,-0.28)--(\x,0.28);
  }
  \node[filterred,anchor=west,font=\small] at (0.52,2.10) {unsafe};
  \node[covergreen!60!black,anchor=west,font=\small] at (0.52,1.20) {safe};
  \node[covergreen!60!black,anchor=west,font=\small] at (0.52,0.30) {safe};
\end{tikzpicture}

\vspace{1.20em}
\textbf{(b) Query families for one shifted partition}\par
\vspace{0.65em}
\begin{tikzpicture}[x=1.55cm,y=1.05cm,font=\small]
  \begin{scope}[xscale=1.50]
    \node[anchor=east,font=\small] at (-0.25,1.00)
      {offset query};
    \node[anchor=east,font=\small] at (-0.25,-0.90)
      {block-index query};

    \foreach \q in {0,1,2,3} {
      \fill[queryblue!22] (\q,0.65) rectangle (\q+0.58,1.35);
      \draw[gray!75] (\q,0.65) rectangle (\q+1,1.35);
    }
    \draw[<->,queryblue,thick] (0,1.65)--(0.58,1.65);
    \node[queryblue,font=\small,anchor=south] at (0.29,1.68) {$T$};

    \draw[filterred,densely dashed,thick] (1.72,0.28)--(1.72,1.68);
    \node[filterred,font=\small,anchor=west] at (1.80,1.48) {$\mu$};
    \draw[<->,filterred,thick] (1,0.36)--(1.72,0.36);
    \node[filterred,font=\small,anchor=north] at (1.36,0.30) {$u_\nu$};

    \foreach \q in {0,1,3} {
      \fill[covergreen!25] (\q,-1.25) rectangle (\q+1,-0.55);
    }
    \foreach \q in {0,1,2,3} {
      \draw[gray!75] (\q,-1.25) rectangle (\q+1,-0.55);
    }
    \foreach \q/\bit in {0/1,1/1,2/0,3/1} {
      \node[font=\small] at (\q+0.5,-0.90) {$\bit$};
    }
  \end{scope}
\end{tikzpicture}
\caption{The shifted block construction.  In panel~(a), each row represents a partition of $\mathbb R$ into blocks of width $L=\Theta\left(\max\left\{\sigma, \lambda/s\right\}\right)$, with the three partitions shifted relative to one another; the blue ticks mark their block boundaries.  The dashed red line marks $\mu$, and the shaded strip consists of the points at distance less than $L/6$ from $\mu$.  Because the three boundary phases are separated
by $L/3$ modulo $L$, this strip contains a boundary from at most one
partition; hence at least two partitions are safe.  In panel~(b), the offset
query uses the same draw $T\in[0,L]$ in every block that intersects
$[-\lambda,\lambda]$: from each block $[b,b+L)$, it includes the subinterval
$[b,b+T]$.  The block-index query includes each such block whose codeword
has a one in the sampled coordinate. Since at most $s$ blocks intersect $[-\lambda,\lambda]$, both one-sets
have at most $s$ interval components.}
\label{fig:shifted-block-localization}
\end{figure}
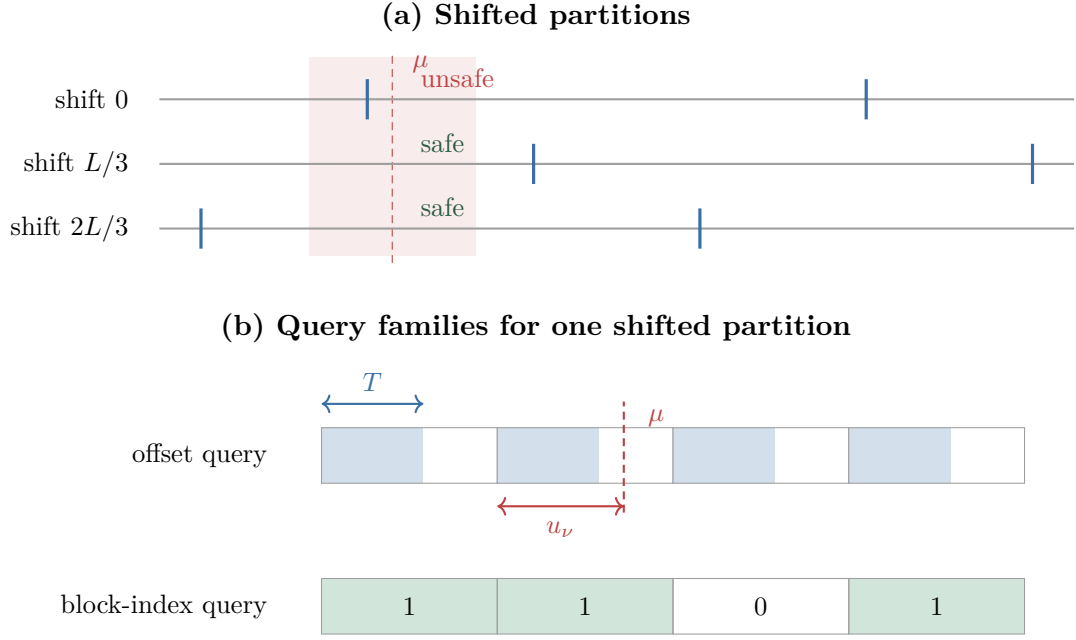

\paragraph{A common random-threshold lemma.}
Both regimes use the following auxiliary lemma with an unknown point
$u\in[0,L]$ and a margin parameter $h$.  The response to a uniformly sampled threshold $T$ must agree with the ideal comparison $\1\{u\le T\}$ with probability at least $3/4$ whenever $|T-u|\ge h$; no condition is imposed when $|T-u|<h$.  The proof of Lemma~\ref{lem:sharp-s-localization} shows that this condition
holds with $h=4\sigma$ for both the $s\le8$ threshold queries and the offset queries under a safe shift.

\begin{lemma}[Random-threshold localization]
\label{lem:random-threshold-localization}
Fix an interval length $L>0$, an unknown point $u\in[0,L]$, a margin
parameter $h\in(0,L]$, and a confidence level $\alpha\in(0,1)$. 
Let $n$ be a positive integer satisfying
\begin{equation}
\label{eq:random-threshold-sample-count}
    n
    \ge
    C\frac{L}{h}
    \cdot\log\frac{2}{\alpha},
\end{equation}
where $C>0$ is a sufficiently large universal constant.
Let $(T_j,Y_j)_{j=1}^n$ be independent pairs, where
$T_j\sim\operatorname{Unif}[0,L]$ and $Y_j\in\{0,1\}$.  Suppose that, for each $j\in\{1,\dots,n\}$ and for Lebesgue-almost every
$t\in[0,L]\setminus(u-h,u+h)$,
\begin{equation}
\label{eq:random-threshold-agreement}
    \Pr\left(
      Y_j=\1\{u \le t\}
      \,\middle|\,
      T_j=t
    \right)
    \ge
    \frac{3}{4}.
\end{equation}
Let $\widehat u$ be any minimizer of the empirical disagreement count:
\begin{equation}
\label{eq:random-threshold-minimizer}
    \widehat u
    \in
    \argmin_{w\in[0,L]}
    \sum_{j=1}^n
      \1\left\{
        Y_j\ne\1\{T_j\ge w\}
      \right\}.
\end{equation}
Then
\[
    \Pr\left(
      |\widehat u-u|\le2h
    \right)
    \ge
    1-\alpha.
\]
\end{lemma}
The proof of Lemma~\ref{lem:random-threshold-localization} is deferred to
Appendix~\ref{app:random-threshold-proof}.  We now use its stated guarantee
to prove Lemma~\ref{lem:sharp-s-localization}.

\begin{proof}[Proof of Lemma~\ref{lem:sharp-s-localization}]
Fix an arbitrary distribution of $X$ with mean
$\mu\in[-\lambda,\lambda]$ and
$\E[|X-\mu|]\le\sigma$.  We construct below a protocol depending only on
$\lambda$, $\sigma$, $s$, and $\delta_{\rm loc}$, and prove its guarantee
for this arbitrary distribution.

\emph{Random thresholds for $s\le8$.}
First suppose $s\le8$.  For each $j\in\{1,\ldots,n\}$, independently draw
$T_j\sim\operatorname{Unif}[-\lambda-4\sigma,\lambda+4\sigma]$ and issue
the threshold query $Q_j(x)=\1\{x\le T_j\}$.  Write $Y_j=Q_j(X_j)$.
For every possible threshold value $t$ with $|t-\mu|\ge4\sigma$,
conditional on $T_j=t$ the event $|X_j-\mu|<4\sigma$ implies
$Y_j=\1\{t\ge\mu\}$.
Since $T_j$ is independent of $X_j$, Markov's inequality therefore gives
\begin{equation}
\label{eq:global-threshold-agreement}
    \Pr\left(
      Y_j=\1\{t\ge\mu\}
      \,\middle|\,
      T_j=t
    \right)
    \ge
    \Pr\left(
      |X_j-\mu|<4\sigma
    \right)
    \ge
    \frac{3}{4}.
\end{equation}
To put the construction in the form required by
Lemma~\ref{lem:random-threshold-localization}, we translate the threshold
interval to start at zero.  Specifically, we set
\[
    \widetilde T_j=T_j+\lambda+4\sigma,
    \qquad
    L_0=2\lambda+8\sigma
    \quad\text{and}\quad
    u_0=\mu+\lambda+4\sigma.
\]
The pairs $(\widetilde T_j,Y_j)$ are independent across $j$, and
$\widetilde T_j\sim\operatorname{Unif}[0,L_0]$.  Since
$\widetilde T_j-u_0=T_j-\mu$, the translation preserves both the margin and
the ideal threshold response.  Hence, the preceding agreement bound~\eqref{eq:global-threshold-agreement} verifies the hypotheses of
Lemma~\ref{lem:random-threshold-localization} with $L=L_0$, $u=u_0$, and
$h=4\sigma$.
Taking $\alpha=\delta_{\rm loc}$, the lemma produces an estimate
$\widehat u_0$ satisfying
\[
    \Pr\left(
      |\widehat u_0-u_0|\le8\sigma
    \right)
    \ge
    1-\delta_{\rm loc}.
\]
By the sample requirement~\eqref{eq:random-threshold-sample-count}, the required number of samples is
\[
    O\left(
      \frac{L_0}{4\sigma}
      \cdot\log\frac{2}{\delta_{\rm loc}}
    \right)
    =
    O\left(
      \frac{2\lambda+8\sigma}{4\sigma}
      \cdot\log\frac{1}{\delta_{\rm loc}}
    \right)
    =
    O\left(
      \frac{\lambda}{\sigma}
      \cdot\log\frac{1}{\delta_{\rm loc}}
    \right)
\]
noting that $\lambda\ge\sigma$.  Since $s\le8$, we have $\lambda/\sigma\le8\max\{1,\lambda/(s\sigma)\}$, so this has the order claimed in~\eqref{eq:sharp-s-localization}.
In view of $u_0=\mu+\lambda+4\sigma$, we define $c$ by projecting
$\widehat u_0-\lambda-4\sigma$ onto $[-\lambda,\lambda]$.
Because $\mu\in[-\lambda,\lambda]$, projection cannot increase the error, so the preceding success guarantee also holds for $c$.  Finally,
$Q_j^{-1}(1)=(-\infty,T_j]$ is one interval for every $j$.  This proves the lemma when $s\le8$.

\emph{Block-and-offset construction for $s>8$.}
Now suppose that $s>8$.  We construct three coarse partitions of
$[-\lambda,\lambda]$ by restricting shifted block partitions of
$\mathbb R$ to this interval.  The underlying block width $L$ must be no smaller
than the scale $\sigma$ and large enough that at most $s$ blocks intersect
$[-\lambda,\lambda]$.  To achieve this, we choose
\begin{equation}
\label{eq:sharp-block-parameters}
    M
    \coloneqq
    \left\lceil
      \max\left\{1,\frac{2\lambda}{s\sigma}\right\}
    \right\rceil
    =\Theta\left(\max\left\{1, \frac{\lambda}{s\sigma}\right\} \right)
    \quad\text{and}\quad
    L
    \coloneqq
    96M\sigma
    =\Theta\left(\max \left\{\sigma, \frac{\lambda}{s} \right\} \right).
\end{equation}
For each shift index $\nu\in\{0,1,2\}$, let
$\beta_\nu=\nu L/3$ be the corresponding shift.  The resulting shifted
partition of $\mathbb R$ consists of the half-open blocks
$B_{\nu,q} \coloneqq [\beta_\nu+qL,\beta_\nu+(q+1)L)$
for $q\in\mathbb Z$.
Among these blocks, we retain those that intersect $[-\lambda,\lambda]$: Let
\[
    \calI_\nu
    \coloneqq
    \left\{
      q\in\mathbb Z:
      B_{\nu,q}\cap[-\lambda,\lambda]\ne\varnothing
    \right\}
\]
be the set of retained block indices, and let
$K_\nu\coloneqq|\calI_\nu|$ be the number of retained blocks.  The choice of block width in~\eqref{eq:sharp-block-parameters} ensures that~$K_\nu$ respects the interval budget:
\begin{equation}
\label{eq:candidate-block-count}
    K_\nu
    \le
    \frac{2\lambda}{L}+2
    \le
    \frac{2\lambda}{96M\sigma}+2
    \le
    \min\left\{\frac{2 \lambda}{96\sigma}, \frac{s}{96}\right\} + 2
    \le
    \frac{s}{96}+2
    \le
    s,
\end{equation}
where the final inequality uses $s>8$.  For each shifted partition, let $q_\nu^\star$ denote the index of the block
containing $\mu$, and let $u_\nu$ denote the corresponding offset from
the left endpoint of that block.  This gives the block--offset decomposition
\begin{equation}
\label{eq:block-offset-decomposition}
    \mu
    =
    \beta_\nu+q_\nu^\star L+u_\nu,
    \quad\text{where}\quad
    u_\nu\in[0,L).
\end{equation}
We call $B_{\nu,q_\nu^\star}$ the \emph{true block} in the partition indexed
by $\nu$.
Recovering the block index and offset is simplest when observations near
$\mu$ remain in the true block.  If $\mu$ lies close to a block boundary,
even a small deviation of the observation from $\mu$ can place it in a
neighboring block.  We therefore call the shifted partition indexed by
$\nu$ \emph{safe for $\mu$} if $\mu$ lies at least $L/6$ from every block
boundary:
\begin{equation}
\label{eq:safe-phase}
    \dist\left(
      \mu,\{\beta_\nu+qL:q\in\mathbb Z\}
    \right)
    \ge
    \frac{L}{6}.
\end{equation}
Modulo $L$, the boundary phases of the three shifted partitions are
$0$, $L/3$, and $2L/3$.  These phases are separated by $L/3$, so their
open neighborhoods of radius $L/6$ are pairwise disjoint.  
Hence, for every $\mu$, at most one shifted partition is unsafe and at least 
two are safe (see also Figure~\ref{fig:shifted-block-localization}(a)).

\emph{Query families and recovery for a safe shifted partition.}
If $K_\nu=1$, the unique retained block contains $\mu$, so its index
$q_\nu^\star$ is already known and only the offset must be recovered.  This
is achieved via the same random-threshold argument as the case $s\le8$.  When
$K_\nu\ge2$, both the block index and the offset are unknown, and the
protocol uses two query families.  We define and analyze the offset-query
family first; this part applies in both cases.  We then define and analyze
the block-index-query family under $K_\nu\ge2$.  All queries are issued
without knowing which shifted partitions are safe; safety is used only in
the accuracy analysis.

Fix a confidence level $\alpha\in(0,1)$ and an arbitrary shifted partition
indexed by~$\nu$ that is safe for~$\mu$.  We will show that
$O(M\cdot\log(2/\alpha))$ offset queries recover $u_\nu$ to within
$8\sigma$.  When $K_\nu\ge2$, we will also show that
$O(\log K_\nu+\log(2/\alpha))$ block-index queries recover the true block
index $q_\nu^\star$ exactly.
Each of these recovery guarantees holds with probability at
least $1-\alpha$.  The two recovery arguments use safety differently.  For
offset recovery, we will combine the block-width
choice~\eqref{eq:sharp-block-parameters} with the safety
condition~\eqref{eq:safe-phase} to establish the neighborhood
containment~\eqref{eq:safe-neighborhood-containment}.  For block-index
recovery, we will later use the same separation from the block boundaries to
derive the block-crossing bound~\eqref{eq:safe-block-contamination}.

For each $q\in\calI_\nu$, let
$b_{\nu,q}\coloneqq\beta_\nu+qL$ denote the left endpoint of
$B_{\nu,q}$.  For this fixed safe partition, use
$q_\star$, $b_\star$, and $u$ to denote the true block index, its left
endpoint, and the corresponding offset, respectively:
\[
    q_\star
    \coloneqq
    q_\nu^\star,
    \qquad
    b_\star
    \coloneqq
    b_{\nu,q_\star}
    \quad\text{and}\quad
    u
    \coloneqq
    u_\nu
    =
    \mu-b_\star.
\]
Thus $B_{\nu,q_\star}=[b_\star,b_\star+L)$ is the true block for this
fixed safe shifted partition.  The safety
condition~\eqref{eq:safe-phase} gives
$L/6\le u\le5L/6$, while
\eqref{eq:sharp-block-parameters} gives
$4\sigma\le L/24<L/6$.  Hence, $[u-4\sigma,u+4\sigma]\subseteq(0,L)$.  Since
$\mu=b_\star+u$, translating this inclusion by $b_\star$ gives
\begin{equation}
\label{eq:safe-neighborhood-containment}
    [\mu-4\sigma,\mu+4\sigma]
    \subseteq
    [b_\star,b_\star+L)
    =
    B_{\nu,q_\star}.
\end{equation}
In particular, every observation satisfying $|X-\mu|<4\sigma$ falls in
the true block.

\emph{Offset query and recovery.}
To estimate $u$ without knowing $q_\star$, we include the leftmost
subinterval of the same randomly chosen length in every block indexed by~$\calI_\nu$.  For the $j$th offset query, draw
$T_j\sim\operatorname{Unif}[0,L]$ independently across $j$ and independently
of the samples, and define
\begin{equation}
\label{eq:periodic-prefix-query}
    Q_{\nu,T_j}^{\rm off}(x)
    \coloneqq
    \1\left\{
      x\in
      \bigcup_{q\in\calI_\nu}
        [b_{\nu,q},b_{\nu,q}+T_j]
    \right\}.
\end{equation}
By~\eqref{eq:candidate-block-count}, the one-set has at most
$K_\nu\le s$ interval components.
Write $Y_j\coloneqq Q_{\nu,T_j}^{\rm off}(X_j)$ for the response.  On the
event $|X_j-\mu|<4\sigma$, the neighborhood
containment~\eqref{eq:safe-neighborhood-containment} gives
$
    X_j
    \in
    B_{\nu,q_\star}
    =
    [b_\star,b_\star+L).
$
The query definition~\eqref{eq:periodic-prefix-query} therefore reduces to
\begin{equation}
\label{eq:offset-as-threshold}
    Y_j
    =
    \1\{X_j\in[b_\star,b_\star+T_j]\}
    =
    \1\{X_j-b_\star\le T_j\}.
\end{equation}
Moreover, since $u=\mu-b_\star$, the same event gives
$|(X_j-b_\star)-u|=|X_j-\mu|<4\sigma$.
Thus, whenever $|X_j-\mu|<4\sigma$, the offset response has the same
random-threshold form as in the case $s\le8$, now with shifted observation
$X_j-b_\star$ and point $u$.  Since $T_j$ is independent of
$X_j$, the argument leading to
\eqref{eq:global-threshold-agreement} gives, for every $t\in[0,L]$ satisfying
$|t-u|\ge4\sigma$,
\begin{equation}
\label{eq:offset-threshold-agreement}
    \Pr\left(
      Y_j=\1\{t\ge u\}
      \,\middle|\,
      T_j=t
    \right)
    \ge
    \Pr\left(
      |X_j-\mu|<4\sigma
    \right)
    \ge
    \frac{3}{4}.
\end{equation}
The agreement bound~\eqref{eq:offset-threshold-agreement} verifies the
hypothesis of Lemma~\ref{lem:random-threshold-localization} with interval
length $L$, point $u=u_\nu$, margin $h=4\sigma$, and confidence
level $\alpha$.  Since
\eqref{eq:sharp-block-parameters} gives $L/h=24M$, the sample requirement
in~\eqref{eq:random-threshold-sample-count} is satisfied using
\begin{equation}
\label{eq:safe-offset-count}
    m_{{\rm off},\nu}
    =
    O\left(
      M\cdot\log\frac{2}{\alpha}
    \right)
\end{equation}
independent offset queries.  For this choice, the lemma produces an estimate
$\widehat u_\nu$ satisfying
\begin{equation}
\label{eq:safe-offset-success}
    \Pr\left(
      |\widehat u_\nu-u_\nu|\le8\sigma
    \right)
    \ge
    1-\alpha.
\end{equation}

\smallskip
\noindent\emph{Block-index query and recovery.}
We next recover the block index.  If $K_\nu=1$, set $\widehat q_\nu$ equal to the unique element of~$\calI_\nu$.  Then $\widehat q_\nu=q_\nu^\star$ deterministically and
$m_{{\rm idx},\nu}=0$.  We thus suppose that $K_\nu\ge2$.  
For block-index recovery, we use the sampled-coordinate coding construction of Lemma~\ref{lem:LS-localization}, with the blocks indexed by
$\calI_\nu$ replacing the bins used there.
Assign each block indexed by $q\in\calI_\nu$ a codeword
$z_q\in\{0,1\}^{d_{\rm code}}$ using
Lemma~\ref{lem:nearly-equidistant-code} with $\xi=1/4$.  We may take
$d_{\rm code}=O(\log K_\nu)$ while ensuring~$d_H(z_q,z_{q'})/d_{\rm code}\ge1/4$ for every pair of distinct retained
blocks $q$ and $q'$.   For the $j$th block-index query, draw
$I_j\sim\operatorname{Unif}\{1,\ldots,d_{\rm code}\}$ independently across
$j$ and independently of the samples, and define
\begin{equation}
\label{eq:block-index-query}
    Q_{\nu,I_j}^{\rm idx}(x)
    \coloneqq
    \1\left\{
      x\in
      \bigcup_{q\in\calI_\nu:z_{q,I_j}=1}
        B_{\nu,q}
    \right\}.
\end{equation}
By~\eqref{eq:candidate-block-count}, its one-set is a union of at most
$K_\nu\le s$ blocks.  Write
$Y_j\coloneqq Q_{\nu,I_j}^{\rm idx}(X_j)$ for the response.  Whenever
$X_j\in B_{\nu,q_\star}$, the query
definition~\eqref{eq:block-index-query} gives
$Y_j=z_{q_\star,I_j}$.  Thus, unless the observation leaves the true block,
the response is exactly a uniformly sampled code bit of the true block, as
in Lemma~\ref{lem:LS-localization}.
It remains to check that block crossings are sufficiently rare.  By the
safety condition~\eqref{eq:safe-phase}, leaving the true block requires
$|X-\mu|\ge L/6$.  Markov's inequality and the block-width
choice~\eqref{eq:sharp-block-parameters} therefore give
\begin{equation}
\label{eq:safe-block-contamination}
    p_{\rm out}
    \coloneqq
    \Pr\left(
      X\notin B_{\nu,q_\star}
    \right)
    \le
    \frac{6\sigma}{L}
    =
    \frac{1}{16M}
    \le
    \frac{1}{16}.
\end{equation}
As in Appendix~\ref{app:tight-coding-localization}, we decode by minimum
disagreement:  Given $m_{{\rm idx},\nu}$ responses
$(I_j,Y_j)_{j=1}^{m_{{\rm idx},\nu}}$, set
\[
    \widehat q_\nu
    \in
    \arg\min_{q\in\calI_\nu}
    \sum_{j=1}^{m_{{\rm idx},\nu}}
      \1\{Y_j\ne z_{q,I_j}\}.
\]
For a generic block-index coordinate and response $(I,Y)$, let $R_q^{\rm idx}\coloneqq
\E[\1\{Y\ne z_{q,I}\}]$
denote the population disagreement loss of block $q$.
Since the response
equals $z_{q_\star,I}$ whenever the observation remains in the true block,
we have $R_{q_\star}^{\rm idx}\le p_{\rm out}$.  For any
$q\ne q_\star$, independence of $I$ and $X$, uniform sampling of $I$, and
the relative-distance guarantee give
$R_q^{\rm idx}\ge
(1-p_{\rm out})\cdot d_H(z_q,z_{q_\star})/d_{\rm code}
\ge(1-p_{\rm out})/4$.
Consequently, every $q\ne q_\star$ satisfies the constant population-loss
gap
\begin{equation}
\label{eq:block-score-gap}
    R_q^{\rm idx}-R_{q_\star}^{\rm idx}
    \ge
    (1-p_{\rm out})\cdot\frac{1}{4}-p_{\rm out}
    \ge
    \frac{15}{16}\cdot\frac{1}{4}-\frac{1}{16}
    =
    \frac{11}{64}.
\end{equation}
For the fixed safe partition, this comparison is stronger than the one in
the proof of Lemma~\ref{lem:LS-localization} in
Appendix~\ref{app:tight-coding-localization}.  There, the argument rules out
only bins outside the set consisting of the true bin and its two neighbors;
here,~\eqref{eq:block-score-gap} compares the true block $q_\star$ directly
with every $q\ne q_\star$.

The remaining concentration step follows the same proof.  Applying
Hoeffding's inequality followed by a union bound over the $K_\nu$ candidate
blocks, as in~\eqref{eq:localization-uniform-concentration}, and using the
score gap $11/64$ from~\eqref{eq:block-score-gap}, it suffices to take
\begin{equation}
\label{eq:safe-index-count}
    m_{{\rm idx},\nu}
    =
    O\left(
      \log K_\nu
      +
      \log\frac{2}{\alpha}
    \right)
\end{equation}
independent block-index queries. 
On the resulting concentration event, the
empirical losses preserve the positive gap in
\eqref{eq:block-score-gap}, so the decoder selects the true block.  Hence,
\begin{equation}
\label{eq:safe-index-success}
    \Pr\left(
      \widehat q_\nu=q_\nu^\star
    \right)
    \ge
    1-\alpha.
\end{equation}
Together with~\eqref{eq:safe-offset-success}, the preceding analysis
establishes the two recovery guarantees stated above for every safe
shifted partition.

\noindent\emph{Combining the shifted reconstructions.}
For each shift index $\nu\in\{0,1,2\}$, run the offset and block-index
procedures using mutually independent queries and samples.  Set
$\alpha=\delta_{\rm loc}/6$ in the two recovery guarantees and define
\[
    c_\nu
    \coloneqq
    \beta_\nu+\widehat q_\nu L+\widehat u_\nu.
\]
If the shifted partition indexed by $\nu$ is safe, then the recovery
guarantees~\eqref{eq:safe-offset-success} and
\eqref{eq:safe-index-success} give
\[
    |\widehat u_\nu-u_\nu|
    \le
    8\sigma
    \quad\text{and}\quad
    \widehat q_\nu
    =
    q_\nu^\star
\]
on the corresponding success events.
The block--offset decomposition~\eqref{eq:block-offset-decomposition}
therefore gives
\begin{equation}
\label{eq:safe-shift-center-error}
    |c_\nu-\mu|
    =
    |\widehat u_\nu-u_\nu|
    \le
    8\sigma.
\end{equation}
For the fixed mean $\mu$, at most six recovery procedures are associated
with safe shifted partitions, and each such procedure fails with probability
at most $\alpha$.  A union bound therefore shows that all procedures
associated with safe shifted partitions succeed with probability at least
$1-\delta_{\rm loc}$.  At least two
shifted partitions are safe by the boundary-phase argument following
\eqref{eq:safe-phase}.  Hence, on this event, at least two of
$c_0,c_1,c_2$ lie in
$[\mu-8\sigma,\mu+8\sigma]$ by~\eqref{eq:safe-shift-center-error}, and their median lies in the same interval.
Define the final estimate $c$ by projecting this median onto
$[-\lambda,\lambda]$.  Since $\mu\in[-\lambda,\lambda]$, projection cannot
increase the error, and therefore $|c-\mu|\le8\sigma$.  Consequently,
$\Pr(|c-\mu|\le8\sigma)\ge1-\delta_{\rm loc}$.

\emph{Bounding the sample count.}
Summing~\eqref{eq:safe-offset-count} and~\eqref{eq:safe-index-count} over the three shifted partitions, and using the bounds on $M$ and $K_\nu$ in~\eqref{eq:sharp-block-parameters} and
\eqref{eq:candidate-block-count}, gives
\begin{equation}
\label{eq:s-localization-prescribed-count}
    \sum_{\nu=0}^2
    \left(
      m_{{\rm off},\nu}+m_{{\rm idx},\nu}
    \right)
    =
    O\left(
      \log\frac{\lambda}{\sigma}
      +
      \max\left\{
        1,\frac{\lambda}{s\sigma}
      \right\}
      \cdot\log\frac{1}{\delta_{\rm loc}}
    \right).
\end{equation}

All query functions are generated before any response is observed, so the
protocol is non-adaptive.  The interval-count verifications following
\eqref{eq:periodic-prefix-query} and~\eqref{eq:block-index-query} show that
every realized one-set has at most $s$ interval components.  This proves
Lemma~\ref{lem:sharp-s-localization}.
\end{proof}
For use in Appendix~\ref{sec:iid-s-localization}, note that the construction
for $s\le8$ consists of one i.i.d. query family, whereas the construction
for $s>8$ consists of at most six nonempty families: one offset family and,
when $K_\nu\ge2$, one block-index family for each of the three shifted
partitions.  Within each family the queries are i.i.d., and queries are
independent across families, although the family laws need not agree.  
Hence, the construction has the finite-family product structure
used in the proof of Lemma~\ref{lem:iid-s-localization}.

\subsection{Sample Complexity and Completion of
  Theorem~\ref{thm:interval-complexity}}
\label{sec:interval-complexity-assembly}

We now prove the upper bound in~\eqref{eq:full-s-tradeoff}, thereby
completing the proof of Theorem~\ref{thm:interval-complexity}.  The
argument has two stages.  We first
use the grouped moment bounds (Lemma~\ref{lem:grouped-refinement-moments}) in
Appendix~\ref{sec:grouped-refinement-proof} to show that the original
refinement analysis remains valid and that only the query count changes.
We then sum the resulting grouped-refinement budget to obtain
\eqref{eq:grouped-refinement-cost} and add the localization cost
\eqref{eq:sharp-s-localization}.  

\paragraph{Reduction to sample counting.}
Let $c$ be the output of the $s$-interval localizer in
Lemma~\ref{lem:sharp-s-localization}.  Condition on a successful
realization of the localization transcript.  Then $c$ is fixed and
satisfies $|c-\mu|\le8\sigma$.  Set $\bar\sigma=9\sigma$.
Equation~\eqref{eq:moment-transfer} gives
$\E[|X-c|^k]\le\bar\sigma^k$.  Moreover,
\eqref{eq:grouped-lemma-first-moment} and
\eqref{eq:grouped-lemma-second-moment} give the grouped counterparts of
the per-scale first-moment identity and tail-local second-moment bound in
Lemma~\ref{lem:per-scale-main}.  Consequently, after enlarging the universal constant in the allocation~\eqref{eq:main-ni} if necessary, the bias bound~\eqref{eq:base-bias}, the variance aggregation based on~\eqref{eq:main-ni} and~\eqref{eq:dyadic-tail-sum}, and the median
amplification in Appendix~\ref{sec:amplification} apply unchanged when
each scale statistic is replaced by its sum over groups.  Thus grouping
changes only the query count: by~\eqref{eq:number-scale-groups}, the
refinement cost at scale $i$ is multiplied by $G_i$ relative to the
ungrouped construction.

\paragraph{Grouped refinement sample complexity.}
For each scale $i$, group index $g\in\{1,\ldots,G_i\}$, and orientation
$d\in\{\mathrm L,\mathrm R\}$, draw $n_i$ queries independently using the
allocation in~\eqref{eq:main-ni}.  One grouped base estimate is
\[
    \widehat\theta_{\rm grp}(c)
    =
    \sum_{i=1}^{\imax}
    \frac{1}{n_i}
    \sum_{t=1}^{n_i}
    \sum_{g=1}^{G_i}
      \left(
        V_{i,g,t}^{\mathrm L}(c)
        +V_{i,g,t}^{\mathrm R}(c)
      \right).
\]
This estimate uses $2\sum_{i=1}^{\imax}G_i n_i$ observations.  The factor
two accounts for the two orientations and is absorbed into the constants.
By~\eqref{eq:main-ni} and
$\ell_i=\bar\sigma\cdot 2^{i-1}$,
\[
    n_i
    \le
    C_1\frac{\bar\sigma^2}{\eps^2}
      2^{(2-k)(i-1)}+1
    \quad \text{and} \quad
    G_i
    \le
    C_2\left(
      1+\frac{\lambda}{s\bar\sigma}2^{-(i-1)}
    \right).
\]
Multiplying these bounds and summing over the scales gives
\begin{equation}
\label{eq:grouped-count-expansion}
    \sum_{i=1}^{\imax}G_i \cdot n_i
    \le
      C\frac{\bar\sigma^2}{\eps^2}
        \sum_{i=1}^{\imax}2^{(2-k)(i-1)}
      +C\frac{\lambda\bar\sigma}{s\eps^2}
        \sum_{i=1}^{\imax}2^{(1-k)(i-1)} 
      +C\imax
      +C\frac{\lambda}{s\bar\sigma}
        \sum_{i=1}^{\imax}2^{-(i-1)}.
\end{equation}
Since $k>1$,
\begin{equation}
\label{eq:grouped-count-geometric-sums}
    \sum_{i=1}^{\imax}2^{(1-k)(i-1)}
    \le
    \frac{1}{1-2^{1-k}}
      \quad \text{and} \quad
    \sum_{i=1}^{\imax}2^{-(i-1)}
    \le2.
\end{equation}
Moreover, since
$\imax=O_k(\log(\bar\sigma/\eps))$, the assumptions
$\eps<\sigma$ and $\bar\sigma=9\sigma$ give 
\begin{equation}
\label{eq:grouped-count-rounding-bounds}
    \imax
    \le
    C_k\frac{\bar\sigma^2}{\eps^2}
    \le
    C_k\frac{\bar\sigma^2}{\eps^2}
      \sum_{i=1}^{\imax}2^{(2-k)(i-1)}
    \quad \text{and} \quad
    \frac{\lambda}{s\bar\sigma}
      \sum_{i=1}^{\imax}2^{-(i-1)}
    \le
    \frac{2\lambda}{s\bar\sigma}
    \le
    \frac{2\lambda\bar\sigma}{s\eps^2}.
\end{equation}
Applying~\eqref{eq:grouped-count-geometric-sums} and
\eqref{eq:grouped-count-rounding-bounds} to
\eqref{eq:grouped-count-expansion}, and using
$\bar\sigma=9\sigma$, gives
\begin{equation}
\label{eq:grouped-count-before-amplification}
    \sum_{i=1}^{\imax}G_i \cdot n_i
    \le
    C_k\frac{\bar\sigma^2}{\eps^2}
      \sum_{i=1}^{\imax}2^{(2-k)(i-1)}
    +
    C_k\frac{\lambda\sigma}{s\eps^2}.
\end{equation}
Up to constants, the first term on the right-hand side of
\eqref{eq:grouped-count-before-amplification} has the same order as the
ungrouped base-estimator count in~\eqref{eq:base-sample-cases}.
Applying the median amplification from
Appendix~\ref{sec:amplification}, with
$K=O(\log(1/\delta_{\rm ref}))$ independent grouped base blocks and
$\delta_{\rm ref}=\delta/4$, therefore gives the refinement sample
complexity
\begin{equation}
\label{eq:grouped-refinement-cost}
    O_k\left(
      \mathfrak R_k(\sigma/\eps,\delta)
      +
      \frac{\lambda\sigma}{s\eps^2}
      \cdot\log\frac{1}{\delta}
    \right).
\end{equation}

\paragraph{Completing Theorem~\ref{thm:interval-complexity}.}
Take $\delta_{\rm loc}=\delta/4$.  Together with $\delta_{\rm ref}=\delta/4$ above, this reserves the remaining
failure probability $\delta/2$ for the outer branch-count event introduced in Appendix~\ref{app:localization-iid}.
The localization upper bound~\eqref{eq:sharp-s-localization} of
Lemma~\ref{lem:sharp-s-localization}, together with
$\max\{1,x\}\le1+x$, gives
\begin{equation}
\label{eq:n-loc-s-expand}
    n_{\rm loc}^{(s)} =
    O\left(
      \log\frac{\lambda}{\sigma}
      +
      \max\left\{1,\frac{\lambda}{s\sigma}\right\}
      \log\frac{1}{\delta}
    \right)
    =
    O\left(
      \log\frac{\lambda}{\sigma}
      +
      \log\frac{1}{\delta}
      +
      \frac{\lambda}{s\sigma} \cdot
      \log\frac{1}{\delta}
    \right).
\end{equation}
Combining~\eqref{eq:grouped-refinement-cost} and~\eqref{eq:n-loc-s-expand}
and using $\eps\le\sigma$  and $\log(1/\delta)= O_k\left( \mathfrak R_k(\sigma/\eps,\delta) \right)$ yields 
\[
    n = O_k\left(
      \mathfrak R_k(\sigma/\eps,\delta)
      +
      \log\frac{\lambda}{\sigma}
      +
      \frac{\lambda\sigma}{s\eps^2}
      \log\frac{1}{\delta}
    \right),
\]
which is the upper bound in~\eqref{eq:full-s-tradeoff} for the prescribed
independent query families.

The analysis above applies to independent queries with type-dependent
laws.  Appendix~\ref{app:localization-iid} converts the prescribed refinement and
localization families into a single i.i.d. query law without changing the
sample order or the $s$-interval constraint.  This proves the upper side of
\eqref{eq:full-s-tradeoff} in the i.i.d. model.  Together with the
previously established lower side of~\eqref{eq:full-s-tradeoff}, this
completes the proof of Theorem~\ref{thm:interval-complexity}.

\subsection{Proof of the Random-Threshold Localization
  (Lemma~\ref{lem:random-threshold-localization})}
\label{app:random-threshold-proof}

This subsection proves Lemma~\ref{lem:random-threshold-localization}, which was stated
and used in Appendix~\ref{app:sharp-localization}. 

\emph{Proof outline.}
To rule out the event $\{\widehat u>u+2h\}$, we show that, with high
probability, every candidate $v>u+2h$ disagrees with more observed
responses than the reference candidate $u+h$ and therefore cannot be a
minimizer.  We compare with $u+h$ rather than with $u$ because
\eqref{eq:random-threshold-agreement} gives no guarantee when
$t\in(u-h,u+h)$.

For any candidate $v>u+2h$, the indicators
$\1\{t\ge u+h\}$ and $\1\{t\ge v\}$ coincide outside $[u+h,v)$.
For every threshold $t$ inside this interval,
\[
    \1\{t\ge u\}
    =
    \1\{t\ge u+h\}
    =
    1,
    \qquad
    \1\{t\ge v\}
    =
    0.
\]
By~\eqref{eq:random-threshold-agreement}, a response paired with such a
threshold is therefore more likely to agree with $u+h$ than with $v$.
Moreover, $[u+h,v)$ contains the length-$h$ interval
$[u+h,u+2h)$, and the expected number of sampled thresholds falling in this subinterval
is $nh/L$.  A concentration argument then shows that, with high probability, for every
$v>u+2h$, the responses paired with thresholds in $[u+h,v)$ agree more
often with $u+h$ than with $v$.  Since the two indicators coincide outside
this interval, every such $v$ has more empirical disagreements overall than
$u+h$ and therefore cannot be a minimizer.
The analogous comparison with $u-h$ rules out all candidates
below $u-2h$.

We now formalize the outline.
\begin{proof}[Proof of Lemma~\ref{lem:random-threshold-localization}]
For each candidate point $w\in[0,L]$ and each
$j\in\{1,\ldots,n\}$, define the label predicted by $w$ at the sampled
threshold $T_j$ to be $\1\{T_j\ge w\}$.  Accordingly, define the empirical
disagreement count of $w$ by
\begin{equation}
\label{eq:random-threshold-empirical-disagreement}
    \widehat R(w)
    \coloneqq
    \sum_{j=1}^n
      \1\left\{
        Y_j\ne\1\{T_j\ge w\}
      \right\}.
\end{equation}
Hence by its definition in~\eqref{eq:random-threshold-minimizer}, $\widehat u$ is a minimizer of $\widehat R$ over $[0,L]$.  The failure
event can be written as
\[
    \left\{
      |\widehat u-u|>2h
    \right\}
    =
    \{\widehat u>u+2h\}
    \cup
    \{\widehat u<u-2h\}.
\]
\noindent\emph{Upper deviation.}
If $u+2h\ge L$, then the event $\{\widehat u>u+2h\}$ is empty.  Hence,
suppose that $u+2h<L$.
By the definition of the predicted label, the candidates $v$ and $u+h$
predict different labels exactly for samples whose thresholds lie in
$[u+h,v)$.  On these samples, $u+h$ predicts $1$ and $v$ predicts $0$.
Each such sample contributes
\[
    \1\{Y_j\ne0\}
    -
    \1\{Y_j\ne1\}
    =
    Y_j-(1-Y_j)
    =
    2Y_j-1
\]
to $\widehat R(v)-\widehat R(u+h)$, whereas all other samples contribute
zero.  Thus, for every $v\in(u+h,L]$,
\begin{equation}
\label{eq:random-threshold-risk-right}
    \widehat R(v)-\widehat R(u+h)
    =
    \sum_{j:T_j\in[u+h,v)}
      (2Y_j-1).
\end{equation}
Every interval $[u+h,v)$ with $v>u+2h$ contains the fixed interval
$[u+h,u+2h)$.  Denote the number of sampled thresholds in this fixed
interval by
\begin{equation}
\label{eq:random-threshold-count-right}
    N_+
    \coloneqq
    \sum_{j=1}^n
      \1\{T_j\in[u+h,u+2h)\}.
\end{equation}
Since the thresholds are sampled uniformly from $[0,L]$, we have
$N_+ \sim \operatorname{Bin}\left( n, h/L \right)$.

Condition on the sampled thresholds $T_1,\ldots,T_n$.  To control
\eqref{eq:random-threshold-risk-right} simultaneously over all candidates
$v>u+2h$, order the relevant responses so that every candidate comparison
becomes a sum of the first $r$ terms of a single sequence.  
Let
\begin{equation}
\label{eq:random-threshold-total-count-right}
    N_{\mathrm R}
    \coloneqq
    \sum_{j=1}^n
      \1\{T_j\in[u+h,L)\}
\end{equation}
be the number of sampled thresholds in $[u+h,L)$.  Denote the corresponding
responses in nondecreasing order of their thresholds by
$Y_1^+,\ldots,Y_{N_{\mathrm R}}^+$, and set
\[
    Z_k^+
    \coloneqq
    2Y_k^+-1,
    \qquad
    k\in\{1,\ldots,N_{\mathrm R}\}.
\]
Hence, $Z_k^+$ is the contribution of the $k$-th ordered response to the sum
in~\eqref{eq:random-threshold-risk-right} whenever that response is
included.
For each fixed $v\in(u+h,L]$,~\eqref{eq:random-threshold-risk-right} can be rewritten as
\begin{equation}
\label{eq:random-threshold-ordered-sum-right}
    \widehat R(v)-\widehat R(u+h) =
    \sum_{k=1}^r Z_k^+
    \quad \text{where} \quad
     r =
    \sum_{j=1}^n
      \1\{T_j\in[u+h,v)\},
\end{equation}
where the sum is interpreted as zero when $r=0$.
For every $v\in(u+2h,L]$,~\eqref{eq:random-threshold-count-right},~\eqref{eq:random-threshold-total-count-right}, and the
definition of $r$ in~\eqref{eq:random-threshold-ordered-sum-right} give
\begin{equation}
\label{eq:random-threshold-count-comparison-right}
    N_+ =
  \sum_{j=1}^n
      \1\{T_j\in[u+h,u+2h)\} \le
    \sum_{j=1}^n
      \1\{T_j\in[u+h,v)\}
    =  r 
    \le \sum_{j=1}^n
      \1\{T_j\in[u+h,L)\}
    = N_{\mathrm R}.
\end{equation}
Thus simultaneous control over all $v\in(u+2h,L]$ reduces to controlling
the finite collection of sums indexed by
$r\in\{N_+,\ldots,N_{\mathrm R}\}$.

Suppose $\widehat u>u+2h$.
Since $\widehat u$ is a minimizer of $\widehat R$, applying~\eqref{eq:random-threshold-ordered-sum-right} with $v=\widehat u$ gives
\begin{equation*}
    \sum_{k=1}^{\widehat r}Z_k^+
    =
    \widehat R(\widehat u)-\widehat R(u+h)
    \le 0,
    \quad \text{where} \quad
     \widehat r
    \coloneqq
    \sum_{j=1}^n
      \1\{T_j\in[u+h,\widehat u)\}.
\end{equation*}
Since $\widehat u>u+2h$, applying
\eqref{eq:random-threshold-count-comparison-right} with $v=\widehat u$
gives $N_+\le\widehat r\le N_{\mathrm R}$.
Hence, at least one of the sums indexed by
$r\in\{N_+,\ldots,N_{\mathrm R}\}$ is nonpositive.  Consequently,
\begin{equation}
\label{eq:random-threshold-right-failure-inclusion}
    \{\widehat u>u+2h\}
    \subseteq
    \left\{
      \sum_{k=1}^r Z_k^+\le0
      \text{ for some }
      r\in\{N_+,\ldots,N_{\mathrm R}\}
    \right\}.
\end{equation}
We now bound the event on the right-hand side of
\eqref{eq:random-threshold-right-failure-inclusion}.  Conditional on $T_1,\ldots,T_n$, the ordering is deterministic.
Because the pairs $(T_j,Y_j)$ are independent, the ordered responses
$Y_1^+,\ldots,Y_{N_{\mathrm R}}^+$ are conditionally independent, and the
conditional law of each response depends only on its paired threshold.
Each $Y_k^+$ is paired with a threshold $t\ge u+h>u$, for which
$\1\{t\ge u\}=1$.  Hence,~\eqref{eq:random-threshold-agreement} gives
\begin{equation}
\label{eq:random-threshold-conditional-mean-right}
    \E\left[
      Z_k^+
      \,\middle|\,
      T_1,\ldots,T_n
    \right]
    =
    2
    \Pr\left(
      Y_k^+=1
      \,\middle|\,
      T_1,\ldots,T_n
    \right)
    -1
    \ge
    \frac{1}{2}
\end{equation}
for every $k\in\{1,\ldots,N_{\mathrm R}\}$.
By linearity of expectation, for every
$r\in\{1,\ldots,N_{\mathrm R}\}$, 
\[
    \sum_{k=1}^r
    \E\left[
      Z_k^+
      \,\middle|\,
      T_1,\ldots,T_n
    \right]
    \ge
    \frac{r}{2},
\]
and consequently,
\[
    \left\{
      \sum_{k=1}^r Z_k^+\le0
    \right\}
    \subseteq
    \left\{
      \sum_{k=1}^r
      \left(
        Z_k^+
        -
        \E\left[
          Z_k^+
          \,\middle|\,
          T_1,\ldots,T_n
        \right]
      \right)
      \le
      -\frac{r}{2}
    \right\}.
\]
Combining this inclusion with Hoeffding's inequality and using 
$Z_k^+\in\{-1,1\}$ yields
\begin{equation}
\label{eq:random-threshold-hoeffding-right}
    \Pr\left(
      \sum_{k=1}^r Z_k^+\le0
      \,\middle|\,
      T_1,\ldots,T_n
    \right)
    \le
    \exp\left(
      -\frac{2(r/2)^2}{4r}
    \right)
    =
    e^{-r/8}
\end{equation}
for every $r \ge 1$.
For $r=0$, the same bound holds trivially, with the empty sum interpreted
as zero.  Applying a conditional union bound in
\eqref{eq:random-threshold-right-failure-inclusion} and then using
\eqref{eq:random-threshold-hoeffding-right} gives
\begin{equation}
\label{eq:random-threshold-right-conditional-tail}
    \Pr\left(
      \widehat u>u+2h
      \,\middle|\,
      T_1,\ldots,T_n
    \right)
    \le
    \sum_{r=N_+}^{N_{\mathrm R}}e^{-r/8}
    \le
    \frac{e^{-N_+/8}}{1-e^{-1/8}}.
\end{equation}
Averaging~\eqref{eq:random-threshold-right-conditional-tail} over the
sampled thresholds and using the moment-generating function formula for $N_+\sim\operatorname{Bin}(n,h/L)$ gives
\begin{equation}
\label{eq:random-threshold-right-tail}
\begin{aligned}
    \Pr(\widehat u>u+2h)
    &\le
    \frac{1}{1-e^{-1/8}}
    \E\left[
      e^{-N_+/8}
    \right]
    \\
    &=
    \frac{1}{1-e^{-1/8}}
    \left(
      1-\frac{h}{L}
      \left(
        1-e^{-1/8}
      \right)
    \right)^n
    \\
    &\le
    \frac{1}{1-e^{-1/8}}
    \exp\left(
      -\frac{nh}{L}
      \left(
        1-e^{-1/8}
      \right)
    \right)
    \\
    &\le
    9
    \exp\left(
      -\frac{nh}{9L}
    \right),
\end{aligned}
\end{equation}
where the last two inequalities follow from $1-x \le \exp(-x)$ and $1 - e^{-1/8} \ge 1/9$.

\noindent\emph{Lower deviation.}
If $u-2h\le0$, then the event $\{\widehat u<u-2h\}$ is empty.  Otherwise,
for every $v\in[0,u-h)$, the candidates $v$ and $u-h$ predict different
labels exactly for samples whose thresholds lie in $[v,u-h)$.  On these samples, $v$ predicts $1$ and $u-h$ predicts $0$.  Hence,
\begin{equation}
\label{eq:random-threshold-risk-left}
    \widehat R(v)-\widehat R(u-h)
    =
    \sum_{j:T_j\in[v,u-h)}
      (1-2Y_j)
    =
    \sum_{j:T_j\in[v,u-h)}
      \bigl(2(1-Y_j)-1\bigr).
\end{equation}
Let
\[
    N_-
    \coloneqq
    \sum_{j=1}^n
      \1\{T_j\in[u-2h,u-h)\},
\]
and observe that~$N_-\sim\operatorname{Bin}(n,h/L)$.  Whenever $T_j\le u-h$,
we have by~\eqref{eq:random-threshold-agreement} that
\[
    \E\left[
      2(1-Y_j)-1
      \,\middle|\,
      T_1,\ldots,T_n
    \right]
    =
    2
    \Pr\left(
      Y_j=0
      \,\middle|\,
      T_1,\ldots,T_n
    \right)
    -1
    \ge
    \frac{1}{2}.
\]
This is the lower-deviation counterpart of~\eqref{eq:random-threshold-conditional-mean-right}.  Therefore, after
ordering the responses paired with thresholds in $[0,u-h)$ by
nonincreasing threshold, the argument leading from
\eqref{eq:random-threshold-right-failure-inclusion} to
\eqref{eq:random-threshold-right-tail} applies to
\eqref{eq:random-threshold-risk-left}, with increasing order replaced by
decreasing order, $Y_j$ replaced by $1-Y_j$, and $N_+$ replaced by $N_-$.
Hence,
\begin{equation}
\label{eq:random-threshold-left-tail}
    \Pr(\widehat u<u-2h)
    \le
    9
    \exp\left(
      -\frac{nh}{9L}
    \right).
\end{equation}
Combining~\eqref{eq:random-threshold-right-tail} and
\eqref{eq:random-threshold-left-tail}, and then applying
\eqref{eq:random-threshold-sample-count}, gives
\[
    \Pr\left(
      |\widehat u-u|>2h
    \right)
    \le
    18
    \exp\left(
      -\frac{nh}{9L}
    \right)
    \le
    18
    \left(
      \frac{\alpha}{2}
    \right)^{C/9}
    \le
    \alpha,
\]
where the final inequality holds for a sufficiently large universal
constant $C$.  This proves the lemma.
\end{proof}
\section{i.i.d. Query Conversion
  (Theorems~\ref{thm:main} and~\ref{thm:interval-complexity})}
\label{app:localization-iid}

This appendix converts the prescribed query allocations in
Appendices~\ref{app:upper-bound} and~\ref{app:interval-complexity} into the
common i.i.d. query laws asserted in
Theorems~\ref{thm:main} and~\ref{thm:interval-complexity}.  For each fixed
choice of the theorem parameters and sample budget, the common law may depend
on those parameters.  Here a \emph{query type} records only a query's discrete
role: its scale, orientation, and possibly group index for refinement, or its
procedure and shift for $s$-interval localization.  Conditional on the query types, all within-type randomness was sampled independently.

\paragraph{Overview and organization.}
The proof has three steps.  First,
Appendix~\ref{sec:iid-role-importance} converts the prescribed refinement
allocations analyzed in Appendix~\ref{app:upper-bound} and
Appendix~\ref{sec:grouped-refinement-proof}.  It samples each
refinement type in proportion to its prescribed count and reweights its
decoded contribution.  Lemma~\ref{lem:iidification} shows that this preserves the prescribed block's mean and variance bound without requiring the realized type counts to equal their prescribed values.

Second, Appendix~\ref{sec:iid-s-localization} handles $s$-interval
localization.  The finite-union localizer in
Appendix~\ref{app:coding-localization} is already i.i.d., as is the
random-threshold construction for $s\le8$ in the proof of
Lemma~\ref{lem:sharp-s-localization}.  For $s>8$, each of the three shifted partitions uses one offset family and,
when needed, one block-index family, giving at most six families in total. Lemma~\ref{lem:iid-s-localization} mixes these laws and shows that a constant-factor increase in the number of draws supplies the required number from every law with high probability, without changing the sample order.

Finally, Appendix~\ref{sec:complete-iid-proof} mixes the resulting
localization and refinement laws in proportion to their sample budgets.  The
branch probabilities are defined in
\eqref{eq:main-branch-probabilities}, and the bound
\eqref{eq:two-branch-occupancy} shows that both branches receive enough
responses with high probability.  On this event, the decoder retains the
required responses and forms the refinement blocks in
\eqref{eq:iid-block-estimator}, to which the median-amplification argument
from Appendix~\ref{sec:amplification} is applied.

Figure~\ref{fig:iid-query-law} gives a top-down view of the common query law
in the more involved case $s>8$.  Its caption describes the simplifications
for $s\le8$ and for Theorem~\ref{thm:main}.

\begin{figure}[htbp]
\centering
\begin{tikzpicture}[
  x=1cm,
  y=1cm,
  font=\small,
  box/.style={
    draw=gray!70,
    rounded corners=2pt,
    align=center,
    inner xsep=6pt,
    inner ysep=5pt,
    fill=white
  },
  branch/.style={
    box,
    text width=2.15cm,
    minimum height=0.90cm
  },
  choice/.style={
    box,
    text width=3.05cm,
    minimum height=1.00cm
  },
  terminal/.style={
    box,
    text width=3.45cm,
    minimum height=1.00cm
  },
  arrow/.style={->,thick,gray!75},
  probability/.style={
    font=\footnotesize,
    text=gray!80,
    fill=white,
    inner sep=1.5pt
  }
]
  \node[box,fill=gray!8,minimum width=1.90cm] (agent) at (0,0)
    {one agent};

  \node[branch,fill=covergreen!10] (loc) at (3.20,1.35)
    {localization\\branch};
  \node[branch,fill=queryblue!10] (ref) at (3.20,-1.35)
    {refinement\\branch};

  \coordinate (fork) at (1.55,0);
  \draw[thick,gray!75] (agent.east) -- (fork);
  \draw[arrow] (fork) |-
    node[pos=0.76,above left,probability]
    {$B_{\rm loc}^{(s)}/B_{\rm tot}^{(s)}$} (loc.west);
  \draw[arrow] (fork) |-
    node[pos=0.76,below left,probability]
    {$B_{\rm ref}^{(s)}/B_{\rm tot}^{(s)}$} (ref.west);

  \node[choice] (fam) at (7.35,1.35)
    {choose localization\\family $a$};
  \node[terminal,fill=covergreen!6] (locq) at (11.75,1.35)
    {draw $Q\sim\mathsf P_a$\\with fresh within-family\\randomness};
  \draw[arrow] (loc.east) --
    node[above,probability] {$\pi_a$} (fam.west);
  \draw[arrow] (fam.east) -- (locq.west);

  \node[choice] (role) at (7.35,-1.35)
    {choose grouped role\\$q=(i,d,g)$};
  \node[terminal,fill=queryblue!6] (refq) at (11.75,-1.35)
    {draw the type-$q$ query\\with fresh within-type\\randomness};
  \draw[arrow] (ref.east) --
    node[above,probability] {$p_{i,d,g}^{(s)}$} (role.west);
  \draw[arrow] (role.east) -- (refq.west);

  \node[
    draw=gray!55,
    dashed,
    rounded corners=2pt,
    align=center,
    inner xsep=6pt,
    inner ysep=4pt,
    fill=gray!5
  ] (decode) at (11.75,-2.65)
    {decoder reweights by $1/p_{i,d,g}^{(s)}$};
  \draw[->,dashed,gray!65] (refq.south) -- (decode.north);

  \node[align=center,text width=12.5cm] at (6.00,-3.55)
    {Independent repetition over agents gives one common i.i.d. query law.};
\end{tikzpicture}
\caption{Top-down construction of one query for the interval-restricted
protocol when $s>8$.  The outer split uses
\eqref{eq:main-branch-probabilities} with the superscripted budgets; the
localization-family and grouped-role choices use
\eqref{eq:s-localization-mixture-law} and
\eqref{eq:grouped-refinement-role-law}, respectively.  For $s\le8$, the
localization-family choice disappears and the localization branch uses the
single random-threshold law.  For Theorem~\ref{thm:main}, the localization
branch instead uses the finite-union law from
Appendix~\ref{app:coding-localization}, and the refinement role reduces from
$(i,d,g)$ to $(i,d)$ with probabilities
\eqref{eq:refinement-role-law}.  The dashed box denotes decoder-side
reweighting; refinement blocks are also formed after sampling.}
\label{fig:iid-query-law}
\end{figure}

\subsection{Randomizing Additive Refinement Roles}
\label{sec:iid-role-importance}

This subsection proves the decoder-side refinement reweighting shown in
Figure~\ref{fig:iid-query-law}.  The following lemma applies to both the
ungrouped types $(i,d)$ and the grouped types $(i,d,g)$ and treats one block
of $B$ queries.  In the global construction, it is applied conditionally to
each of the $K$ blocks formed from the refinement stream.

\begin{lemma}[Randomizing an additive allocation]
\label{lem:iidification}
Let $\calQ$ be a finite set of query types.  For each $q\in\calQ$, let
$n_q\ge1$, let $Z_q$ be square-integrable, and let
$Z_{q,1},\ldots,Z_{q,n_q}$ be independent copies of $Z_q$.  Define the
prescribed estimator by
\[
    \widehat\theta_{\rm pres}
    =
    \sum_{q\in\calQ}\frac{1}{n_q}
      \sum_{r=1}^{n_q}Z_{q,r}.
\]

Set $B=\sum_{q\in\calQ}n_q$ and $p_q=n_q/B$.  Draw
$I_1,\ldots,I_B$ independently from $(p_q)_{q\in\calQ}$.  Conditional on these
labels, draw $Z'_1,\ldots,Z'_B$ independently, with $Z'_t$ having the law of
$Z_{I_t}$, and define
\[
    W_t=\frac{Z'_t}{p_{I_t}}
    \quad\text{and}\quad
    \widehat\theta_{\rm iid}
    =\frac{1}{B}\cdot\sum_{t=1}^B W_t.
\]
Then
\begin{equation}
\label{eq:iidification-moments}
    \E[\widehat\theta_{\rm iid}]
    =
    \E[\widehat\theta_{\rm pres}]
    =
    \sum_{q\in\calQ}\E[Z_q]
    \quad\text{and}\quad
    \Var(\widehat\theta_{\rm iid})
    \le
    \sum_{q\in\calQ}\frac{\E[Z_q^2]}{n_q}.
\end{equation}
The same conclusions hold conditionally on a transcript $\mathcal H$ if, given $\mathcal H$, the labels are independent with probabilities
$(p_q)_{q\in\calQ}$ and the decoded contributions are conditionally
independent with the corresponding type laws.  In this case, every moment in \eqref{eq:iidification-moments} is interpreted
conditionally on $\mathcal H$.
\end{lemma}
\begin{proof}
Let $I$, $Z'$, and $W$ denote a generic label, decoded contribution, and
reweighted contribution.  Conditioning on the label $I$ gives
\[
    \E[W]
    =\sum_{q\in\calQ}p_q\cdot\frac{\E[Z_q]}{p_q}
    =\sum_{q\in\calQ}\E[Z_q]
    \quad\text{and}\quad
    \E[W^2]
    =\sum_{q\in\calQ}p_q\cdot\frac{\E[Z_q^2]}{p_q^2}
    =B\cdot\sum_{q\in\calQ}\frac{\E[Z_q^2]}{n_q}.
\]
The first identity and linearity of expectation prove the mean identity in~\eqref{eq:iidification-moments}.  Since $W_1,\ldots,W_B$ are independent, the second identity gives
\[
    \Var(\widehat\theta_{\rm iid})
    =
    \frac{1}{B}\cdot\Var(W)
    \le
    \frac{1}{B}\cdot\E[W^2]
    =
    \sum_{q\in\calQ}\frac{\E[Z_q^2]}{n_q},
\]
which proves the variance bound in~\eqref{eq:iidification-moments}.  The conditional statement follows from the same calculation given $\mathcal H$.
\end{proof}

Lemma~\ref{lem:iidification} requires no lower bound on the realized count of
any refinement type, because every realized contribution is reweighted by its
type probability.  For $s>8$, the localization decoder instead requires
prescribed numbers of responses from several query families.

\subsection{i.i.d. Conversion of the
  \texorpdfstring{$s$}{s}-Interval Localizer}
\label{sec:iid-s-localization}

This subsection constructs the localization law shown in
Figure~\ref{fig:iid-query-law}.  The finite-union localizer used for
Theorem~\ref{thm:main} is already i.i.d. by
Appendix~\ref{app:coding-localization}, and the random-threshold construction
in Lemma~\ref{lem:sharp-s-localization} is already i.i.d. for $s\le8$.  For
$s>8$, we convert the prescribed family allocation in that lemma into one
i.i.d. law by mixing the offset and block-index laws used by the three shifted
partitions and retaining the required number of responses from each law.

\begin{lemma}[i.i.d. $s$-interval localization]
\label{lem:iid-s-localization}
For every $\lambda\ge\sigma>0$, integer $s\ge1$, and
$\delta_{\rm loc}\in(0,1/4)$, there is a distribution over $s$-interval query
functions and a decoder that uses $B_{\rm loc}^{(s)}$ i.i.d. draws from this
distribution, where
\begin{equation}
\label{eq:iid-s-localization-budget}
    B_{\rm loc}^{(s)}
    =
    O\left(
      \log\frac{\lambda}{\sigma}
      +
      \max\left\{1,\frac{\lambda}{s\sigma}\right\}
        \cdot\log\frac{1}{\delta_{\rm loc}}
    \right).
\end{equation}
The decoder returns $c\in[-\lambda,\lambda]$ satisfying
$|c-\mu|\le8\sigma$ with probability at least $1-\delta_{\rm loc}$, uniformly
over $\mu\in[-\lambda,\lambda]$ and distributions satisfying
$\E[|X-\mu|]\le\sigma$.
\end{lemma}

\begin{proof}
For $s\le8$, the constant-interval-budget part of the proof of
Lemma~\ref{lem:sharp-s-localization} already draws every query independently from the same random-threshold law.  Its stated guarantee and sample order therefore apply directly.  We henceforth assume $s>8$.

To specify the family laws and sample counts that the i.i.d. mixture must
reproduce, consider the prescribed-family construction in the proof of
Lemma~\ref{lem:sharp-s-localization} with its failure parameter set to
$\delta_{\rm loc}/2$.  The remaining $\delta_{\rm loc}/2$ will account for
the possibility that the mixture supplies too few queries from some family.
For each shift $\nu\in\{0,1,2\}$, the prescribed allocation contains the
offset-query family in \eqref{eq:periodic-prefix-query}, with the count in
\eqref{eq:safe-offset-count}, and, when $K_\nu\ge2$, the block-index family
in \eqref{eq:block-index-query}, with the count in
\eqref{eq:safe-index-count}.

Let $\mathcal A$ index the nonempty pairs $a=(\nu,r)$, where
$\nu\in\{0,1,2\}$ and
$r\in\{\mathrm{off},\mathrm{idx}\}$; thus $|\mathcal A|\le6$.  For each $a\in\mathcal A$, let $m_a$ be its prescribed count and let $\mathsf P_a$ be
its query law.  The queries are mutually independent, and those within family
$a$ are i.i.d. from $\mathsf P_a$.  By increasing the prescribed counts if necessary, without changing their order, we may assume that
$m_a\ge c_0\log(1/\delta_{\rm loc})$ for every nonempty family, where
$c_0>0$ is universal.

Set $B_{\rm pres}^{(s)}=\sum_{a\in\mathcal A}m_a$ and define
\begin{equation}
\label{eq:s-localization-mixture-law}
    \pi_a
    =
    \frac{m_a}{B_{\rm pres}^{(s)}}
    \quad\text{and}\quad
    \mathsf P_{\rm loc}^{(s)}
    =
    \sum_{a\in\mathcal A}\pi_a\cdot\mathsf P_a.
\end{equation}
Choose a sufficiently large universal constant $C_{\rm fam}$ and take
$B_{\rm loc}^{(s)}
=\left\lceil C_{\rm fam}\cdot B_{\rm pres}^{(s)}\right\rceil$
i.i.d. draws from this law, retaining each family label.  
For each $a\in\mathcal A$, let $N_a$ be the number of sampled queries carrying family label $a$, and define the event
\[
    \mathcal E_{\rm fam}
    =
    \bigcap_{a\in\mathcal A}\{N_a\ge m_a\}.
\]
For every $a$,
$\E[N_a]
=B_{\rm loc}^{(s)}\cdot\pi_a
\ge C_{\rm fam}\cdot m_a$.  Since $C_{\rm fam}$ was chosen sufficiently large, a Chernoff bound,
the estimate $m_a=\Omega(\log(1/\delta_{\rm loc}))$, and
$|\mathcal A|\le6$ give
\[
    \Pr(\mathcal E_{\rm fam}^{\mathsf c})
    \le
    \sum_{a\in\mathcal A}\exp\left(-\Omega(m_a)\right)
    \le
    \frac{\delta_{\rm loc}}{2}.
\]
If $\mathcal E_{\rm fam}$ fails, the localization decoder outputs an
arbitrary point of $[-\lambda,\lambda]$.  On $\mathcal E_{\rm fam}$, it
retains the first $m_a$ query--response pairs from each family.  This
selection depends only on the family labels.  Conditional on those labels, the retained queries remain independent and the
family-$a$ query subsequence has law $\mathsf P_a^{\otimes m_a}$.  The
corresponding observations remain independent with the original data law and
are independent of the query randomness.  Hence, the retained query--response
pairs have exactly the joint law analyzed in the prescribed construction, so
\[
    \Pr\left(\{|c-\mu|>8\sigma\}\cap\mathcal E_{\rm fam}\right)
    \le
    \frac{\delta_{\rm loc}}{2}.
\]
Consequently,
\[
    \Pr(|c-\mu|>8\sigma)
    \le
    \Pr(\mathcal E_{\rm fam}^{\mathsf c})
    +
    \Pr\left(\{|c-\mu|>8\sigma\}\cap\mathcal E_{\rm fam}\right)
    \le
    \delta_{\rm loc}.
\]
Every law $\mathsf P_a$ is supported on $s$-interval queries.  Moreover,
\eqref{eq:s-localization-prescribed-count} gives the order in
\eqref{eq:iid-s-localization-budget}, and
$B_{\rm loc}^{(s)}=O\left(B_{\rm pres}^{(s)}\right)$.
\end{proof}
The failure guarantee in Lemma~\ref{lem:iid-s-localization} already includes
the event $\mathcal E_{\rm fam}^{\mathsf c}$ that some localization family
supplies fewer than its prescribed number of queries. The global construction below
only needs an additional count event for the outer
localization--refinement split.

\subsection{The Global i.i.d. Query Law}
\label{sec:complete-iid-proof}

This subsection constructs the outer localization--refinement mixture.
Figure~\ref{fig:iid-query-law} shows its $s>8$ interval-restricted form.  We
first give the proof for Theorem~\ref{thm:main} and then list the changes for
Theorem~\ref{thm:interval-complexity}.

\paragraph{The common query law for Theorem~\ref{thm:main}.}
Set $\delta_{\rm loc}=\delta_{\rm ref}=\delta/4$.  Let $B_{\rm loc}$ be the
sample budget of the finite-union localizer in
Lemma~\ref{lem:LS-localization}, and let
$K=\Theta(\log(1/\delta_{\rm ref}))$ be the number of refinement blocks
chosen in Appendix~\ref{sec:amplification}.  If necessary, enlarge
$B_{\rm loc}$ so that
$B_{\rm loc}\ge\lceil\log(1/\delta_{\rm loc})\rceil$.  The decoder ignores
any additional localization responses, so this enlargement changes neither
the localization guarantee nor the sample order.
We then define
\begin{equation}
\label{eq:main-iid-budgets}
    B_{\rm base}
    =
    2\sum_{i=1}^{\imax}n_i,
    \qquad
    B_{\rm ref}
    =
    K\cdot B_{\rm base},
    \qquad
    B_{\rm tot}
    =
    B_{\rm loc}+B_{\rm ref},
    \quad\text{and}\quad
    n_{\rm tot}=\lceil C B_{\rm tot}\rceil
\end{equation}
for a sufficiently large universal constant $C$.
Here $n_{\rm tot}$ is the total number of agents, or equivalently the total
sample size, in the i.i.d. protocol.

Independently for each agent, draw a branch label
$A\in\{\mathrm{loc},\mathrm{ref}\}$ according to
\begin{equation}
\label{eq:main-branch-probabilities}
    \Pr(A=\mathrm{loc})
    =\frac{B_{\rm loc}}{B_{\rm tot}}
    \quad\text{and}\quad
    \Pr(A=\mathrm{ref})
    =\frac{B_{\rm ref}}{B_{\rm tot}}.
\end{equation}
In the localization branch, draw a query from the i.i.d. law in
Appendix~\ref{app:coding-localization}.  In the refinement branch, draw the
type $q=(i,d)$ according to
\begin{equation}
\label{eq:refinement-role-law}
    p_{i,d}=\frac{n_i}{B_{\rm base}},
    \quad \text{for each }
    i=1,\ldots,\imax
    \quad\text{and}\quad
    d\in\{\mathrm L,\mathrm R\},
\end{equation}
and then draw fresh split and sign variables for that type.  The probabilities
in \eqref{eq:refinement-role-law} sum to one because
$B_{\rm base}=2\sum_i n_i$.  Neither the branch choice nor a refinement query
depends on the eventual localization output $c$; only the decoder uses $c$.
Thus, all $n_{\rm tot}$ full query functions are i.i.d.  The decoder retains
the branch label, the refinement type when applicable, and the internal random
choices.

\paragraph{Enough agents in both branches.}
Let $L_{\rm loc}$ and $L_{\rm ref}$ be the numbers of agents assigned to the
two branches, and define
\[
    \mathcal E_{\rm count}
    =
    \{L_{\rm loc}\ge B_{\rm loc}\}
    \cap
    \{L_{\rm ref}\ge B_{\rm ref}\}.
\]
By the outer branch probabilities in~\eqref{eq:main-branch-probabilities} 
and the budget definitions in~\eqref{eq:main-iid-budgets},
\[
    \E[L_{\rm loc}]
    =
    n_{\rm tot}\cdot\frac{B_{\rm loc}}{B_{\rm tot}}
    \ge
    C B_{\rm loc}
    \quad\text{and}\quad
    \E[L_{\rm ref}]
    =
    n_{\rm tot}\cdot\frac{B_{\rm ref}}{B_{\rm tot}}
    \ge
    C B_{\rm ref}.
\]
Since the universal constant $C$ in the definition of $n_{\rm tot}$ was chosen sufficiently large, a Chernoff bound gives
\begin{equation}
\label{eq:two-branch-occupancy}
    \Pr(\mathcal E_{\rm count}^{\mathsf c})
    \le
    e^{-\Omega(B_{\rm loc})}
    +
    e^{-\Omega(B_{\rm ref})}
    \le
    \delta_{\rm loc}+\delta_{\rm ref}
    = \frac{\delta}{4} + \frac{\delta}{4}
    = \frac{\delta}{2}.
\end{equation}
The second inequality uses
$B_{\rm loc}=\Omega(\log(1/\delta_{\rm loc}))$ from
Appendix~\ref{app:coding-localization} and
$B_{\rm ref}=\Omega(\log(1/\delta_{\rm ref}))$ from
Appendix~\ref{sec:amplification}, together with the sufficiently large choice
of $C$.

If $\mathcal E_{\rm count}$ fails, the decoder outputs an arbitrary estimate
in $[-\lambda,\lambda]$.  On $\mathcal E_{\rm count}$, it retains the first
$B_{\rm loc}$ localization-labeled query--response pairs and the first
$B_{\rm ref}$ refinement-labeled pairs.  Because this selection
uses only the branch labels, conditional on those labels the retained pairs remain independent with their original within-branch laws.

\paragraph{Localization and refinement blocks.}
On $\mathcal E_{\rm count}$, apply the finite-union localization decoder to the retained localization responses and let
$\mathcal E_{\rm loc}=\{|c-\mu|\le8\sigma\}$.  Conditional on any branch-label
realization in $\mathcal E_{\rm count}$, these responses are i.i.d. from the
law in Appendix~\ref{app:coding-localization}.  Hence,
\[
    \Pr(\mathcal E_{\rm loc}^{\mathsf c}
    \cap 
    \mathcal E_{\rm count})
\le\delta_{\rm loc}.
\]
Divide the retained refinement responses into $K$ consecutive blocks of
exactly $B_{\rm base}$ responses.  Let $\mathcal H$ contain all branch labels
and the localization transcript, including $c$.  Conditional on $\mathcal H$,
the refinement queries and responses remain independent with type probabilities given by \eqref{eq:refinement-role-law}; hence the $K$ blocks are conditionally independent.

For each block $b$, let $I_{b,t}$ denote the type of its $t$-th query and let $V_{b,t}(c)$ denote the corresponding decoded refinement contribution.
Define
\begin{equation}
\label{eq:iid-block-estimator}
    W_{b,t}
    =\frac{V_{b,t}(c)}{p_{I_{b,t}}}
    \quad\text{and}\quad
    \widehat\theta_b
    =\frac{1}{B_{\rm base}}
      \sum_{t=1}^{B_{\rm base}}W_{b,t}.
\end{equation}
On $\mathcal E_{\rm count}\cap\mathcal E_{\rm loc}$, condition on
$\mathcal H$.  Then $c$ is fixed, and
Lemma~\ref{lem:iidification} applies to $\widehat\theta_b$ with prescribed
counts $n_{i,d}=n_i$.  Combining the mean identity in
\eqref{eq:iidification-moments} with the bias bound
\eqref{eq:base-bias}, and the variance inequality in
\eqref{eq:iidification-moments} with the single-query second-moment bound
\eqref{eq:single-query-second-moment}, the allocation
\eqref{eq:main-ni}, and the dyadic tail bound
\eqref{eq:dyadic-tail-sum}, gives
\begin{equation}
\label{eq:iid-block-properties}
    \left|
      \E[\widehat\theta_b\mid\mathcal H]-(\mu-c)
    \right|
    \le\frac{\eps}{2}
    \quad\text{and}\quad
    \Var(\widehat\theta_b\mid\mathcal H)
    \le\frac{\eps^2}{64}.
\end{equation}

\paragraph{Completion of Theorem~\ref{thm:main}.}
With the conditional per-block moment bounds in
\eqref{eq:iid-block-properties} established, the remainder follows the
corresponding proof in Appendix~\ref{app:upper-bound}.  We reuse the per-block Chebyshev bound and the median-amplification argument from Appendix~\ref{sec:amplification}, as well as the sample-count and
interval-bound analysis from Appendix~\ref{sec:sample-complexity}.  The only additional ingredient is the
outer branch-count failure in~\eqref{eq:two-branch-occupancy}.
Define
\begin{equation}
\label{eq:iid-final-estimator}
    \widehat\mu
    =
    c+\operatorname{med}_{1\le b\le K}\widehat\theta_b.
\end{equation}
On $\mathcal E_{\rm count}\cap\mathcal E_{\rm loc}$, condition on
$\mathcal H$.  The Chebyshev argument leading to
\eqref{eq:base-failure}, applied with
\eqref{eq:iid-block-properties}, shows that each block differs from
$\mu-c$ by more than $\eps$ with conditional probability at most $1/16$.
Since the blocks are conditionally independent, the median-amplification
bound in~\eqref{eq:median-failure} gives conditional refinement failure
probability at most $\delta_{\rm ref}$.  Combining this bound with the branch-count
and localization failures gives
\begin{equation}
\label{eq:iid-main-total-failure}
    \Pr(|\widehat\mu-\mu|>\eps)
    \le
    \Pr(\mathcal E_{\rm count}^{\mathsf c})
    +
    \Pr(\mathcal E_{\rm loc}^{\mathsf c} \cap \mathcal E_{\rm count})
    +
    \delta_{\rm ref}
    \le
    \frac{\delta}{2}
    +
    \delta_{\rm loc}
    +
    \delta_{\rm ref}
    =
    \delta.
\end{equation}
The calculation in Appendix~\ref{sec:sample-complexity} shows that
$B_{\rm tot}$ has the order in \eqref{eq:main-rate-compact}.  Since
$n_{\rm tot}=O(B_{\rm tot})$ by definition~\eqref{eq:main-iid-budgets}, the i.i.d. conversion preserves this order.  It also preserves the per-query interval bound: Appendix~\ref{sec:query-details}, specifically
\eqref{eq:bank-cardinality-detail}, shows that every refinement query has $O(\lambda/\sigma)$ interval components, while Lemma~\ref{lem:LS-localization} gives the same bound for every localization
query.  This proves the i.i.d. assertion in
Theorem~\ref{thm:main}.

\paragraph{Changes for Theorem~\ref{thm:interval-complexity}.}
The proof differs from the preceding completion only in the localization and
refinement branch laws and their budgets.  We record these substitutions
below; the outer branch-count argument in
\eqref{eq:two-branch-occupancy}, the conditional block construction in
\eqref{eq:iid-block-estimator}, and the subsequent confidence-amplification
and union-bound arguments apply unchanged.

Retain $\delta_{\rm loc}=\delta_{\rm ref}=\delta/4$ and the same
$K=\Theta(\log(1/\delta_{\rm ref}))$.  For localization, use
Lemma~\ref{lem:iid-s-localization}, which converts the prescribed localizer
of Lemma~\ref{lem:sharp-s-localization} into one i.i.d. law.  For refinement,
use the grouped query types $(i,d,g)$ from
\eqref{eq:grouped-refinement-query}, with the group counts $G_i$ in
\eqref{eq:number-scale-groups}.  The prescribed construction in
Appendix~\ref{sec:interval-complexity-assembly} assigns $n_i$ queries to each
such type.  Accordingly, for
$i\in\{1,\ldots,\imax\}$,
$d\in\{\mathrm L,\mathrm R\}$, and
$g\in\{1,\ldots,G_i\}$, define
\begin{equation}
\label{eq:grouped-refinement-role-law}
    B_{\rm base}^{(s)}
    =
    2\sum_{i=1}^{\imax}G_i\cdot n_i
    \quad\text{and}\quad
    p_{i,d,g}^{(s)}
    =
    \frac{n_i}{B_{\rm base}^{(s)}}.
\end{equation}
The probabilities in \eqref{eq:grouped-refinement-role-law} sum to one.
Set
\[
    B_{\rm ref}^{(s)}
    =
    K\cdot B_{\rm base}^{(s)},
    \qquad
    B_{\rm tot}^{(s)}
    =
    B_{\rm loc}^{(s)}+B_{\rm ref}^{(s)},
    \quad\text{and}\quad
    n_{\rm tot}^{(s)}
    =
    \left\lceil C B_{\rm tot}^{(s)}\right\rceil.
\]

Use the branch rule in \eqref{eq:main-branch-probabilities} with the
superscripted budgets.  In the localization branch, draw from the law in
Lemma~\ref{lem:iid-s-localization}; in the refinement branch, draw
$(i,d,g)$ with probability $p_{i,d,g}^{(s)}$ and then sample its within-type
randomness afresh.  Thus the full query functions are i.i.d.

Let $\mathcal E_{\rm count}^{(s)}$ be the event that the localization and
refinement branches receive at least $B_{\rm loc}^{(s)}$ and
$B_{\rm ref}^{(s)}$ agents, respectively.  The argument in
\eqref{eq:two-branch-occupancy} gives
$
    \Pr\left(
      (\mathcal E_{\rm count}^{(s)})^{\mathsf c}
    \right)
    \le
    \delta/2.
$
This event concerns only the outer branch split; the internal localization
family counts are already included in
Lemma~\ref{lem:iid-s-localization}.
If $\mathcal E_{\rm count}^{(s)}$ fails, the decoder outputs an arbitrary
estimate in $[-\lambda,\lambda]$.  On $\mathcal E_{\rm count}^{(s)}$, it
retains the required localization responses and divides the refinement
stream into $K$ blocks of size $B_{\rm base}^{(s)}$.  Define each block estimator by
\eqref{eq:iid-block-estimator}, replacing $B_{\rm base}$, $I_{b,t}$, and
$p_{I_{b,t}}$ by $B_{\rm base}^{(s)}$, $(i,d,g)$, and
$p_{i,d,g}^{(s)}$, respectively, and define $\widehat\mu$ by
\eqref{eq:iid-final-estimator}.

Let
$\mathcal E_{\rm loc}^{(s)}=\{|c-\mu|\le8\sigma\}$, and let
$\mathcal H^{(s)}$ contain the branch labels and the localization transcript.
Conditional on any branch-label realization in
$\mathcal E_{\rm count}^{(s)}$, the retained localization sample has the law
in Lemma~\ref{lem:iid-s-localization} and hence
\[
    \Pr\left(
      (\mathcal E_{\rm loc}^{(s)})^{\mathsf c}
      \cap
      \mathcal E_{\rm count}^{(s)}
    \right)
    \le\delta_{\rm loc}.
\]
On
$\mathcal E_{\rm count}^{(s)}\cap\mathcal E_{\rm loc}^{(s)}$, conditional on
$\mathcal H^{(s)}$, the grouped refinement blocks are independent.
Lemma~\ref{lem:iidification}, applied with $n_{i,d,g}=n_i$, together with
\eqref{eq:grouped-lemma-first-moment} and
\eqref{eq:grouped-lemma-second-moment}, gives the same conditional block
bounds as \eqref{eq:iid-block-properties}.  The median-amplification bound in
\eqref{eq:median-failure} and the union bound in
\eqref{eq:iid-main-total-failure}, with the superscripted events, therefore
give total failure probability at most~$\delta$.  As before, no lower bound
on any realized grouped-type count is required.

Every localization query has at most $s$ interval components by
Lemma~\ref{lem:iid-s-localization}, and every grouped refinement query has at
most $s$ components by \eqref{eq:grouped-refinement-query}.  Moreover, the
localization order in \eqref{eq:iid-s-localization-budget} matches the
prescribed order in \eqref{eq:n-loc-s-expand}, while
\eqref{eq:grouped-refinement-cost} gives the grouped refinement order.
It follows that $B_{\rm tot}^{(s)}$, and hence $n_{\rm tot}^{(s)}$, has the order in
\eqref{eq:full-s-tradeoff}.  This proves the i.i.d. assertion in
Theorem~\ref{thm:interval-complexity}; taking $s=s_k^{\rm opt}$ also proves
the achievability statement in
Corollary~\ref{cor:optimal-components}.
\section{Further Related Work}
\label{app:concurrent-work}

In this appendix, we first place the problem within the broader literature on mean
estimation and communication constraints.  We then compare our refinement construction with the concurrent protocols
of \citet{miao2026universal} and \citet{hu2026interaction}, and explain how
the two VALG theorem candidates relate to these refinement mechanisms.  The comparisons follow the two goals in
the rate analysis: recovering the residual mean $\mu-c$ up to controlled bias,
and making each scale's second moment carry the corresponding tail-probability
factor.  We use our notation for shared objects and retain the terminology of the cited papers only when it marks a substantive difference. 

\subsection{Prior Work}
\label{app:earlier-work}

Section~\ref{sec:related-work} discusses the 1-bit mean-estimation papers most directly related to ours. Here we place the problem in the broader literatures on high-probability mean estimation and  communication-constrained learning and estimation.

\paragraph{Classical mean estimation.}
Without communication constraints, high-probability mean estimation under finite-moment assumptions provides the statistical benchmark. Under finite variance, \citet{devroye2016subgaussian} characterized the
possibilities and limitations of sub-Gaussian mean estimation, and
\citet{lee2022optimal} attained the optimal asymptotic constant without
requiring the variance as an input. \citet{cherapanamjeri2022optimal}
established optimal guarantees under only a $1+\alpha$ moment, when the
variance need not exist. Following the optimal finite-variance result of \citet{lee2022optimal}, \citet{minsker2023efficient} showed that a modified median-of-means estimator attains the same asymptotically optimal sub-Gaussian constant under a finite $2+\eta$ moment.

\paragraph{Communication constrained learning and estimation.}
A broad literature studies distributed estimation under communication,
privacy, and related information constraints. Representative lower-bound
approaches include
information-theoretic arguments
\citep{zhang2013information,shamir2014fundamental}, distributed
data-processing inequalities \citep{braverman2016communication}, geometric
methods \citep{han2018geometric}, reductions from communication complexity
\citep{duchi2019lower}, Fisher-information arguments
\citep{barnes2020lower}, and a unified framework for interactive protocols
\citep{acharya2023unified}.

\paragraph{The role of interaction.}
For learning and estimation under communication and related information
constraints, whether interaction improves the optimal rate depends on the problem.  \citet{dagan2020interaction} proved an exponential separation in distributed learning, while \citet{acharya2022role} established separations for structured high-dimensional estimation.  For distributed Gaussian mean estimation with unknown variance,~\citet{cai2022distributed} showed that interaction can reduce the communication cost of adaptation.  For locally private hypothesis selection, \citet{gopi2020locally} developed multi-round protocols, and \citet{pour2024sample} subsequently proved a strict sample-complexity advantage over noninteractive protocols.  By contrast, \citet{kazemi2025sample} showed that sequential interaction does not reduce the sample complexity of distributed simple binary hypothesis testing.  These examples motivate the question studied here, although none treats the scalar nonparametric 1-bit mean-estimation model.

\paragraph{Distributed empirical mean estimation.}
A separate literature studies communication-efficient estimation of the
empirical mean of a fixed collection of vectors, particularly for distributed and federated learning~\citep{suresh2017distributed,
  konevcny2018randomized, davies2020new, vargaftik2021drive,
  mayekar2021wyner, vargaftik2022EDEN, benbasat2024accelerating,
babu2025unbiased}.
These works treat the client vectors as fixed and measure error relative to their arithmetic average, rather than estimating the population mean of an unknown distribution; see~\citep[\S 1.3]{lau2026order} for further discussion on this distinction.

\subsection{Summary of Comparisons with Concurrent Work}

The four non-adaptive protocols represented in
Table~\ref{tab:concurrent-comparison}, i.e., ours, the two constructions of
\citet{miao2026universal}, and that of \citet{hu2026interaction}, share
the outer architecture shown in Figure~\ref{fig:information-flow}(b) and use
the same coding-based localization strategy from
\citep{lau2026order}.  Their main methodological differences therefore lie
in refinement.
Table~\ref{tab:concurrent-comparison} summarizes the differences between the refinement mechanisms, which are
developed further in Appendices~\ref{app:miao-comparison}
and~\ref{app:hu-zhong-comparison}.  

Section~\ref{sec:related-work} also discusses VALG \citep{zhang2026valg}, an agentic system for machine learning theory whose evaluation includes two theorem candidates for the present problem.  We do not give these candidates separate columns in Table~\ref{tab:concurrent-comparison}, primarily because the main ideas in their refinement mechanisms are already captured in the
table:  The first closely parallels
the dyadic-offset construction of~\citet{miao2026universal}, while the second is methodologically closest to the successive-scale construction of~\citet{hu2026interaction}, though its per-scale refinement mechanism differs.  The second candidate also uses Rademacher signs to isolate contributions from selected regions, as do Hu and Zhong's construction, Miao's signed-grid construction, and ours.

Among them, ours has the most direct conceptual connection to the refinement strategy in \citep{lau2026order}: it likewise estimates and sums contributions from individual cells, but fixes the candidate refinement queries in advance and uses the center $c$ to select the relevant cells and to filter and reweight the stored responses at the decoder; see Section~\ref{sec:stoquant-intuition}.  The other concurrent constructions instead use alternative refinement strategies based on different identities for recovering the residual $\mu-c$.  We do not claim that this closer conceptual connection is necessarily an advantage for either theoretical or
practical purposes.

\begin{table}[htbp]
\centering
\footnotesize
\setlength{\tabcolsep}{4pt}
\renewcommand{\arraystretch}{1.18}
\begin{tabularx}{\textwidth}{@{}>{\raggedright\arraybackslash}p{2.15cm}
  >{\raggedright\arraybackslash}X
  >{\raggedright\arraybackslash}X
  >{\raggedright\arraybackslash}X@{}}
\toprule
Aspect & This paper & \citet{miao2026universal} & \citet{hu2026interaction} \\
\midrule
Decomposition of $\mu-c$
& Sum of $\E[(X-c)\1\{X\in J\}]$ over cells selected using $c$
& \emph{Dyadic offsets:} within-cell offsets telescope across successive
scales.  \emph{Signed grids:} averaging neighboring-cell contributions over
random shifts and widths approximates $X-c$
& Successive truncations telescope to a large-scale approximation of $X-c$ \\
\addlinespace[2pt]
Use of decoded center
& Select the relevant geometric cover from the finite dyadic banks
& \emph{Dyadic offsets:} choose grids whose boundaries are far from $c$.
  \emph{Signed grids:} retain shifts that place $c$ away from its cell
  boundaries
& Use the shifted block containing $c$ and retain the query pair only when
  $c$ is away from the block boundaries \\
\addlinespace[2pt]
Source of the tail factor in the second moment
& Filter responses influenced by near-center candidate intervals and reweight the retained responses
& \emph{Dyadic offsets:} successive offset changes cancel when $X$ is near~$c$.  \emph{Signed grids:} boundary separation keeps $X$ and $c$ in the
  same cell when they are close, making the subtraction zero
& The derivative weight is zero for thresholds near $c$; on retained shifts,
  a nonzero contribution requires an observation far from~$c$ \\
\addlinespace[2pt]
\bottomrule
\end{tabularx}
\caption{Comparison of the four non-adaptive protocols.  The two labeled
descriptions in the Miao column refer to its alternative dyadic-offset and signed-grid constructions.  
}
\label{tab:concurrent-comparison}
\end{table}

\subsection[Comparison with Miao]{Comparison with~\citet{miao2026universal}}
\label{app:miao-comparison}

Once the center $c$ has been decoded, the remaining task is to estimate the residual $\mu-c$.  \citet{miao2026universal} gives two constructions that estimate this residual through different identities while obtaining the same type of tail-dependent second-moment bound discussed in Section~\ref{sec:stoquant-intuition}.

\paragraph{First construction: dyadic offsets.}
The first construction uses two grids at each dyadic scale, shifted by half a cell relative to one another.  For a given grid, the offset of a point is its distance from the left endpoint of the cell containing it.  The queries randomize over the possible grid shifts in advance.  After $c$ is known, the decoder keeps only those queries whose shifts place $c$ away from the relevant cell boundaries and reweights the retained contributions to preserve the desired expectation.
At the finest grid scale, uniform thresholding gives an unbiased estimate of the difference between the offsets of $X$ and $c$.  At each subsequent scale, the same method estimates how this offset difference changes between two successive scales.  These changes telescope, leaving the offset difference at the coarsest scale.  This final difference equals $X-c$ whenever no grid boundary lies between them, and the moment assumption controls the resulting bias.
Moreover, when $X$ is sufficiently close to $c$ relative to a given scale, the offsets of $X$ and $c$ differ by exactly $X-c$ at both that scale and the next.  The correction between the two scales therefore vanishes.  The preceding filtering step enables this cancellation, which is the source
of the required tail factor in the second-moment bound.

\paragraph{Second construction: signed random grids.}
The second construction draws a random cell width and grid shift, then assigns an independent Rademacher sign to every cell.  The returned bit encodes the sign of the cell containing $X$.  After recovering $c$, the decoder keeps only those queries whose shifts place $c$ away from the relevant cell boundaries.  For each retained query, it subtracts the sign of the cell containing~$c$ from the returned sign and multiplies the result by the difference between the signs of the two neighboring cells.  The subtraction is zero whenever $X$ and $c$ occupy the same cell.  Averaging over the random signs cancels the contributions from all cells
except the two neighboring cells in expectation.  Averaging over the grid shift and over all cell widths then reconstructs $X-c$ exactly; restricting the widths to a suitable finite range introduces a controlled bias.
The same-cell cancellation and boundary filter also yield the required tail
factor.  Without the boundary filter, $X$ could cross a cell boundary even
when $|X-c|$ is arbitrarily small.  For a retained query of cell width $r$,
however, such a crossing requires $|X-c|$ to be at least a fixed fraction of
$r$.  The decoder contribution therefore vanishes below this distance, which
supplies the tail factor in the second-moment bound.

\paragraph{Comparison with our construction.}
Both constructions resemble ours at a high level: after $c$ is decoded, the decoder filters and reweights responses from refinement queries collected
in advance.  The quantities reconstructed are nevertheless different.  Miao's first
construction telescopes changes in within-cell offsets, while the second
averages contributions from neighboring cells over random shifts and cell
widths.  We instead select a spatial cover and directly estimate
$\E[(X-c) \cdot \1\{X\in J\}]$ for its cells.  The filtering also plays different
roles.  In the first construction, choosing grids whose boundaries are far
from $c$ controls the resulting bias and enables the cancellation used in the
second-moment bound.  In the second, the boundary filter is used in both the
reconstruction identity and the second-moment bound.  Our candidate-interval
filter is used only for second-moment control.  Miao's second construction
also shares its randomly shifted signed partitions and boundary filtering
with \citet{hu2026interaction}, although the two methods reconstruct different
quantities.

\subsection[Comparison with Hu and Zhong]{Comparison with \citet{hu2026interaction}}
\label{app:hu-zhong-comparison}

\paragraph{Successive-scale decomposition.}
\citet{hu2026interaction} estimate the residual $\mu-c$ by decomposing an
approximation to the identity function $z\mapsto z$ across geometrically
increasing scales $\ell_i$.  For each $i$, they define a function $H_i(z)$ that equals $z$ when $|z|\leq\ell_i$ and vanishes when $|z|\geq2\ell_i$.  They set $g_0=H_0$ and $g_i=H_i-H_{i-1}$ for $i\geq1$.   These functions telescope: $\sum_{i=0}^M g_i=H_M$.  Consequently, summing the quantities estimated at the different scales gives $\E[H_M(X-c)]$ rather than $\E[X-c]$.  The moment
assumption bounds the resulting bias.

\paragraph{Center-independent fixed-scale construction.}
To estimate $\E[g_i(X-c)]$ without knowing $c$ when the queries are chosen,
each query pair randomly shifts a partition of $\mathbb R$ into blocks whose length is proportional to $\ell_i$.  It draws one relative threshold position uniformly at random and uses that position in every block.  It also assigns a Rademacher sign to each block that the decoder could later select, with these signs pairwise independent. The first query combines the sign of the block containing the observation with its comparison to the threshold in that block, while the second query encodes only the block sign.  The two queries use independent
observations.

After $c$ is known, the decoder selects the block containing it and retains
the query pair only when $c$ is sufficiently far from the block endpoints,
reweighting the retained contribution to preserve the desired expectation.
Multiplying the signed versions of both responses by the sign of the selected
block cancels the contributions from all other blocks in expectation.  The second response supplies, in expectation, the correction needed to
subtract the threshold comparison at $c$ from that at $X$.  Averaging this
difference over the random threshold, with a weight determined by the
derivative of~$g_i$, recovers $\E[g_i(X-c)]$.  This boundary filter is needed for the identity because it ensures that every~$x$ for which $g_i(x-c)$ can be nonzero lies inside the selected block.

The preceding boundary filter also enters the second-moment bound.  For
$i\geq1$, the derivative of $g_i$ vanishes near zero.  Together, the vanishing derivative and the boundary filter make the decoder contribution vanish whenever both observations in the query pair are sufficiently close to $c$ relative to $\ell_i$.  This supplies the required tail factor in the second-moment bound.  The component at $i=0$ instead uses the unconditional bound at the smallest
scale.

\paragraph{Comparison with our construction.}
The use of random signs and decoder-side filtering resembles our construction at a high level, but the quantities being reconstructed are different.  At each scale, Hu and Zhong use the block containing $c$ to recover $\E[g_i(X-c)]$ and then add these quantities across scales.  We instead split
decoder-selected cells and directly estimate $\E[(X-c)\1\{X\in J\}]$ for each selected cell $J$.  The filtering also plays different roles.  Hu and Zhong retain a query pair according to the position of $c$ within its randomly shifted block; this filtering is needed both to recover
$\E[g_i(X-c)]$ and to obtain the second-moment bound.  Our filtering depends
on the location of a candidate interval relative to $c$ and is used only for
second-moment control.

\end{document}